\documentclass[11pt, hidelinks]{article}

\usepackage{arxiv}

\usepackage[utf8]{inputenc}
\usepackage[T1]{fontenc}
\usepackage{hyperref}
\usepackage{url}
\usepackage{booktabs}
\usepackage{amsfonts}
\usepackage{nicefrac}
\usepackage{microtype}
\usepackage{graphicx}

\usepackage[natbib=true,
style=ieee,
citestyle=numeric-comp,
maxcitenames=1, mincitenames=1,
maxbibnames=999,
url=false,isbn=false,doi=false]{biblatex}
\DefineBibliographyStrings{english}{%
  andothers = {et al\adddot}
}
\AtEveryBibitem{%
  \clearlist{language}
  \clearfield{series}
  \ifentrytype{book}{}{
      \clearfield{issn} 
      \clearfield{doi} 
      \clearlist{address}
      \clearfield{address}
      \clearfield{month}
      \clearname{editor}
      \clearfield{eprint}
      \clearlist{location}
      \clearfield{chapter}
  }
  \ifentrytype{online}{}{
    \clearfield{url}
  }
  \ifentrytype{inproceedings}{
    \clearlist{publisher}
  }{}
}
\newcommand{\ArchCondText}{Architecture condition}

\usepackage{amsmath,amsfonts,bm}
\usepackage{amsthm}
\newtheorem{theorem}{Theorem}
\newtheorem{corollary}{Corollary}
\newtheorem{proposition}{Proposition}
\newtheorem{lemma}{Lemma}
\newtheorem{fact}{Fact}

\theoremstyle{definition}
\newtheorem{definition}{Definition}
\theoremstyle{remark}
\newtheorem*{remark}{Remark}

\usepackage{tikz}
\usetikzlibrary{shapes.geometric}
\usetikzlibrary {shapes.misc}
\usetikzlibrary{positioning}
\usepackage{lscape}
\usepackage{fontawesome}

\usepackage{subcaption}
\usepackage{apptools}
\AtAppendix{\counterwithin{theorem}{section}}

\renewcommand{\cite}{\citep}

\usepackage{mathtools}
\usepackage{dotlessi}

\usepackage{longtable}
\usepackage{enumerate}

\usepackage{mleftright}

\usepackage[toc,page]{appendix}

\newcommand\blfootnote[1]{%
  \begingroup
  \renewcommand\thefootnote{}\footnote{#1}%
  \addtocounter{footnote}{-1}%
  \endgroup
}

\author{
Isao Ishikawa\textsuperscript{$\dagger$}\\
Ehime University, RIKEN\\
\texttt{ishikawa.isao.zx@ehime-u.ac.jp}\\
\And
Takeshi Teshima\textsuperscript{$\dagger$}\thanks{This work was done when the author was with The University of Tokyo and RIKEN.}\\
The University of Tokyo, RIKEN\\
\texttt{takeshi.teshima@a.riken.jp}\\
\And
Koichi Tojo\\
RIKEN\\
\texttt{koichi.tojo@riken.jp}\\
\AND
Kenta Oono\\
RIKEN\\
\texttt{kenta.oono@a.riken.jp}\\
\And
Masahiro Ikeda\\
RIKEN\\
\texttt{masahiro.ikeda@riken.jp}\\
\And
Masashi Sugiyama\\
RIKEN, The University of Tokyo\\
\texttt{sugi@k.u-tokyo.ac.jp}
}

\newcommand{\titleSentenceCase}{Universal approximation property of invertible neural networks}
\title{\titleSentenceCase}
\title{Universal approximation property of invertible neural networks}

\newcommand{\acknowledgmentContent}{
We would also like to thank Dr.~Taiji Suzuki, Associate Professor of the University of Tokyo, for his valuable comments and fruitful discussions on the distributional universality.
TT was supported by RIKEN Junior Research Associate Program and Masason Foundation.
II and MI were supported by CREST: JPMJCR1913.
II was supported by ACTX: JPMJAX2004.
MS was supported by KAKENHI 20H04206.
}

\hypersetup{
pdftitle={\titleSentenceCase},
pdfsubject={},
pdfauthor={Isao Ishikawa, Takeshi Teshima, Koichi Tojo, Kenta Oono, Masahiro Ikeda, Masashi Sugiyama},
pdfkeywords={Invertible neural network, Normalizing flow, Affine coupling, Neural ordinary differential equations, Diffeomorphism, Universal approximation property, Universality},
}

\date{}

\usepackage{xcolor}

\newcommand{\ARFINN}{CF-INN}
\newcommand{\ARFINNs}{\ARFINN{}s}

\newcommand{\INNs}{INNs}

\newcommand{\CF}{CF}
\newcommand{\CFs}{\CF{}s}

\newcommand{\ACF}{ACF}
\newcommand{\ACFs}{\ACF{}s}

\newcommand{\Gronwall}{Gr\"{o}nwall}

\usepackage{xcolor}

\newcommand{\status}[1]{} 

\newcommand{\Restrict}[2]{#1\vert_{#2}}

\newcommand{\Jac}[1]{D#1}
\newcommand{\vol}[1]{{\rm vol}(#1)}

\newcommand{\pushforward}[2]{{#1}_{*}{#2}}
\newcommand{\upto}[2]{{#2_{{}\leq #1}}}
\newcommand{\from}[2]{{#2_{{}> #1}}}
\newcommand{\LipConst}[1]{L_{#1}}
\newcommand{\supp}[1]{\mathrm{supp}\ #1}
\newcommand{\pdiff}[2]{\partial^{#1}#2}
\newcommand{\truncate}[2]{#2\vert_{#1}}
\newcommand{\diam}[1]{\mathrm{diam}(#1)}
\newcommand{\interior}[1]{{#1}^\circ}

\newcommand{\Identity}{\mathrm{Id}}
\newcommand{\Indicator}[1]{\mathbf{1}_{#1}}

\newcommand{\x}{\mbox{\boldmath $x$}}
\newcommand{\y}{\mbox{\boldmath $y$}}
\newcommand{\ba}{\mbox{\boldmath $a$}}

\def \R {\mathbb{R}}
\def \Re {\mathbb{R}}
\newcommand{\ReD}{\Re^d}
\newcommand{\ReDminus}{\Re^{d-1}}
\newcommand{\Na}{\mathbb{N}}

\newcommand{\Aff}[3]{\Psi_{#1,#2,#3}}

\newcommand{\CcinftyRDminus}{C^\infty_c(\ReDminus)}

\newcommand{\Lipsp}{{\rm Lip}}

\newcommand{\IPM}[2]{\mathrm{IPM}_{#1}{\left(#2\right)}}

\newcommand{\norm}[2]{\left\Vert#1\right\Vert_{#2}}

\newcommand{\inftynorm}[1]{\norm{#1}{\sup}}
\newcommand{\LpKnorm}[1]{\norm{#1}{K,0,p}}
\newcommand{\supRangenorm}[2]{\norm{#2}{#1,0,\infty}}
\newcommand{\LpRangenorm}[2]{\norm{#2}{#1,0, p}}

\newcommand{\supKnorm}[1]{\supRangenorm{K}{#1}}
\newcommand{\Euclideannorm}[1]{\norm{#1}{}}
\newcommand{\opnorm}[1]{\norm{#1}{\mathrm{op}}}
\newcommand{\WspKnorm}[4]{\norm{#4}{#1,#2,#3}}
\newcommand{\Lipnorm}[1]{\norm{#1}{\mathrm{Lip}}}
\newcommand{\bddLipnorm}[1]{\norm{#1}{\mathrm{BL}}}

\newcommand{\zeroMeasure}{\mathbf{0}}
\newcommand{\INNModelGeneric}{\mathcal{M}}

\newcommand{\FLin}{\mathrm{Aff}}
\newcommand{\FGL}{\mathrm{GL}}

\newcommand{\INN}[1]{\mathrm{INN}_{#1}}
\newcommand{\AutoRegressiveFn}{h}

\newcommand{\ACFINNUniversalClass}{\mathcal{H}}

\newcommand{\ARFINNFlow}{\mathcal{G}}

\newcommand{\FDSF}{\mathrm{DSF}}
\newcommand{\FSoS}{\mathrm{SoS}}
\newcommand{\HSoS}{\mathcal{H}\text{-}{\rm SoS}}

\newcommand{\FSACFH}{\FSACFcmd{\ACFINNUniversalClass}}

\newcommand{\FACFHINN}{\INN{\FSACFH}}
\newcommand{\ACFHINN}{\(\FACFHINN\)}

\newcommand{\AutoRegressive}[3]{\AutoRegressiveFn{}_{#1,#2,#3}}

\newcommand{\FSACFcmd}[1]{#1\text{-}\mathrm{ACF}}
\newcommand{\NODEJacobiFuncClass}{\mathcal{H}}

\newcommand{\INNHNODE}{\INN{\Psi(\mathcal{H})}}

\newcommand{\IVPFunc}[1]{\mathrm{IVP}[#1]}
\newcommand{\IVP}[3]{\IVPFunc{#1}(#2, #3)}
\newcommand{\ODEFlowEnds}[1]{\Psi(#1)}

\newcommand{\targetError}[1]{\Delta_{\x_0}(#1)}
\newcommand{\BoundDef}{\delta e^{\LipConst{F}}}
\newcommand{\Bound}{B}

\newcommand{\CtwoDomainDiff}{{\mathcal{D}^2}}
\newcommand{\CrDomainDiff}{{\mathcal{D}^r}}

\newcommand{\DrV}[2]{{\mathcal{D}^{#1}_{#2}}}
\newcommand{\DcRDCmd}[1]{\mathrm{Diff}^{#1}_\mathrm{c}}
\newcommand{\DcRD}{\DcRDCmd{2}}
\newcommand{\DcRDr}{\DcRDCmd{r}}

\newcommand{\FlowEnds}[1]{\Xi^{#1}}
\newcommand{\OneDimTriangularCmd}[1]{\mathcal{S}^{#1}_\mathrm{c}}
\newcommand{\CinftyOneDimTriangular}{\OneDimTriangularCmd{\infty}}
\newcommand{\ConeOneDimTriangular}{\OneDimTriangularCmd{1}}
\newcommand{\CzeroOneDimTriangular}{\OneDimTriangularCmd{0}}
\newcommand{\OneDimTriangular}{\CinftyOneDimTriangular}
\newcommand{\CtwoOneDimTriangular}{\OneDimTriangularCmd{2}}
\newcommand{\CrOneDimTriangular}{\OneDimTriangularCmd{r}}
\newcommand{\Triangular}{\mathcal{T}^\infty}

\newcommand{\vIPMClass}{\mathcal{F}}
\newcommand{\vIPMFn}{f}
\newcommand{\vIPMBaseSp}{\mathcal{X}}
\newcommand{\vIPMBaseDist}{\rho}

\newcommand{\vRKHS}{\mathcal{H}}
\newcommand{\RKHSnorm}[1]{\norm{#1}{\vRKHS}}
\newcommand{\vRKHSkern}{k}

\newcommand{\cTextDudley}{\mathrm{Dud}}
\newcommand{\cTextOneWasserstein}{W_1}
\newcommand{\cTextTV}{\mathrm{TV}}
\newcommand{\cTextMMD}{{\mathrm{MMD}}}

\newcommand{\cIPMDudley}{\vIPMClass_{\cTextDudley}}
\newcommand{\cIPMOneWasserstein}{\vIPMClass_{\cTextOneWasserstein}}
\newcommand{\cOneWassersteinDists}{\mathcal{P}_{\cTextOneWasserstein}}
\newcommand{\cIPMTV}{{\vIPMClass_{\cTextTV}}}
\newcommand{\cIPMMMD}{{\vIPMClass_{\cTextMMD}}}

\begin{document}

\maketitle
\blfootnote{\textsuperscript{$\dagger$}Equal contribution.}
\begin{abstract}
Invertible neural networks (INNs) are neural network architectures with invertibility by design.
Thanks to their invertibility and the tractability of Jacobian, INNs have various machine learning applications such as probabilistic modeling, generative modeling, and representation learning.
However, their attractive properties often come at the cost of restricting the layer designs, which poses a question on their representation power: can we use these models to approximate sufficiently diverse functions?
To answer this question, we have developed a general theoretical framework to investigate the representation power of INNs, building on a structure theorem of differential geometry.
The framework simplifies the approximation problem of diffeomorphisms, which enables us to show the universal approximation properties of INNs.
We apply the framework to two representative classes of INNs, namely Coupling-Flow-based INNs (CF-INNs) and Neural Ordinary Differential Equations (NODEs), and elucidate their high representation power despite the restrictions on their architectures.
\end{abstract}

\section{Introduction}
\emph{Invertible neural networks} (INNs) are neural network architectures with invertibility by design.
They are often endowed with tractable algorithms to compute the inverse map and the Jacobian determinant, such as their explicit formulas.
These characteristics of INNs have enabled a series of new techniques in various machine learning tasks, e.g., generative modeling \citep{DinhDensity2016a,KingmaGlow2018,OordParallel2018,KimFloWaveNet2019,ZhouDensity2019}, probabilistic inference \citep{pmlr-v89-bauer19a,WardImproving2019,pmlr-v70-louizos17a}, solving inverse problems \citep{ArdizzoneAnalyzing2018b}, feature extraction and manipulation \citep{KingmaGlow2018,NalisnickHybrid2019,IzmailovSemisupervised2020,TeshimaFewshot2020}, quantum field theory \citep{AlbergoFlowbased2019}, modeling non-linear dynamics \citep{BevandaKoopmanizingFlows2021,BevandaLearning2021}, and 3D point cloud generation \citep{YangPointFlow2019,KimSoftFlow2020,KimuraChartPointFlow2021}.

INNs have been realized by the careful designs of the special invertible layers called the \emph{flow layers}.
Examples of flow layer designs include \emph{coupling flows} (CFs; \citealp{PapamakariosNormalizing2019,KobyzevNormalizing2019}) and \emph{neural ordinary differential equations} (NODEs; \citealp{ChenNeural2018a}).
CFs employ a highly restricted network architecture in which only some of the input variables undergo some transformations, and the rest of the input variables become the output as-is without being transformed (Section~\ref{sec: CF-INN definitions}).
Also, NODEs offer flow layers by indirectly modeling an invertible function by transforming an input vector through an ordinary differential equation (ODE).
To construct more flexible INNs, multiple such flow layers are composed as well as invertible affine transformation layers.
Moreover, a variety of CF layer designs have been proposed to construct \ARFINNs{} with high representation power, e.g., the affine coupling flow \citep{DinhNICE2014a,DinhDensity2016a,KingmaGlow2018,PapamakariosMasked2017a,KingmaImproved2016}, the neural autoregressive flow \citep{HuangNeural2018b,CaoBlock2019,HoFlow2018}, and the polynomial flow \citep{DBLP:conf/icml/JainiSY19}, each demonstrating enhanced empirical performance.

However, despite the diversity of flow-layer designs \citep{PapamakariosNormalizing2019,KobyzevNormalizing2019}, and their popularity in practice, the theoretical understanding of the representation power of INNs had been limited.
Indeed, the most basic property as a function approximator, namely the \emph{universal approximation property} (or \emph{universality} for short) \citep{CybenkoApproximation1989,HornikMultilayer1989}, had not been elucidated until recently \citep{TeshimaCouplingbased2020,TeshimaUniversal2020,FunahashiApproximate1989}. The universality can be crucial when INNs are used to learn an invertible transformation such as feature extraction \citep{NalisnickHybrid2019} or independent component analysis \citep{TeshimaFewshot2020} because, informally speaking, lack of universality implies that there exists an invertible transformation, even among well-behaved ones, that the INN can never approximate. It would render the model class unreliable for the task of function approximation.

In this work, we show the high representation power of some representative architectures of CF-based INNs and NODE-based INNs by showing their universal approximation properties for a fairly large class of \emph{diffeomorphisms}, i.e., smooth invertible maps with smooth inverse.
The present article is an extended version of \citet{TeshimaCouplingbased2020} and \citet{TeshimaUniversal2020}, but with substantial extensions.
First, we extend the theoretical framework of \citet{TeshimaCouplingbased2020} by taking into account the approximation of the \emph{derivatives} in addition to the function values.
Investigating the representation power to approximate the derivatives can be important in providing machine learning methods with theoretical guarantees. For example, in \citet[Appendix~C.7.]{TeshimaFewshot2020}, the Sobolev norm has been used to characterize the approximation error of an invertible model.

By such an extension, we also strengthen the theoretical guarantees for the distributional approximation using INNs.
Whereas the preliminary version of the framework in \citet{TeshimaCouplingbased2020} could only guarantee the approximation capability in terms of the weak convergence topology, the present framework can elucidate the universality in terms of the \emph{total variation} distance of distributions.
Approximation in total variation distance is a stronger notion that can be useful in providing machine learning algorithms with theoretical guarantees.
See Remark~\ref{rem: IPM implications} in Appendix~\ref{appendix: IPM}.

The difficulty in proving the universality of INNs comes from two complications. (i)~Only function composition can be leveraged to make accurate approximators (e.g., a linear combination of sub-networks is not allowed, as opposed to standard fully-connected neural networks).
(ii)~INNs have architecture-specific inflexibility: CF layers have restricted function forms and NODE layers can only model functions that can be realized by differential equations.
We overcome these complications by problem reduction: we decompose a general diffeomorphism into much simpler ones by using a structural theorem of differential geometry that untangles the structure of a certain diffeomorphism group.
By showing that CF layers and NODE layers can approximate the simple components of the target diffeomorphism, we prove the universality results.

We first provide a general theorem that shows the equivalence of the universality for certain diffeomorphism classes, which can be used to reduce the approximation of a general diffeomorphism to that of a much simpler one.
Then, by leveraging this problem reduction, we show that certain example CF layer designs and NODE result in universal approximators for a general class of diffeomorphisms.

\paragraph{Our contributions.}
Our contributions are summarized as follows.
\begin{enumerate}
  \item We present a theorem to show the equivalence of universal approximation properties for certain classes of functions. The result enables the reduction of the task of proving the universality for general diffeomorphisms to that for much simpler coordinate-wise ones (Theorem~\ref{thm:sobolev-universality-equivalence}.)
        It generalizes and unifies the equivalence theorems previously shown by \citet{TeshimaCouplingbased2020} and \citet{TeshimaUniversal2020}.
  \item We relate functional universality (i.e., universality for approximating functions) to distributional universality (i.e., universality for approximating distributions by pushforward). We introduce a new type of functional approximation property, namely \emph{Sobolev universality}, which is a stronger notion of what has been previously considered by \citet{TeshimaCouplingbased2020} and \citet{TeshimaUniversal2020}. Then, we show Sobolev universality implies the distributional universality in terms of the \emph{weak} topology (Corollary \ref{cor: distributional universality for INN}) and the topology induced by the \emph{total variation} norm (Corollary \ref{cor: total variation distributional universality for INN}) under appropriate assumptions.
  \item We show that the INNs based on certain CF architectures have the Sobolev universality, implying they may be more suitable choices for obtaining theoretical guarantees in the machine learning tasks that require the approximation of derivatives.
\end{enumerate}

\paragraph{Notation}
We list the mathematical notations we use in this paper in the notation tables in Appendix.
We also summarize several mathematical notions and their properties in Appendix \ref{sec:appendix:piecewise diffeo}.

\section{Preliminaries and Related Work}
\label{sec:preliminary}
In this section, we describe the models analyzed in this study, the notion of universality, and related work. 

\subsection{Invertible Layers}
We introduce several invertible layers we consider in this paper, which constitute invertible neural networks.

\subsubsection{Coupling-flow Based Invertible Neural Networks (CF-INNs)}
\label{sec: CF-INN definitions}
We fix $d \in \Na$ and assume $d \geq 2$.
For a vector \(\x \in \ReD\) and \(k\in[d-1]\), we define \(\upto{k}{\x}\) as the vector \((x_1, \ldots, x_k)^\top \in \Re^k\) and \(\from{k}{\x}\) the vector \((x_{k+1}, \ldots, x_d)^\top \in \Re^{d-k}\).

\begin{definition}[Coupling flows]
We define a coupling flow (\CF{}) \citep{PapamakariosNormalizing2019} $\AutoRegressive{k}{\tau}{\theta}$ by
\(\AutoRegressive{k}{\tau}{\theta}(\upto{k}{\x}, \from{k}{\x}) 
= (\upto{k}{\x}, \tau(\from{k}{\x}, \theta(\upto{k}{\x})))\),
where \(k\in[d-1]\), \(\theta\colon \Re^{k} \to \Re^l\) and \(\tau: \Re^{d-k} \times \Re^l \to \Re^{d-k}\) are maps, and \(\tau(\cdot, \theta(\bm{y}))\) is an invertible map for any $\bm{y}\in \Re^{k}$.
\end{definition}

One of the most standard types of \CFs{} is \emph{affine coupling flows} \citep{DinhDensity2016a, KingmaGlow2018, KingmaImproved2016, PapamakariosMasked2017a}.
\begin{definition}[Affine coupling flows]
We define an affine coupling (ACF) flow by the map \(\Aff{k}{s}{t}\) from $\ReD$ to  \( \ReD\) such that \[\Aff{k}{s}{t}(\upto{k}{\x}, \from{k}{\x}) = (\upto{k}{\x}, \from{k}{\x} \odot \exp(s(\upto{k}{\x})) + t(\upto{k}{\x})),\]
where \(k \in [d-1]\), $\odot$ is the Hadamard product, \(\exp\) is applied in an element-wise manner, and \(s,t:\mathbb{R}^{k}\to \mathbb{R}^{d-k}\) are maps.
\end{definition}
The maps $s$ and $t$ are typically parametrized by neural networks.
 
\begin{definition}[Single-coordinate affine coupling flows]
Let $\ACFINNUniversalClass$ be a set of functions from $\R^{d-1}$ to \(\Re\). We define the set of \emph{$\ACFINNUniversalClass$-single-coordinate affine coupling flows} as a subclass of ACFs by $\FSACFH:=\{\Aff{d-1}{s}{t}: s,t\in\ACFINNUniversalClass\}$.
\end{definition}
\(\FSACFH\) is the least expressive flow design appearing in this paper. However, we show in Section~\ref{sec: main result 2} that it can form a \ARFINN{} with universality.
Later, we require various regularity conditions on \(\ACFINNUniversalClass\) depending on the type of universality we want to show.

\subsubsection{Neural ordinary differential equations (NODEs)}
Here, we define the family of NODEs considered in the present paper.
NODE is based on the following fact that any \emph{autonomous} ODE (i.e., an ODE is defined by a time-invariant vector field) with a Lipschitz continuous vector field has a solution and that the solution is unique:
\begin{fact}[Existence and uniqueness of a global solution to an ODE]\label{fact:ODE solution exists for Lip}
Let $\,f \in  \Lipsp{}$.
Then, a solution $z\colon \Re\to \R^d$ to the following ODE exists and it is unique:
\begin{equation}\label{eq:initial value problem}
z(0) = \x, \quad \dot{z}(t)= f(z(t)), \quad t \in \Re,
\end{equation}
where $\x\in \R^d$, and $\dot{z}$ denotes the derivative of $z$ (see \citet{Derrickglobal1976} for example).
\end{fact}

In view of Fact~\ref{fact:ODE solution exists for Lip}, we use the following notation.
\begin{definition}[Autonomous-ODE flow endpoints; \citet{LiDeep2020}]
\label{def:ivp}
\label{def: autonomous ODE flow endpoints}
\sloppy
For $f \in \Lipsp{}$, $\x \in \R^d$, and $t \in \Re$, we define
\[
\IVP{f}{\x}{t} := z(t),
\]
where $z: \Re \to \R^d$ is the unique solution to Equation~\eqref{eq:initial value problem}.
Then, for $\mathcal{F}\subset \Lipsp{}$, we define
\[
\ODEFlowEnds{\mathcal{F}}:= \{\IVP{f}{\cdot}{1} \ |\ f\in \mathcal{F}\}.
\]
\end{definition}
Note that the elements of $\ODEFlowEnds{\mathcal{F}}$ are invertible.

\subsection{Invertible Neural Networks (INNs)}
We consider the INN architectures constructed by composing flow layers, defined as follows.
\begin{definition}[INNs]
\label{def: INNM}
Let $\ARFINNFlow$ be a set consisting of bijective maps on $\ReD$. We define the set of INNs based on \(\ARFINNFlow\) as
\begin{align}
    \INN{\ARFINNFlow} := \left\{ W_1\circ g_1\circ\cdots\circ W_n\circ g_n : \ n \in \Na, g_i\in \ARFINNFlow, W_i \in \FLin\right\}. \label{eq: def of INN}
\end{align}
\end{definition}

\begin{remark}
Previous studies such as \citet{KingmaGlow2018} used \(\FGL\) (see Table \ref{tbl:notation-table:specific} for its definition) in place of \(\FLin\) in the definition of \(\INN{\ARFINNFlow}\). This difference is not a problem in most cases.
For example, if there exists finite elements of \(\ARFINNFlow\) such that their composition equals the map $x \mapsto x + b$ for an arbitrary vector $b \in \ReD$, then, replacing \(\FLin\) with \(\FGL\) does not change the function set \(\INN{\ARFINNFlow}\).
In fact, when \(\ARFINNFlow\) contains \(\FSACFH\) with minimal requirements on $\mathcal{H}$,  we can further reduce the set of linear transformations for INNs from \(\FLin\) to the symmetric group $\mathfrak{S}_d$, that is, the permutations of variables.
See Appendix~\ref{sec:appendix:elementary matrix} for details.
\end{remark}

\subsection{Universal Approximation Properties}
Here, we clarify the notions of universality in this paper.
The definitions use general topological terms, generalizing the $L^p$-universality and $\sup$-universality in \citet{TeshimaCouplingbased2020,TeshimaUniversal2020}.

\subsubsection{Functional universality}
We define the notion of universality for sets of functions, which is a key notion in this paper.
Roughly speaking, a model class is universal for a set of target functions if one can always find a model in the proximity of any target function.
The notion of proximity is stated in general terms of topology.

\begin{definition}[General functional universality]
\label{def: Lp univ. approx.}
Let $U$ be a subset of $\Re^m$ and let $\mathcal{F}_0$ be an $\Re^n$-valued function space on $U$ with some topology and let $\mathcal{F} \subset \mathcal{F}_0$ be a subset.
Let $\INNModelGeneric$ be a model, which is a set of measurable maps from $\Re^m$ to $\Re^n$.  
We say that $\INNModelGeneric$ is an $\mathcal{F}_0$-universal approximator for $\mathcal{F}$ (or has an  $\mathcal{F}_0$-universal approximation property for $\mathcal{F}$),
if $\{g|_U : g \in \INNModelGeneric\}$ is a subset of $\mathcal{F}_0$ and its closure contains $\mathcal{F}$. 
\end{definition}

It is well-known that 2-layer neural networks with suitable activation functions are universal, namely, they can approximate any continuous functions on any compact set in $\ReD$ (see, e.g., \citet{CybenkoApproximation1989}).
In the manner of Definition \ref{def: Lp univ. approx.}, we can translate this fact into the $C^0(\ReD)$-universal approximation property of 2-layer neural networks for $C^0(\ReD)$, where we equip $C^0(\ReD)$ with the topology with semi-norms composed of the sup norms on compact sets.

As an example of $\mathcal{F}_0$, we typically use the $\mathbb{R}^n$-valued {\em local Sobolev space} $W^{r,p}_{\rm loc}(U, \mathbb{R}^n)$, which is roughly speaking the space of $r$-times (weakly-) differentiable measurable functions $f$ such that for any compact set $K \subset U$, $\|f\|_{K,r,p} < \infty$,
where
\[\|f\|_{K,r,p}
:=
\begin{cases}
\displaystyle \sum_{|\alpha|\le r } \left(\int_K \|\pdiff{\alpha}{f}(x)\|^p dx\right)^{1/p} & \text{ if }p<\infty,\\[3pt]
\displaystyle\sum_{|\alpha|\le r } {\rm ess.sup}_{x \in K} \|\partial^\alpha f(x)\| & \text{ if }p = \infty. 
\end{cases}
\]
Formally, we define the local Sobolev space as follows.
\begin{definition}[{\citealp[Appendix~B]{McDuffJholomorphic2004}}]
\sloppy
Let $U$ be a subset of $\mathbb{R}^m$, $r$ a non-negative integer, and $p\in [1, \infty]$.
We define the local Sobolev space $W^{r,p}_{\rm loc}(U, \mathbb{R}^n)$ by
\begin{align*}
W^{r,p}_{\rm loc}(U, \mathbb{R}^n) &:= \lim_{\underset{V}{\longleftarrow}} W^{r,p}(V, \mathbb{R}^n),
\end{align*}
where the right hand side is explicitly defined as the following set:
\begin{align*}
\left\{
(f_V)_V \in \prod_{\substack{V \subset U\text{\rm : open} \\ \overline{V} \subset U}} W^{r,p}(V, \Re^n)
:
f_{V_1}|_{V_2} = f_{V_2} \text{ if }V_2 \subset V_1
\right\}.
\end{align*}
Here, $W^{r,p}(V, \mathbb{R}^n)$ is the $\mathbb{R}^n$-valued Sobolev space on $V$.
We denote $W^{0,p}_{\rm loc}(U, \mathbb{R}^n)$ by $L^p_{\rm loc}(U, \Re^n)$.
\end{definition}

\begin{proposition}
Let $r \ge 1$ be an integer and let $U \subset \mathbb{R}^m$ be an open subset.
Let $f : U \rightarrow \mathbb{R}^n$ be locally $C^{r-1,1}$ (see Table~\ref{tbl:notation-table:specific} for the definition).
Then,  $f \in W^{r,\infty}_{\rm loc}(U, \Re^n)$.
\end{proposition}
\begin{proof}
It follows from Remark~2.12 of \citet{ErnFinite2021} and induction on \(r\).
\end{proof}

This proposition implies that usual models, for example, Multilayer perceptron (MLP) with rectifier linear unit (ReLU) activation functions, are contained in $W^{1,p}_{\rm loc}$ as they are usually locally Lipschitz (note that locally $C^{0, 1}$ means locally Lipschitz).
We call $W^{r,p}_{\rm loc}(U, \Re^n)$-universality the \emph{Sobolev} universality and introduce a special notion for simplicity:
\begin{definition}[$W^{r,p}$-universality and $L^p$-universality]
Notations are as in Definition \ref{def: Lp univ. approx.}.
Let $r$ be a non-negative integer and let $p \in [1,\infty]$.
We say a model $\mathcal{M}$ is a $W^{r,p}$-universal approximator for $\mathcal{F}$ (or has a $W^{r,p}$-universal approximation property for $\mathcal{F}$) if the model $\mathcal{M}$ is a $W^{r,p}_{\rm loc}(U, \mathbb{R}^n)$-universal approximator
for $\mathcal{F}$.
In the case of $r=0$, we use $L^p$- instead of $W^{0,p}$-, for example, we say an $L^p$-universal approximator instead of a $W^{0,p}$-universal approximator.
\end{definition}
\begin{remark}
\label{rem: remark for sup universality}
If $\mathcal{F}^0$ in Definition~\ref{def: Lp univ. approx.} is the space of locally bounded measurable maps with seminorms of \(\sup\) (not ess.sup) norms on compact sets, a model with $\mathcal{F}^0$-universal approximation property is called a $\sup$-universal approximator.
The notion of $\sup$-universality was introduced in \citet{TeshimaCouplingbased2020} and \citet{TeshimaUniversal2020} and is a slightly different concept from $L^\infty$-universality.
We mainly deal with $L^\infty$-universality in this paper.
\end{remark}

\subsubsection{Distributional universality}
We define the notion of distributional universality.
Distributional universality has been used as a notion of theoretical guarantees in the literature on normalizing flows, i.e., probability distribution models constructed using INNs~\citep{KobyzevNormalizing2019}.
We here provide a generalized version of the classical distributional universality as follows:
\begin{definition}[General distributional universality]
\label{def: dist. univ. approx.}
Let $\INNModelGeneric$ be a model which is a set of measurable maps from $\R^m$ to $\R^n$.
Let $\mathcal{P}_0$ be a set of probability measures on $\R^n$ with some topology.
Let $\mathcal{Q} \subset \mathcal{P}_0$ be a subset.
Fix probability measure $\mu_0$ on $\R^m$.
We say that a model $\INNModelGeneric$ is a \emph{$(\mathcal{P}_0, \mu_0)$-distributional universal approximator} for $\mathcal{Q}$ (or \emph{has the $(\mathcal{P}_0, \mu_0)$-distributional universal approximation property} for $\mathcal{Q}$) if $\{\pushforward{g}{\mu_0}: g \in \INNModelGeneric\} \subset \mathcal{P}_0$ and the closure of the set $\{ g_*\mu_0 : g \in \INNModelGeneric\}$ in $\mathcal{P}_0$ contains $\mathcal{Q}$. 
Here, $g_*\mu_0$ denotes the pushforward of $\mu_0$ by $g$.  
\end{definition}
\begin{remark}
\sloppy
When $\mathcal{P}_0 = \mathcal{Q} = \mathcal{P}^{\rm w}$ (see Table~\ref{tbl:notation-table:specific} for the definition of $\mathcal{P}^{\rm w}$), $(\mathcal{P}_0, \mu_0)$-distributional universality for $\mathcal{Q}$ is equivalent to the sequential convergence, that is, the existence of a sequence $\{g_i\}_{i=1}^\infty\subset\INNModelGeneric$ for each $\nu \in \mathcal{P}$ such that $(g_i)_*\mu_0$ converges to $\nu$ in distribution as $i\rightarrow\infty$. 
\end{remark}

\begin{remark}
The distributional universality described in Definition \ref{def: dist. univ. approx.} is a generalized notion considered in existing work.
For example, the distributional universality in~\citet{DBLP:conf/icml/JainiSY19} is rephrased as a $(\mathcal{P}^{\rm w},\nu)$-distributional universal approximation property for $\mathcal{P}_{\rm ab}$ for any $\nu \in \mathcal{P}_{\rm ab}$ in our terminology.
\citet{TeshimaCouplingbased2020} extended the definition by~\citet{DBLP:conf/icml/JainiSY19}. Their distributional universality is a $(\mathcal{P}^{\rm w},\nu)$-distributional universal approximation property for $\mathcal{P}$ for any $\nu \in \mathcal{P}_{\rm ab}$.
It is worth noting that these two concepts of distributional universal approximation are equivalent.
This is essentially because absolutely continuous probability measures are dense in the set of all the probability measures. We prove this fact as Lemma \ref{lem:appendix: abs aprox any} in Appendix \ref{appendix: from lp to dist}.
\end{remark}

\sloppy
The different notions of universality are interrelated.
Most importantly, the \(L^p\)-universality for a certain function class implies the distributional universality (see Proposition~\ref{lem:body:distributional-universality}).
Moreover, if a model \(\INNModelGeneric\) is a \(\sup\)-universal approximator for \(\mathcal{F}\), it is also an $L^p$-universal approximator for \(\mathcal{F}\) for any $p \in [1, \infty]$.

\subsection{Related Work}

\sloppy
Several studies showed the functional or distributional universality of INNs other than CF-INNs and NODEs.
They are not competitive with but complementary to ours as their problem settings are different from ours in target models and evaluation norms.
\citet{GopalELF2021} proposed a type of INNs named Exact-Lipschitz Flows (ELF) and proved their functional universality (more specifically, \(\sup\)-universality in our terminology).
\citet{KongUniversal2021} showed the universality of residual flows in terms of the maximum mean discrepancy (MMD).
They quantitatively evaluated the number of layers needed to approximate a target function with prescribed precision.

Another line of work is to study the expressive power of specific forms of CF-INNs and NODEs.
\citet{HuangConvex2021} introduced Convex Potential Flows, which is a parameterization of invertible models inspired by the optimal transport theory. They proved its distributional universality.
\citet{Ruiz-BaletNeural2021} analyzed a NODE coming from the following form:
\begin{align*}
    \dot{x}(t)=W(t)\sigma(A(t)x(t)+b(t)), 
\end{align*}
where $A$, $W$, and $b$ are time-dependent matrices and a vector. 
They showed that, despite the restricted form, the flow generated by the ODE above has the $L^2$-universal approximation property.
It is an interesting research direction to develop a general theory to broaden the applicability of our results to models like theirs

Since the publication of our previous work~\citep{TeshimaCouplingbased2020,TeshimaUniversal2020}, several researchers have studied the universality of INNs based on our theory.
\citet{PuthawalaUniversal2022} showed that injective flows between $\mathbb{R}^n$ and $\mathbb{R}^m$ ($n\leq m$) universally approximate measures supported on the images of \emph{extendable embeddings}, which is a composition of a full-rank linear transformation followed by a diffeomorphism, in terms of the Wasserstein distance. Their results were built on our previous result of the \(\sup\)-universality of neural autoregressive flows.
\citet{AbeAbelian2021} proposed a novel network architecture called \emph{Abelian group networks} that employs INNs as building blocks.
They proved that Abelian group networks have a functional universal approximation property for Abelian Lie group operations on a Euclidean space. They essentially used the universality of INNs in the proof of the theorem.
Also, concurrently with the present work, \citet{LyuUniversality2022} showed the universality of CF-INNs in the \(C^k\)-norm, i.e., a notion of universality taking into account the approximation of derivatives.
Their result on the \emph{\(C^k\)-universality}, namely Theorem~3.5 in \citet{LyuUniversality2022}, can be reproduced as a special case in our Theorem~\ref{thm:sobolev-universality-equivalence} by selecting \(p = \infty\) and \(\mathcal{G}\) to be a set of diffeomorphisms.
While their proof has the advantage of being more concise thanks to focusing on this special case, they require the models to be smooth everywhere.
On the other hand, our result can accommodate those flow layers which are not smooth everywhere, e.g., CF layers with ReLU activation function which are prevalent in applications.
On a more technical side, our result provides a finer understanding of the diffeomorphism group \(\DcRDr\), which allows us to provide a theoretical guarantee of NODE-based INNs.
More concretely, their proof directly uses the fact that the elements of \(\DcRDr\) can be decomposed into near-\(\Identity\) diffeomorphisms, while our Theorem~\ref{thm:sobolev-universality-equivalence} indicates that \(\DcRDr\) can be decomposed into the elements of \(\Xi^r\), which can be further decomposed into near-\(\Identity\) diffeomorphisms.

As for theoretical limitations of INNs, \citet{OkunoMinimax2021} showed the lower bound (in a minimax sense) of estimation risks in non-parametric regression problems for estimating invertible functions on a plane. Although they constructed an estimator that achieved the lower bound, it is not known whether INNs of any kind can achieve this optimality.

\section{General Framework}
\label{sec:main-results}
In this section, we present the main results (Theorems~\ref{thm:sobolev-universality-equivalence} and \ref{thm: sup universality}) of this paper on the universality of \INNs{}.
The main theorem breaks down the functional universality for a general class of diffeomorphisms into that for a much simpler class of diffeomorphisms.
We also explain the implication of the main theorem to the distributional universality.
The results in this section are derived and stated in a general setup so that it is not limited to the representation power analyses of specific INN architectures.

\subsection{Equivalence of Universal Approximation Properties}
Our first main theorem allows us to lift a universality result for a restricted set of diffeomorphisms to the universality for a fairly general class of diffeomorphisms by showing a certain equivalence of universalities.
Thanks to this problem reduction, we can essentially circumvent the major complication in proving the universality of \ARFINNs{}, namely that only function composition can be leveraged to make complex approximators (e.g., a linear combination is not allowed).

We define the following classes of invertible functions: \(C^r\)-diffeomorphisms $\mathcal{D}^r$, flow endpoints $\FlowEnds{r}$, triangular transformations $\Triangular$, and single-coordinate transformations \(\CrOneDimTriangular\). Our main theorem later reveals an equivalence of \(W^{r,p}\)-universality for these classes.

First, we define the set of \(C^r\)-diffeomorphisms.
\begin{definition}[\(C^r\)-diffeomorphisms: $\mathcal{D}^r$]\label{def: D}
Let $0 \leq r \leq \infty$. For each open subset $U \subset \ReD$, we define $\DrV{r}{U}$ to be the set of maps from $U$ to $\ReD$ which are $C^r$-diffeomorphisms from $U$ to their images.
We denote $\mathcal{D}^r := \sqcup_U \DrV{r}{U}$ (the formal disjoint union of the sets), where $U \subset \ReD$ runs over the set of all open subsets which are $C^r$-diffeomorphic to $\ReD$.
Let $s \le r$. We say that a model $\mathcal{M}$ is a $W^{s,p}$-universal approximator for $\mathcal{D}^r$ if $\mathcal{M}$ is a $W^{s,p}$-universal approximator for $\DrV{r}{U}$ for any open subset $U \subset \ReD$ that is $C^r$-diffeomorphic to $\ReD$.
\end{definition}
We require the domain $U$ to be $C^r$-diffeomorphic to $\ReD$ for technical reasons. However, this constraint would not be too strong: the entire \(\ReD\), any open convex set, and, more generally, any star-shaped open set, all satisfy this condition.
In addition, it is known that if $d \ge 5$, any connected and simply connected open subset in $\ReD$ is always $C^\infty$-diffeomorphic to $\ReD$.

Before going to the second class, we define the set of \emph{compactly-supported} diffeomorphisms on \(\ReD\) as its container.
\begin{definition}[Compactly supported diffeomorphism: $\DcRDr$]
We say a diffeomorphism $f$ on $\mathbb{R}^d$ is {\em compactly supported} if 
there exists a compact subset $K\subset \mathbb{R}^d$ such that for any $x\notin K$, $f(x)=x$.
We use $\DcRDr$ to denote the set of all compactly supported \(C^r\)-diffeomorphisms ($1 \leq r \leq \infty$) from $\R^d$ to $\R^d$.
We regard $\DcRDr$ as a group whose group operation is function composition.
For $f \in \DcRDr$, we define ${\rm supp}f \subset \mathbb{R}^d$ by the closure of the set  $\{ x \in \mathbb{R}^d : f(x) \neq x \}$, which is compact by definition.
\end{definition}

Our second class is a subset $\FlowEnds{r}$ of $\DcRDr$ consisting of \emph{flow endpoints}.

\begin{definition}[Flow endpoints: $\FlowEnds{r}$]
\label{def: flow endpoints}
Let $1 \leq r \leq \infty$. Let $\FlowEnds{r}\subset\DcRDr$ be the set of diffeomorphisms $g$ of the form $g(\bm{x})=\Phi(\bm{x},1)$ for some map $\Phi:\mathbb{R}^d\times U\rightarrow\mathbb{R}^d$ such that
    \begin{itemize}
        \item $U \subset \R$ is an open interval containing $[0, 1]$,
        \item $\Phi(\bm{x},0)=\bm{x}$,
        \item $\Phi(\cdot,t)\in\DcRDr$ for any $t\in U$,
        \item $\Phi(\bm{x},s+t)=\Phi(\Phi(\bm{x},s),t)$ for any $s,t\in U$ with $s+t\in U$,
        \item $\Phi$ is $C^r$ on $\R^d \times U$,
        \item there exists a compact subset $K_\Phi \subset \R^d$ such that $\cup_{t \in U} \mathrm{supp}{\Phi(\cdot, t)} \subset K_\Phi$.
    \end{itemize}
\end{definition}

\begin{remark}
Definition~\ref{def: flow endpoints} is the same as Definition 7 of \citet{TeshimaUniversal2020}.
A similar definition of flow endpoints can be found in Definition 9 of \citet{TeshimaCouplingbased2020}.
The difference between Definition~\ref{def: flow endpoints} and the one of \citet{TeshimaCouplingbased2020} mainly lies in the last two conditions.
Technically, these two conditions are used in Theorem~\ref{thm: NODE is sup-universal} for showing that the partial derivative of $\Phi$ in $t$ at $t=0$ is Lipschitz continuous.
We can prove the universality of \ARFINNs{} without these two conditions, as done in~\citet{TeshimaCouplingbased2020}.
\end{remark}

Finally, we define two subclasses of $\mathcal{D}^r_{\ReD}$ as follows:
\begin{definition}[Triangular transformations: $\Triangular$]
\label{def: T}
We define $\Triangular$ as the set of all 
\emph{increasing triangular} $C^\infty$-maps from \(\ReD\) to \(\ReD\). Here, we say a map $\tau=(\tau_1, \ldots, \tau_d):\ReD\to\ReD$ is increasing triangular if each \(\tau_k(\x)\) depends only on \(\upto{k}{\x}\) and is strictly increasing with respect to \(x_k\).
\end{definition}
\begin{definition}[Single-coordinate transformations: \(\CrOneDimTriangular\)]
\label{def: Ts}
We define $\CrOneDimTriangular$ as the set of all compactly-supported $C^r$-diffeomorphisms $\tau$ satisfying $\tau(\x)=(x_1, \ldots, x_{d-1}, \tau_d(\x))$, i.e., those which alter only the last coordinate.
\end{definition}

Note that for any $r \ge 1$, we have 
\[
\begin{array}{ccccc}
\DrV{0}{\Re^d} & \supset & \DcRDCmd{0} & &~\\[2pt]
\rotatebox{90}{$\subset$}&&\rotatebox{90}{$\subset$}&&\rotatebox{90}{}\\
\DrV{r}{\Re^d} & \supset & \DcRDr & \supset &~\Xi^r \\
\rotatebox{90}{$\subset$}&&\rotatebox{90}{$\subset$}&&\\[2pt]
\Triangular& \supset & \CinftyOneDimTriangular &  & \\
\end{array}
\]
Remark that $\tau_d$ for $\tau \in \CrOneDimTriangular$ ($r\geq 0$) is strictly increasing with respect to $x_d$ since the $C^r$- diffeomorphism $\tau$ is compactly supported. 
Among the above classes of invertible functions, $\mathcal{D}^r$ is our main approximation target, and it is a fairly large class. 
The class \(\Triangular\) relates to the distributional universality as we will see in Proposition~\ref{lem:body:distributional-universality}.
The class \(\OneDimTriangular\) is a much simpler class of diffeomorphisms that we use as a stepladder for showing the universality for \(\mathcal{D}^r\).

Now we are ready to state the first main theorem. It reveals an equivalence among the universalities for \(\mathcal{D}^r\), \(\FlowEnds{\infty}\), \(\Triangular\), and \(\OneDimTriangular\), under mild regularity conditions. We can use the theorem to lift up the universality for \(\OneDimTriangular\) to that for \(\mathcal{D}^r\).

\begin{theorem}[Equivalence for Sobolev universality]\label{thm:sobolev-universality-equivalence}
Let $p \in [1, \infty]$ and let $r \ge0$ be a nonnegative integer.
Let $\ARFINNFlow$ be a set of invertible functions from $\mathbb{R}^d$ to $\mathbb{R}^d$.
\begin{enumerate}[(A)]
\item $p<\infty$ case \label{main thm: A}

Assume that all elements of $\ARFINNFlow$ are piecewise $C^{r+ 1}$-diffeomorphisms (and $C^{r}$ if $r \ge 1$).
Then, the following statements are equivalent:
\begin{enumerate}[1.]
    \item \label{main thm: Dtwo} $\INN{\ARFINNFlow}$ is a  $W^{r,p}$-universal approximator for $\mathcal{D}^{r}$,
    \item \label{main thm: Xitwo} $\INN{\ARFINNFlow}$ is a  $W^{r,p}$-universal approximator for $\FlowEnds{\infty}$,
    \item \label{main thm: Tinfty} $\INN{\ARFINNFlow}$ is a  $W^{r,p}$-universal approximator for $\Triangular$,
    \item \label{main thm: Scinfty} $\INN{\ARFINNFlow}$ is a  $W^{r,p}$-universal approximator for $\OneDimTriangular$.
\end{enumerate}
Moreover,  we may replace $\mathcal{D}^r$ in \eqref{main thm: Dtwo} with ``$C^0(U, \ReD)$ for any open subset of $U\subset \ReD$'' in the case of $r=0$.

\item $p=\infty$ case\label{main thm: B}

Assume the following two conditions: (i) all elements of $\ARFINNFlow$ are locally $C^{r-1,1}$ if $r\ge1$ or locally $L^\infty$ if $r=0$ and (ii) their inverse image of a nullset is again a nullset.
Then, the following statements are equivalent:
\begin{enumerate}[1.] 
    \item \label{main thm: Dtwo p infty} $\INN{\ARFINNFlow}$ is a  $W^{r,p}$-universal approximator for $\mathcal{D}^{\max\{r,1\}}$,
    \item \label{main thm: Xitwo p infty} $\INN{\ARFINNFlow}$ is a  $W^{r,p}$-universal approximator for $\FlowEnds{\infty}$,
    \item \label{main thm: Tinfty p infty} $\INN{\ARFINNFlow}$ is a  $W^{r,p}$-universal approximator for $\Triangular$,
    \item \label{main thm: Scinfty p infty} $\INN{\ARFINNFlow}$ is a  $W^{r,p}$-universal approximator for $\OneDimTriangular$.
\end{enumerate}
\end{enumerate}
\label{thm:body:diffeo-universal-equivalences}
\label{theorem:main:1}
\end{theorem}
\sloppy
The proof is provided in Appendix~\ref{sec:appendix:universality-proof}.
For the definitions of the piecewise $C^r$-diffeomorphisms, locally $C^{r-1,1}$, and locally $L^\infty$, see Appendix~\ref{sec:appendix:piecewise diffeo}.
The regularity conditions in (\ref{main thm: A}) and (\ref{main thm: B}) assure that the functional composition within \(\ARFINNFlow\) is compatible with approximations (see Appendix~\ref{appendix: compatibility of approximation and composition} for details).
These conditions are usually satisfied.

The key step of the proof of this theorem is a decomposition of $f$ into flow endpoints, which is realized by relying on a structure theorem of $\DcRDCmd{\infty}$ (Fact~\ref{thm: simplicity} in Appendix~\ref{sec:appendix:universality-proof}) attributed to \citet{HermanGroupe1973}, \citet{ThurstonFoliations1974}, \citet{Epsteinsimplicity1970}, and \citet{MatherCommutators1974, MatherCommutators1975}.

\begin{remark}
In the case of $p=\infty$,  the ``$\max\{r,1\}$'' on $\mathcal{D}$ is essential, i.e., the target function class cannot be relaxed to $\mathcal{D}^0$.
We can show this by contradiction. If we supposed the equivalence of universality between $\mathcal{D}^0$ and $\mathcal{S}_c^\infty$, then we could see that a diffeomorphism on $\ReD$ can arbitrarily approximate a homeomorphism $\ReD$, but it is not true, namely there exists a homeomorphism that cannot be approximated by any diffeomorphism.
\end{remark}

As for the $\sup$-universality (Remark \ref{rem: remark for sup universality}), we have a similar result:
\begin{theorem}\label{thm: sup universality}
Assume that $\mathcal{G}$ consists of locally bounded measurable mappings.
The equivalence of \eqref{main thm: B} in Theorem \ref{theorem:main:1} is valid if we replace ``$W^{r,p}$-'' with ``$\sup$-'' and set $r=0$. 
\end{theorem}
This theorem slightly strengthens Theorem~1 in \citet{TeshimaCouplingbased2020} which provides the equivalence of the universality between $\CinftyOneDimTriangular$ and $\mathcal{D}^2$ instead of $\mathcal{D}^1$.

\subsection{Implications of the Main Theorem for Distributional Universality}

Next, we give two consequences of Theorem~\ref{thm:sobolev-universality-equivalence} (namely, Corollary~\ref{cor: distributional universality for INN} and Corollary~\ref{cor: total variation distributional universality for INN}).
We first note the relationship between functional universality (Definition \ref{def: Lp univ. approx.}) and distributional universality (Definition \ref{def: dist. univ. approx.}).
\begin{proposition}
\label{lem:body:distributional-universality}
Let $p \in [1,\infty]$.
An \(L^p\)-universal approximator for \(\Triangular\) is a $(\mathcal{P}^{\rm w}, \nu)$-distributional universal approximator for $\mathcal{P}$ for any $\nu \in \mathcal{P}_{\rm ab}$
\end{proposition}
The proof is based on the existence of a triangular map connecting two absolutely continuous distributions \citep{BogachevTriangular2005}. See Appendix~\ref{appendix: from lp to dist} for details.
Note that the previous studies \citep{DBLP:conf/icml/JainiSY19,HuangNeural2018b} have discussed the distributional universality of some flow architectures essentially via showing the \(\sup\)-universality for \(\Triangular\).
Proposition~\ref{lem:body:distributional-universality} clarifies that the weaker notion of \(L^p\)-universality is sufficient for the distributional universality since \(\sup\)-universality implies \(L^p\)-universality.

Proposition~\ref{lem:body:distributional-universality} can be combined with both cases of (\ref{main thm: A}) and (\ref{main thm: B}) in Theorem~\ref{thm:body:diffeo-universal-equivalences}, namely, we have the following corollary:
\begin{corollary}[Sobolev universality implies weak topology universality] \label{cor: distributional universality for INN}
Notations and assumptions are as in Theorem~\ref{theorem:main:1}.
Then, if $\INN{\ARFINNFlow}$ is a $W^{r,p}$-universal approximator for $\OneDimTriangular$, then it is a $(\mathcal{P}^{\rm w}, \nu)$-distributional universal approximator for $\mathcal{P}$ for any $\nu \in \mathcal{P}_{\rm ab}$.
\end{corollary}

If the model can also universally approximate the derivatives, then it is guaranteed to have a stronger distributional universality in terms of the total variation distance, as we see in the following proposition:

\begin{proposition}
\label{proposition:body:generalized-distributional-universality}
Let $r \ge 1$.
Let $\mathcal{F}_0 := W^{0,\infty}_{\rm loc}(\ReD,\mathbb{R}^d) \cap W^{1,1}_{\rm loc}(\ReD,\mathbb{R}^d)$, where we define the topology $\mathcal{F}_0$ to be the weakest topology such that the inclusion maps $\dotlessi_0: \mathcal{F}_0 \xhookrightarrow{} W_{\rm loc}^{0,\infty}(\ReD, \mathbb{R}^d)$ and $\dotlessi_1: \mathcal{F}_0 \xhookrightarrow{} W_{\rm loc}^{1,1}(\ReD, \mathbb{R}^d)$ are both continuous.
Suppose any element in model $\mathcal{M}$ is 
locally $C^{0,1}$ and 
a piecewise $C^1$-diffeomorphism.
If $\mathcal{M}$ is an $\mathcal{F}_0$-universal approximator for $\Triangular$, 
then $\mathcal{M}$ is a $(\mathcal{P}^{\rm TV}, \nu)$-distributional universal approximator for $\mathcal{P}_{\rm ab}$ for any $\nu \in \mathcal{P}_{\rm ab}$.
\end{proposition}
Since $W^{1,\infty}_{\rm loc}(\ReD,\ReD)$ is continuously included in the space $\mathcal{F}_0$ defined in Proposition \ref{proposition:body:generalized-distributional-universality}, we immediately have
\begin{corollary}[Sobolev universality implies total variation universality] \label{cor: total variation distributional universality for INN}
Notation is the same as Theorem \ref{theorem:main:1}.
Assume that any element of $\mathcal{G}$ is locally $C^{0,1}$ and a piecewise $C^1$-diffeomorphism.
Then, if $\INN{\ARFINNFlow}$ is a $W^{1,\infty}$-universal approximator for $\CinftyOneDimTriangular$, then so is a $(\mathcal{P}^{\rm TV}, \nu)$-distributional universal approximator for $\mathcal{P}_{\rm ab}$ for any $\nu \in \mathcal{P}_{\rm ab}$.
\end{corollary}
We defer their proofs to Appendix~\ref{sec:proof-of-generalized-distributional-universality}.

\section{Application of the General Framework}
In this section, we show several crucial results for the universalities of INNs with certain flow layers.
\subsection{Affine Coupling Flows (ACFs)}\label{sec: main result 2}

Here, we reveal the $L^p$-universality of $\INN{\FSACFH}$.
This result affirmatively answers an unsolved problem for the distributional universality of ACF-based invertible neural networks.
\begin{theorem}[$L^p$-universality of \(\INN{\FSACFH}\)]
\label{prop:body:acfinn-Lp}\label{theorem:main:2}
Let \(p \in [1, \infty)\).
Assume that \(\ACFINNUniversalClass\) is an $L^\infty$-universal approximator for $C^0(\Re^{d-1})$ and that it consists of piecewise \(C^1\)-functions.
Then, \(\INN{\FSACFH}\) is an \(L^p\)-universal approximator for 
$C^0(U, \ReD)$ for any open subset $U \subset \ReD$.
\end{theorem}
We remark that the universality is still valid if we restrict the affine layers of $\INN{\FSACFH}$ to elements in $\mathfrak{S}_d$, the permutations of variables. 
For the definition of piecewise \(C^1\)-functions, see Appendix~\ref{sec:appendix:piecewise diffeo}.
We provide the proof of Theorem \ref{theorem:main:2} by combining Theorem \ref{theorem:main:1} with Theorem \ref{thm:strong approximation property} and a slightly general result, which is an $L^p$-universal approximation property of $\INN{\FSACFH}$ for $\CzeroOneDimTriangular$, in Appendix~\ref{appendix: theorem 2 proof}.
Examples of \(\ACFINNUniversalClass\) satisfying the condition of Theorem~\ref{prop:body:acfinn-Lp} include MLP models with ReLU activation \citep{LeCunDeep2015} and a linear-in-parameter model with smooth universal kernels \citep{MicchelliUniversal2006a}.

By combining Theorem~\ref{theorem:main:1}, Theorem~\ref{theorem:main:2}, and Proposition~\ref{lem:body:distributional-universality}, we can affirmatively answer a previously unsolved problem \citep[p.13]{PapamakariosNormalizing2019}, the distributional universality of \ARFINN{} based on \ACFs{}, and we can confirm the theoretical plausibility of using it for normalizing flows.
\begin{theorem}[Distributional universality of \(\INN{\FSACFH}\)]
Under the conditions of Theorem~\ref{theorem:main:2}, \(\INN{\FSACFH}\) is a $(\mathcal{P}^{\rm w}, \nu)$-distributional universal approximator for $\mathcal{P}$ for any $\nu \in \mathcal{P}_{\rm ab}$.
\label{cor:body:acfinn-dist-universal}
\label{cor:body:afinn-distribution-universal}
\end{theorem}

\subsection{Neural Ordinary Differential Equations (NODEs)}

The following shows that the INNs based on NODEs can approximate diffeomorphisms with respect to the $W^{r,\infty}$-norm.
We denote by $\Lipsp{} \cap C^r$ the space of Lipschitz and $C^r$ maps from $\ReD$ to $\ReD$ and we equip it with the relative topology of $W^{r,\infty}_{\rm loc}(\ReD, \ReD)$.
\begin{theorem}[Universality of NODEs]\label{thm: NODE is sup-universal}
Let $r \ge 0$.
Assume \(\NODEJacobiFuncClass \subset \Lipsp{}\cap C^r \) is a $W^{r,\infty}$-universal approximator for $\Lipsp{} \cap C^r$.
Then, \(\INNHNODE\) is a $W^{r,\infty}$-universal approximator for $\mathcal{D}^{\max(r,1)}$.
\end{theorem}
Theorem~\ref{thm: NODE is sup-universal} is shown by applying Theorem~\ref{thm:body:diffeo-universal-equivalences} in combination with Lemma~\ref{red to comp. supp. diff} (Appendix~\ref{sec:appendix:reduction to cpt supp}) to approximate the elements of $\FlowEnds{\infty}$ by NODEs.
A proof is in Appendix~\ref{sec:appendix:universality-proof-NODE-version}.
We remark that the universality in this theorem still holds if we restrict the affine layers of $\INNHNODE$ to identity except the last one, which is denoted by $W_1$ in Definition \ref{def: INNM} (see Proposition \ref{prop: strong ver. of NODE universality}.
Examples of $\mathcal{H}$ include the MLP with finite weights and Lipschitz-continuous activation functions such as ReLU activation \citep{LeCunDeep2015,ChenNeural2018a}, as well as the \emph{Lipschitz Networks} \citep[Theorem~3]{AnilSorting2019}.

\subsection{Sum-of-Squares Polynomial Flows (SoS Flows)} \label{subsec: SOS}
The sum-of-squares polynomial flow (SoS flow) \citep{DBLP:conf/icml/JainiSY19} is an important example of the flow layer for INNs (see also Section \ref{appendix: SOS}).
Here, we consider a special class of SoS flow layers $\HSoS$ where only the last dimension is converted (for the general description of SoS flow layers, see Section \ref{appendix: SOS}).
\begin{definition}
Let $\mathcal{H}$ be a set of measurable functions on $\Re^{d-1}$.
For $c \in \Re$ and $h_1,\dots, h_k \in \mathcal{H}$, let
\[g(\x; c,h_1,\dots, h_k) := c + \int_0^{x_{d}} \sum_{l=0}^k h_l(\x_{\le d-1} )u^l du.\]
Then, we define $\HSoS$ to be the set of all maps of the form $\x \mapsto (\x_{\le d-1}, g(\x; c, h_1,\dots, h_k))$ where $k \ge 1$, $c \in \Re$, and $h_1,\dots, h_k \in \mathcal{H}$.
\end{definition}
Although the universality for SoS based INN was proved in \citet{DBLP:conf/icml/JainiSY19}, we prove a much stronger universality for the architecture (Proposition \ref{prop: universality for sos}):
\begin{theorem} \label{thm: universality for sos}
Let $r \ge 0$ and let $\mathcal{H}$ be a set of measurable functions on $\Re^{d-1}$.
Assume that all elements of $\mathcal{H}$ are locally $C^{r-1,1}$ if $r \ge 1$ or locally $L^\infty$ if $r=0$ and that $\mathcal{H}$ is a $W^{r,\infty}$-universal approximator for the set of $(d-1)$-variable polynomials.
Then, \(\INN{\HSoS}\) is a $W^{r,\infty}$-universal approximator for $\mathcal{D}^{\max(r,1)}$.
\end{theorem}
This theorem immediately follows from Proposition \ref{prop: universality for sos} and Theorem \ref{theorem:main:1}.
As a direct corollary of Theorem \ref{thm: universality for sos}, Corollary \ref{cor: distributional universality for INN}, and Proposition \ref{proposition:body:generalized-distributional-universality}, we have the following.
\begin{corollary}
Let us use the same notation as in Theorem~\ref{thm: universality for sos}.
Then, $\INN{\HSoS}$ is a $(\mathcal{P}^{\rm w}, \nu)$-distributional universal approximator for $\mathcal{P}$ for any $\nu \in \mathcal{P}_{\rm ab}$.
Moreover, if $r \ge 1$, $\INN{\HSoS}$ is a $(\mathcal{P}^{\rm TV}, \nu)$-distributional universal approximator for $\mathcal{P}_{\rm ab}$ for any $\nu \in \mathcal{P}_{\rm ab}$.
\end{corollary}
\subsection{Other Examples of Flow Layers}

Theorem~\ref{prop:body:acfinn-Lp} can be interpreted as providing a convenient criterion to check the universality of a \ARFINN{}: if the flow architecture \(\ARFINNFlow\) contains \ACFs{} (or even just \(\FSACFH\) with sufficiently expressive \(\ACFINNUniversalClass\)) as special cases, then \(\INN{\ARFINNFlow}\) is an \(L^p\)-universal approximator for $C^0(U,\ReD)$ for any open subset $U \subset \ReD$.
Such examples of \(\ARFINNFlow\) include the \emph{nonlinear squared flow} \citep{ZieglerLatent2019a}, \emph{Flow++} \citep{HoFlow2018}, and the \emph{neural autoregressive flow} \citep{HuangNeural2018b}.

The result may not immediately apply to the typical \emph{Glow} \citep{KingmaGlow2018} architecture for image data that uses the 1x1 invertible convolution layers and convolutional neural networks for the coupling layers.
However, the Glow architecture for non-image data \citep{ArdizzoneAnalyzing2018b,TeshimaFewshot2020} can also be interpreted as \(\INN{\ARFINNFlow}\) with ACF layers, and hence it is an \(L^p\)-universal approximator for $C^0(U,\ReD)$ for any open subset $U \subset \ReD$.
\section{Conclusion}
In this paper, we provided a general framework to analyze the theoretical representation power of a family of invertible function models.
The key idea is to simplify the problem of approximating a general \(C^r\)-diffeomorphism by decomposing it into a finite set of simpler invertible maps by using the structure theorem of the diffeomorphism group.

The general framework was applied to two representative architectures of INNs: the CF-INNs and the NODEs, and we showed the high representation power of these architectures contrary to their apparent limitations on expressiveness.

For future work, it is important to quantitatively evaluate how many flow layers are required to approximate a given target map to assess the efficiency of the approximation.
It includes exploring efficient approximation of well-behaved target functions (e.g., the subset of \(\CtwoDomainDiff\) consisting of bi-Lipschitz diffeomorphisms).
Also, comparing the approximation efficiency of different flow layer designs is an important issue.
We expect that answering these questions provides principled design choices of invertible models tailored for a given task.

\begin{ack}
\acknowledgmentContent{}
\end{ack}

\printbibliography

\newpage
\appendix
\global\csname @topnum\endcsname 0
\global\csname @botnum\endcsname 0

This is the Supplementary~Material for ``\titleSentenceCase{}.''
We provide the proofs for statements in the paper.

Table~\ref{tbl:abbreviation-table} is the list of abbreviations we use in the paper.
Tables~\ref{tbl:notation-table:generic} and \ref{tbl:notation-table:specific} summarize the symbols we employed in the paper.

\begin{table}[tbph]
  \caption{Abbreviations in the paper}
  \label{tbl:abbreviation-table}
  \centering
  \begin{tabular}{ll}
    \toprule
    Abbreviation & Meaning \\
    \midrule
    INN & Invertible neural network\\
    CF-INN & Invertible neural network based on coupling flow\\
    IAF & Inverse autoregressive flow\\
    DSF & Deep sigmoidal flow\\
    SoS & Sum-of-squares polynomial flow\\
    MLP & Multi-layer perceptron\\
    NODE & Neural ordinary differential equation\\
    \bottomrule
  \end{tabular}
\end{table}

\begin{table}[tbph]
  \caption{Notation table (part 1 of 2)}
  \label{tbl:notation-table:generic}
  \centering
  \begin{tabular}{ll}
    \toprule
    Notation & Meaning \\
    \midrule
    $\Re$ & Set of all real numbers \\
    $\Na$ & Set of all positive integers \\
    $[n]$ & Set $\{1,2, \dots, n\}$ \\
    $\Euclideannorm{\cdot}$ & Euclidean norm\\
    $\opnorm{\cdot}$ & Operator norm\\
    $\LpKnorm{\cdot}$ & $L^p$-norm ($p\in [1,\infty)$) on a subset $K\subset \mathbb{R}^d$\\
    \(\Indicator{A}\) & Indicator (characteristic) function of \(A\) \\
    \(\Identity\) & Identity map \\
    \(\supp{}\) & Support of a map or measure\\
    $Df(x)$ & Jacobian matrix of $f$ at $x$ \\
    \bottomrule
  \end{tabular}
\end{table}

\begin{table}[tbph]
  \caption{Notation table (part 2 of 2)}
  \label{tbl:notation-table:specific}
  \centering
  \begin{tabular}{ll}
    \toprule
    Notation & Meaning \\
    \midrule
    CF, $\AutoRegressive{k}{\tau}{\theta}$ & Coupling flow \\
    ACF, $\Aff{k}{s}{t}$ & Affine coupling flow\\
    $\ACFINNUniversalClass$ & Generic notation for a set of functions from $\mathbb{R}^{d-1}$ to $\mathbb{R}$\\
    $\FSACFH, \Aff{d-1}{s}{t}$ & $\ACFINNUniversalClass$-single-coordinate affine coupling flows (\(s, t \in \ACFINNUniversalClass\))\\
    \(\IVP{f}{\x}{t}\) & The (unique) solution to an initial value problem evaluated at $t$ \\
    \(\ODEFlowEnds{\mathcal{F}}\) & Set of NODEs obtained from the Lipschitz continuous vector fields $\mathcal{F}$ \\
    $\ARFINNFlow$ & Generic notation for a set of invertible functions\\
    $\INN{\ARFINNFlow}$ & Set of all invertible neural networks based on \(\ARFINNFlow\)\\
    \midrule
    $d \in \Na$ & Dimensionality of the input/output Euclidean space\\
    $\ell \in \{0\} \cup \Na$ & Differentiability of the model\\
    $\CrDomainDiff$ & Set of all $C^r$-diffeomorphisms with \(C^r\)-diffeomorphic domains\\
    $\DcRDr$ $(1\leq r\leq \infty)$ & Group of compactly-supported $C^r$-diffeomorphisms (on \(\ReD\))\\
    $\FlowEnds{r}$ & Set of all flow endpoints in $\DcRDr$ \\
    $\Triangular$ & Set of all 
    $C^\infty$-increasing triangular mappings\\
    \(\CrOneDimTriangular\) &Set of all $C^r$-single-coordinate transformations\\
    \midrule
    $\mathfrak{S}_d$ & Set of all permutations of variables of $\ReD$\\
    \(\FGL\) & Set of all regular real matrices of size $d$ \\
    \(\FLin \) & Set of all affine transformations, i.e., \( \{\x \mapsto A\x + b: A \in \FGL, b \in \ReD\}\) \\
    \midrule
    $C^r$ & $r$-times continuously differentiable\\
    $C^{r,\alpha}$ & $C^r$ and any $k$-th derivative with $|k|=r$ is $\alpha$-H\"older continuous \\
    $C^r(\Re^m)$ & Set of all $C^r$ functions on $\Re^m$ equipped with local Sobolev topology \\
    $C^\infty_c(\ReD)$ & Set of all compactly-supported $C^\infty$ functions on $\ReD$ \\
    $B_{\rm loc}(\ReD, \Re^m)$ & Set of all locally bounded measurable maps from $\ReD$ to $\Re^m$ \\
    $C^r(U,\Re^n)$ & Set of all $\Re^n$-valued $C^r$ maps on $U$ \\
    $W^{r,p}_{\rm loc}(U, \mathbb{R}^n)$ & $\mathbb{R}^n$-valued local Sobolev space on $U$ \\
    $L^p_{\rm loc}(U, \Re^n)$ & $\mathbb{R}^n$-valued local Lebesgue space on $U$ (equal to $W^{0,p}_{\rm loc}(U, \mathbb{R}^n)$) \\
    \(\Lipsp\) & Set of all Lipschitz continuous maps from $\ReD$ to $\ReD$ \\
    $\Lipsp{} \cap C^r$ & Set of all Lipschitz and $C^r$ maps from $\ReD$ to $\ReD$ with $W^{r,\infty}_{\rm loc}$-topology \\
    \midrule
    $\mathcal{P}$ & Set of all probability measures on $\ReD$ \\
    $\mathcal{P}_{\rm ab}$ & Set of all absolutely continuous probability measures on $\ReD$ \\
    $\mathcal{P}^{\rm w}$ & $\mathcal{P}$ equipped with the weak convergence topology \\
    $\mathcal{P}^{\rm TV}$ & $\mathcal{P}$ equipped with the total variation topology \\
    \bottomrule
  \end{tabular}
\end{table}

\section{Locally bounded maps and piecewise diffeomorphisms}
\label{sec:appendix:piecewise diffeo}
In this section, we provide the notions of locally-ness and piecewise-ness. These notions are used to state the regularity conditions on the invertible layers $\mathcal{G}$ in Theorem~\ref{theorem:main:1} and to prove the results in Section~\ref{appendix: compatibility of approximation and composition}.

\subsection{Definition of locally-ness}
Here, we provide the definition of ``locally'' for functions.

\begin{definition}[locally bounded maps]
Let ${\bf P}$ be a property of functions such as boundedness.
Let $f$ be a map from $\Re^m$ to $\Re^n$.
We say $f$ is \emph{locally {\bf P}} if for each point $\x \in \Re^m$, there exists an open neighborhood $U$ of $\x$ such that $f$ has property ${\bf P}$ on $U$.
\end{definition}
The boundedness is a typical example of ${\bf P}$.
We easily see that a continuous function is locally bounded.

\subsection{Definition and properties of piecewise \texorpdfstring{$C^r$}{Cr}-mappings}
In this section, we define the notion of piecewise properties of functions, for example, piecewise $C^r$-functions.
Examples of piecewise $C^r$-diffeomorphisms appearing in this paper include the \(\FSACFH\) with \(\ACFINNUniversalClass\) being MLPs with ReLU activation.  
We first introduce the notion of \emph{piecewise properties}.
\begin{definition}\label{def:piecewiseC1}
Let ${\bf P}$ be a property of functions such as continuous, $C^r$, $C^{r, \alpha}$, and Lipschitz.
Let $f:\Re^m\rightarrow\Re^n$ be a map.  
We say $f$ is a \emph{piecewise {\bf P}-map} if there exists a mutually disjoint family of (at most countable)
open subsets $\{V_i\}_{i\in I}$ such that
\begin{itemize}
    \item ${\rm vol}(\Re^m\setminus U_f)=0$,
    \item for any $i\in I$, there exists an open subset $W_i$  containing the closure $\overline{V_i}$ of $V_i$, and a map $\tilde{f}_i:W_i\rightarrow\Re^n$ with the property {\bf P} such that $\tilde{f}_i|_{V_i}=f|_{V_i}$, and
    \item for any compact subset $K$, $\#\{i\in I : V_i\cap K\neq \emptyset\}<\infty$.
\end{itemize}
where \(\#(\cdot)\) denotes the cardinality of a set, and we define
\[U_f:=\bigsqcup_{i\in I} V_i.\]
\end{definition}
Although there exist several definitions of piecewise functions, we introduce a generalized definition for our purpose.
We remark that we here do not assume that piecewise $C^r$-maps are continuous everywhere and thus they might have discontinuous points.
We also remark that piecewise continuous mappings are essentially locally bounded in the sense that for any compact subset $K\subset\ReD$, ${\rm ess.sup}_{K}\Vert f\Vert=\supRangenorm{K\cap U_f}{f} <\infty$.

We define the notion of piecewise $C^r$-diffeomorphisms as follows.
\begin{definition}[Piecewise $C^r$-diffeomorphisms]\label{def:piecewiseC1diffeo}
\label{def: piecewise C1 diff}
Let $f:\ReD\rightarrow\ReD$ be a piecewise $C^r$-map.  We say $f$ is a \emph{piecewise $C^r$-diffeomorphism} if we can choose $\{V_i\}_{i\in I}$ and $\{\tilde{f}_i:W_i \to \ReD \}_{i\in I}$ in Definition~\ref{def:piecewiseC1} so that they additionally satisfy the following conditions:
\begin{enumerate}
    \item the image of a nullset (i.e., a Lebesgue-measurable subset of $\ReD$ whose measure is $0$) via $f$ is also a nullset,
    \item $f|_{U_f}$ is injective, 
    \item for $i\in I$, $\tilde{f}_i$ is a $C^r$-diffeomorphism from $W_i$ onto $\tilde{f}_i(W_i)$, 
    \item ${\rm vol}\left(\ReD\setminus f(U_f)\right)=0$, and
    \item for any compact subset $K$, $\#\{i\in I : f(V_i)\cap K\neq \emptyset\}<\infty$.
\end{enumerate}
\end{definition}
We summarize the basic properties of piecewise $C^r$-diffeomorphisms in the proposition below.
Note that for a piecewise $C^r$-diffeomorphism $f$, $Df$ is defined almost everywhere since its value is determined on $U_f$ (hence so is its determinant $|Df|$).

\begin{proposition}[Basic Properties of Piecewise $C^r$-diffeomorphisms]
\label{prop: basic properties}
Let $r \ge 1$ be a positive integer.
Let $f$ be a piecewise $C^r$-diffeomorphism. Then, we have the following:
\begin{enumerate}
    \item There exists a piecewise $C^r$-diffeomorphism $f^\dagger$ such that $f( f^\dagger(x))=x$ for $x\in U_{f^\dagger}$ and $f^\dagger(f(y))=y$ for $y\in U_f$\label{existence of inverse}.
    \item For any $h\in L^1$, we have $\int h(x) dx =\int h(f(x))|Df(x)|dx$. \label{change of var}
    \item For any compact subset $K$, $f^{-1}(K)\cap U_f$ is a bounded subset.\label{compfinvbdd}
    \item For any nullset $F$, then $f^{-1}(F)$ is also a nullset.  \label{finvzero}
    \item For any measurable set $E$ and any compact set $K$, $f^{-1}(E\cap K)$ has a finite volume.\label{finvfin}
    \item For any piecewise $C^r$-map (resp. piecewise Lipschitz map, piecewise $C^r$-diffeomorphism) $g$ , the composition $g \circ f$ is also a piecewise $C^r$-map (resp. piecewise Lipschitz map, piecewise $C^r$-diffeomorphism). \label{composition}
\end{enumerate}
\end{proposition}

\begin{proof}
Let $\{V_i\}_{i\in I}$ and $\{\tilde{f}_i:W_i \to \ReD\}_{i\in I}$ be as in Definition~\ref{def:piecewiseC1diffeo}.

\textit{Proof of \ref{existence of inverse} }:
First we note that since $f|_{V_i}$ is a restriction of the diffeomorphism $\tilde{f}_i$, $f(V_i)$ is an open set and $f|_{V_i}^{-1}$ is a well-defined $C^r$-function on $f(V_i)$.
We also note that since $f|_{U_f}$ is injective, we have $f(U_f) = \bigsqcup_{i\in I} f(V_i)$.
Fix $a\in\ReD$.
We define $f^\dagger(x)=a$ for $x\in \ReD\setminus f(U_f)$ and define $f^\dagger(x):=f|_{V_i}^{-1}(x)$ for $x\in f(V_i)$.
Then, $f^\dagger$ is a piecewise $C^r$-mapping with respect to the family of pairwise disjoint open subsets $\{f(V_i)\}_{i\in I}$, and satisfies the conditions for a piecewise $C^r$-diffeomorphism. 

\textit{Proof of \ref{change of var} }: 
It follows by the following computation:
\begin{align*}
    \int h(x) dx&=\int_{f(U_f)}h(x)dx\\
    &=\sum_{i\in I}\int_{f(V_i)}h(x) dx\\
    &=\sum_{i\in I}\int_{V_i}h(f(x))|Df(x)|dx=\int h(f(x))|Df(x)|dx.
\end{align*}

\textit{Proof of \ref{compfinvbdd}}
It suffices to show that $f^{-1}(K)\cap U_f$ is covered by finitely many compact subsets. We remark that only finitely many $V_i$'s intersect with $f^{-1}(K)$. If not, infinitely many $f(V_i)$'s intersect with $f(f^{-1}(K))= K$, which contradicts the definition of piecewise $C^r$-diffeomorphisms. Let $I_0\subset I$ be a finite subset composed of $i\in I$ such that $V_i$ intersects with $f^{-1}(K)$. For $i\in I_0$, we define a compact subset $F_i:=\tilde{f}_i^{-1}(\tilde{f}_i(\overline{V_i})\cap K)$.  Then we see that $f^{-1}(K)\cap U_f$ is contained in $\cup_{i\in I_0}F_i$. 

\textit{Proof of \ref{finvzero} }: 
It suffices to show that for any compact subset $K$, the volume of $f^{-1}(F)\cap K$ is zero.
By applying \ref{change of var} to the case $h=\Indicator{F}$, we see that
\[\int_{f^{-1}(F)}|Df(x)|dx=0.\]
For  $n>0$, let $E_n:=f^{-1}(F)\cap K \cap \{x\in\ReD:|Df(x)|\ge 1/n\}$. Then we have
\[\frac{{\rm vol}(E_n)}{n}\le \int_{E_n}|Df(x)|dx \le  \int_{f^{-1}(F)}|Df(x)|dx=0,\]
thus ${\rm vol}(K\cap f^{-1}(F))=\lim_{n\rightarrow\infty}{\rm vol}(E_n)=0$

\textit{Proof of \ref{finvfin} }:
By applying \ref{change of var} to the case $h=\Indicator{E\cap K}$, we see that
\[\int_{f^{-1}(E\cap K)}|Df(x)|dx= {\rm vol}(E\cap K).\]

Let $F$ be a closure of $f^{-1}(K)\cap U_f$.  By \ref{compfinvbdd}, $F$ is a compact subset.  Let $I_0:=\{i\in I : F\cap V_i\neq\emptyset\}$ be a finite subset. Then we have 
\begin{align*}
    C&:=\inf_{f^{-1}(K)\cap U_f}|Df|\\
    &\ge \inf_{i\in I_0}\inf_{F\cap\overline{V_i}}|D\tilde{f}_i| >0.
\end{align*}
Thus,
\begin{align*}
\int_{f^{-1}(E\cap K)\cap U_f}|Df(x)|dx 
\geq C{\rm vol}(f^{-1}(E\cap K)),
\end{align*}
where the last equality follows from ${\rm vol}(f^{-1}(E\cap K)\setminus U_f)=0$.
Thus we have ${\rm vol}(f^{-1}(E \cap K))<\infty$

\textit{Proof of \ref{composition} }: 
We first assume that $g$ is a piecewise $C^r$-mapping and prove that $g\circ f$ is a piecewise $C^r$-mapping.
We denote by $\{V_i\}_{i\in I}$, $\{V'_j\}_{j\in J}$ the disjoint open-set families associated with $f$ and $g$, respectively. 
Let $V_{ij}:=f^{-1}(f(V_i)\cap V'_j)\cap U_f$.
We prove $\{V_{ij}\}_{(i, j)\in I \times J}$ is the open-set family associated with $g\circ f$
(i.e., $\{V_{ij}\}$ satisfies the conditions of Definition~\ref{def:piecewiseC1}).
Let $U_{g\circ f}:=\cup_{i,j}V_{ij}=f^{-1}(U_g\cap f(U_f))\cap U_f$.
Then, we have 
\[\ReD \setminus U_{g\circ f}=f^{-1}((\ReD \setminus U_g) \cup (\ReD \setminus f(U_f)))\cup (\ReD \setminus U_f).\]
Since ${\rm vol} (\ReD \setminus U_g) = 0$ and ${\rm vol} (\ReD \setminus f(U_f)) = 0$, we have 
\[{\rm vol} (f^{-1}((\ReD \setminus U_g) \cup (\ReD \setminus f(U_f)))) = 0\]
by \ref{finvzero} of Proposition~\ref{prop: basic properties}.
In addition, since ${\rm vol} (\ReD \setminus U_f) = 0$, we have ${\rm vol} (\ReD\setminus U_{g\circ f}) = 0$.
That is, the first condition is satisfied.
For the second condition, we denote by $\tilde{f}_i$ (resp. $\tilde{g}_j$) the extension of $f|_{V_i}$ (resp. $g|_{V'_j}$).
Then, $\tilde{g}_j\circ\tilde{f}_i$ is an extension of $g\circ f|_{V_{ij}}$ on each $V_{ij}$.
Finally, to prove the third condition, we take an arbitrary compact subset $K$ and prove that $\#\{(i,j)\in I\times J : K\cap V_{ij}\neq \emptyset\}<\infty$.
Indeed, since $f$ is a piecewise $C^r$-diffeomorphism, $f(U_f \cap K)$ is a bounded subset by \ref{compfinvbdd} of Proposition~\ref{prop: basic properties}.
Hence, $M := \overline{f(U_f \cap K)}$ is compact.
Since $f$ is a piecewise $C^r$-diffeomorphism, we have
\[\#\{i \in I \mid M \cap f(V_i) \not = \emptyset\} < \infty.\]
Similarly, since $g$ is a piecewise $C^r$-mapping, we have 
\[\#\{j \in J \mid M \cap V'_j \not = \emptyset \} < \infty.\]
Therefore, the number of pairs $(i, j)$ satisfying $M \cap f(V_i) \cap V'_j \not = \emptyset$ is also finite.
Note that $U_f \cap K \cap V_i \cap f^{-1}(V_j) = K \cap V_{ij}$.
Therefore, by applying the inverse of $f$ (see \ref{existence of inverse} of Proposition~\ref{prop: basic properties}),
we obtain $\#\{(i, j) \mid K \cap V_{ij} \not = \emptyset\} < \infty$.
It means the third condition is satisfied.
Combining the above discussions so far, we conclude that $g\circ f$ is a piecewise $C^r$-mapping.
In the case where $g$ is a piecewise Lipschitz, the proof is the same as above.

Next, we prove $f\circ g$ is a piecewise $C^r$-diffeomorphism when $g$ is a piecewise $C^r$-diffeomorphism.
We check the conditions in Definition \ref{def: piecewise C1 diff}.
The first, second, and third conditions follow by definition.  For the third condition, since
\[ \ReD\setminus (g\circ f(U_{g\circ f}))=\big(\ReD\setminus g(U_g)\big)\cup \big(\ReD\setminus g\big(f(U_f)\big)\subset \ReD\setminus g(f(U_f)\cap U_g),\]
it suffices to show that the volume of $\ReD\setminus g(f(U_f)\cap U_g)$ is zero.
In fact, by the injectivity of $g$ on $U_g$, we have
\[ g(f(U_f)\cap U_g)=g(U_g)\setminus g(U_g\setminus f(U_f)). \]
Thus, we have 
\[\ReD\setminus g(f(U_f)\cap U_g)=(\ReD\setminus g(U_g))\cup g(U_g\setminus f(U_f)).\]
By definition of $C^r$-diffeomorphism, we conclude $\ReD\setminus g(f(U_f)\cap U_g)$ is a null set.
For the fourth condition, let $K$ be a compact subset.  
Let $K$ be a compact set. Suppose $(i, j)\in I\times J$ satisfies $K \cap (g\circ f)(V_{ij}) \not =\emptyset$.
Since $f(V_{ij}) \subset V'_j$, we have 
\begin{equation}
    K \cap g(V'_{j}) \not =\emptyset. \label{eq:j-condition}
\end{equation}
Since $g$ is a piecewise $C^r$-diffeomorphism, there exist finitely many $j$'s satisfying~\eqref{eq:j-condition}.
On the other hand, by applying the inverse of $g$, we have $g^{-1}(K)\cap U_g \cap f(V_{ij}) \not=\emptyset$, which implies
\begin{equation}
    \overline{g^{-1}(K)\cap U_g} \cap f(V_{i}) \not=\emptyset. \label{eq:i-condition}
\end{equation}
Note that $\overline{g^{-1}(K)\cap U_g}$ is compact.
Therefore, using the fact that $f$ is a piecewise $C^r$-diffeomorphism, we see that there exist finitely many $i\in I$ satisfying~\eqref{eq:i-condition}.
Therefore, we have $\#\{ (i, j) \in I \times J \mid K \cap (g\circ f)(V_{ij}) \not =\emptyset\} < \infty$.

\end{proof}
For a measurable mapping $f: \R^m \to \R^n$ and $R>0$, we define a measurable set
\[
\mathcal{L}(R;f):=\{x\in \Re^m : \Vert f(x)-f(y) \Vert > R\Vert x-y\Vert \text{ for some $y\in U_f$}\}.
\]
Then, we have the following proposition:
\begin{proposition}
\label{prop: weak lipschitz}
Let $f:\Re^m\rightarrow\Re^n$ be a piecewise Lipschitz function. 
Assume $f$ is linearly increasing, namely, there exists $a,b>0$ such that $\Vert f(x)\Vert <a\Vert x\Vert + b$ for any $x\in \R^m$.
Then for any compact subset $K\subset \R^m$, ${\rm vol}(\mathcal{L}(R;f)\cap K)\rightarrow 0$ as $R\rightarrow\infty$.
\end{proposition}
\begin{proof}
Let $\{V_i\}_{i\in I}$ be the disjoint family of open sets associated with $f$ satisfying the properties of Definition~\ref{def:piecewiseC1}.
Let $B$ be an $m$-dimensional open ball of radius $r$ containing $K$. Fix an arbitrary $\varepsilon>0$. 
Let $C:=\sup_{x\in \overline{B}}\Vert f (x)\Vert$.
Because the linearly increasing condition of $f$ implies its locally boundedness, we have $C<\infty$.
For $\delta>0$, we define 
\[W_\delta:=\{x\in \overline{B}: {\rm dist}\left(x,\partial{U_f}\cup\partial{B})\right)<\delta\},\] where ${\rm dist}(x,S):=\inf_{y\in S}\{\Vert x-y\Vert\}$.  
By the continuity of the Lebesgue measure, we have $\lim_{\delta \to 0} {\rm vol}(W_\delta) = 0$.
Therefore, we can choose $\delta > 0$ so that ${\rm vol}(W_\delta)<\varepsilon$ holds. 

We claim that 
\[L:=\sup_{(x,y)\in K\times (\Re^m\setminus B)}\frac{\|f(x)-f(y)\|}{\|x-y\|}\]
is finite. In fact, let $r':=\inf_{(x, y) \in K \times (\R^m \setminus B)}\|x-y\|$. Then for $x\in K$ and $y\notin B$, we have
\begin{align*}
    \frac{\|f(x)-f(y)\|}{\|x-y\|}
    &\le \frac{\|f(x)\|+ \|f(y)\|}{\|x-y\|} \\
    &\le \frac{a \|x\| + a\|y\| + 2b }{\|x-y\|} \\    
    &\le \frac{a \|x\| + a(\|x - y\| + \|x\|)+ 2b }{\|x-y\|} \\
    &\le a+\frac{2a\|x\|+2b}{\|x-y\|}\\
    &< a+\frac{2ar+2b}{r'}.
\end{align*}
Thus, $L$ is finite.

Due to the piecewise Lipschitz-ness of $f$, $\overline{B}$ intersects with finitely many $V_i$'s.
It implies that $f|_{B\setminus W_{\delta/2}}$ is a Lipschitz function.
Put $L_\delta>0$ as the Lipschitz constant of $f|_{B\setminus W_{\delta/2}}$. 

For any $R>\max(L, L_\delta, 4C/\delta)$, we claim that $\mathcal{L}(R;f)\cap K$ is contained in $W_\delta$.
To prove it, we show that $x\not \in \mathcal{L} (R;f)$ when $x \in K\setminus W_{\delta}$.
Take arbitrary $y\in \mathbb{R}^m$. (Case 1) When $y\not \in B$, since $x\in K$, we have $\frac{\|f(x)-f(y)\|}{\|x-y\|}\leq L$ by the definition of $L$.
(Case 2) When $y \in B\setminus W_{\delta/2}$, since $x\in K\setminus W_{\delta} \subset B\setminus W_{\delta/2}$, we have $\frac{\|f(x)-f(y)\|}{\|x-y\|}\leq L_{\delta}$ by the definition of $L_{\delta}$.
(Case 3) When $y\in B\cap W_{\delta/2}$, we have $\|x-y\|\geq \frac{\delta}{2}$ because $x\not \in W_{\delta}$. Thus,
\begin{align*}
    \frac{\|f(x)-f(y)\|}{\|x-y\|} \leq \frac{\|f(x)\| + \|f(y)\|}{\delta/2} \leq \frac{C + C}{\delta/2} \leq \frac{4C}{\delta}.
\end{align*}
Combining these three cases, we conclude that $x\not \in \mathcal{L}(R;f)$.
Thus we have ${\rm vol}(\mathcal{L}(R;f)\cap K)<\varepsilon$, namely, we conclude ${\rm vol}(\mathcal{L}(R;f)\cap K)\rightarrow 0$ as $R\rightarrow\infty$.
\end{proof}
\begin{remark}\label{rem:linealy_increasing}
   The linearly increasing condition is important to prove our main theorem. Our approximation targets are compactly supported diffeomorphisms, affine transformations, and the discontinuous \ACFs{} appeared in Section \ref{sec:appendix:lp ACFINN approx general}, all of which satisfy the linearly increasing condition.
\end{remark}

\section{Compatibility of approximation and composition}
\label{appendix: compatibility of approximation and composition}
In this section, we prove the following lemmas.
It enables the component-wise approximation, i.e., approximating a composition of some transformations by approximating each constituent and composing them.
The justification of this procedure is not trivial and requires a fine mathematical argument.
The results here build on the terminologies and the propositions for piecewise $C^1$-diffeomorphisms presented in Section~\ref{sec:appendix:piecewise diffeo}.

\begin{lemma}\label{prop: compatibility of approximation}
Let $p=[1,\infty)$.
Let $m \ge 1$ and let $\mathcal{F}$ be the set of $\Re^m$-valued piecewise Lipschitz mappings. 
Let  $\mathcal{G}$ be the set of piecewise $C^1$-diffeomorphisms on $\mathbb{R}^d$.
Let $\mathcal{F}_0 \subset \mathcal{F}$ and $\mathcal{G}_0 \subset \mathcal{G}$ be the subsets composed of linearly increasing mappings.
Here, a function $f$ on $\ReD$ is linearly increasing if there exists $a,b >0$ such that $\|f(x) \| < a\| x \| + b$ for all $x\in \ReD$.
Then, the map
\begin{align}
     \mathcal{C}: \mathcal{F} \times \mathcal{G}^k \longrightarrow \mathcal{F}; (h, f_1,\dots,f_k) \mapsto h\circ f_1 \circ \cdots \circ f_k
\end{align}
is continuous at any point of $\mathcal{F}_0 \times \mathcal{G}_0^k$ with respect to the relative topology of $W^{0,p}_{\rm loc}(\mathbb{R}^d, \mathbb{R}^m) \times W^{0,p}_{\rm loc}(\mathbb{R}^d, \mathbb{R}^d)^k$.
\end{lemma}
\begin{proof}
Since $\mathcal{C}(\mathcal{F}_0 \times \mathcal{G}_0) \subset \mathcal{F}_0$ (see the statement \ref{composition} of Proposition \ref{prop: basic properties}), the lemma follows from the case $k=1$ via the mathematical induction.
Thus, we only treat the case $k=1$.
Let $(F_2, G_2) \in \mathcal{F}_0 \times \mathcal{G}_0$.
Then, it suffices to show that for any $\varepsilon>0$ and compact set $K\subset\mathbb{R}^d$, 
there exist $\delta>0$ and compact set $K_0 \subset \mathbb{R}^d$ such that for any $(F_1, G_1) \in \mathcal{F} \times \mathcal{G}$ satisfying $\WspKnorm{0}{p}{K_0}{G_2 - G_1}, \WspKnorm{0}{p}{K_0}{F_2 - F_1} < \delta$, we have
\[\WspKnorm{0}{p}{K}{F_2\circ G_2-F_1\circ G_1} < \varepsilon.\]
Fix arbitrary $\varepsilon>0$ and compact set $K\subset\mathbb{R}^d$. 
Put $K':=\overline{G_2(K\cap U_{G_2})}$. 
Then, since $G_2(K\cap U_{G_2})$ is bounded (see the remark under Definition~\ref{def:piecewiseC1}), $K'$ is compact.
We claim that there exists $R>0$ such that
\[\vol{G_2^{-1}\left(\mathcal{L}(R;F_2)\cap K'\right)}^{1/p}<\frac{\varepsilon}{3\underset{K'}{\rm ess.sup}\Vert F_2\Vert},\]
which can be confirmed as follows. Take an increasing sequence $R_n>0$ $(n\geq 1)$ satisfying $\lim_{n\to \infty}R_n=\infty$.
Let $B_n := \mathcal{L}(R_n; F_2)\cap K'$ and $A_n := G_2^{-1}(B_n)$.
Then, from Proposition~\ref{prop: weak lipschitz}, we have $\vol{B_n}\to 0$, which implies  
$\vol{\bigcap_{n=1}^\infty B_n}=0$.
By Proposition~\ref{prop: basic properties} (\ref{finvzero}),
we have $\vol{\bigcap_{n=1}^\infty A_n} = \vol{G_2^{-1}(\bigcap_{n=1}^\infty B_n)}=0$.
By Proposition~\ref{prop: basic properties} (\ref{finvfin}), we have $\vol{A_1}=\vol{G_2^{-1}(B_1)}<\infty$.
Recall that if a decreasing sequence $\{S_n\}_{n=1}^\infty$ of measurable sets satisfies $\vol{S_1}<\infty$ and $\vol{\bigcap_{n=1}^\infty S_n}=0$, 
then $\lim_{n \to \infty}\vol{S_n} = 0$.
Therefore, we obtain $\lim_{n\to \infty}\vol{A_n}=0$ and we have the assertion of the claim.

Take  $G_1\in\mathcal{G}$ such that 
\[\WspKnorm{0}{p}{K_0}{G_2-G_1} <\frac{\varepsilon}{3R}.\]

Put $S:=G_2^{-1}\left(\mathcal{L}(R;F_2)\cap K'\right)$, and define a compact subset $K'':=\overline{(G_1^\dagger)^{-1}(K)\cap U_{G_1^\dagger}}$. Here, the compactness of $K''$ follows from Proposition~\ref{prop: basic properties} (\ref{compfinvbdd}).   
Next, we take $F_1\in\mathcal{F}$ such that
\[\Vert F_2-F_1\Vert_{p,K''} <\frac{\varepsilon}{3\underset{{ (G_1^\dagger)^{-1}(K)}}{\rm ess.sup}|\det(DG_1^\dagger)|}\]
where $G_1^\dagger$ is a piecewise $C^1$-diffeomorphism defined by Proposition \ref{prop: basic properties} (\ref{existence of inverse}).
Therefore, if we take 
\[\delta:=\min\left( \frac{\varepsilon}{3\underset{K'}{\rm ess.sup}\Vert F_2\Vert}, \frac{\varepsilon}{3R}\right) \]
and $K_0 := K \cup K''$, then we have
\begin{align*}
&\WspKnorm{0}{p}{K}{F_2\circ G_2-F_1\circ G_1}\\
&\le\WspKnorm{0}{p}{K_0}{F_2\circ G_2-F_2\circ G_1} + \WspKnorm{0}{p}{K_0}{F_2\circ G_1-F_1\circ G_1}\\
&\le\WspKnorm{0}{p}{K}{(F_2\circ G_2-F_2\circ G_1)\mathbf{1}_S}+ \WspKnorm{0}{p}{K}{(F_2\circ G_2-F_2\circ G_1)\mathbf{1}_{K\setminus S} }\\
& \hspace{11pt}+ \underset{{(G_1^\dagger)^{-1}(K)}}{\rm ess.sup}|\det(DG_1^{\dagger})\WspKnorm{0}{p}{K}{F_2-F_1}\\
&<\varepsilon.
\end{align*}
\end{proof}

\begin{lemma}\label{prop: compatibility of approximation p infity}
Let $m \ge 1$ and let $\mathcal{F}:= W^{0,\infty}(\ReD, \Re^m)$.
Let  $\mathcal{G} $ be a subset $W^{0,\infty}(\ReD, \ReD)$ whose inverse images of any null sets are again null sets.
Let $\mathcal{F}_0 \subset \mathcal{F}$ and $\mathcal{G}_0 \subset \mathcal{G}$ be the subsets composed of continuous mappings.
Then, the map
\begin{align}
     \mathcal{C}: \mathcal{F} \times \mathcal{G}^k \longrightarrow \mathcal{F}; (h, f_1,\dots,f_k) \mapsto h\circ f_1 \circ \cdots \circ f_k
\end{align}
is continuous at any point of $\mathcal{F}_0 \times \mathcal{G}_0^k$ with respect to the relative topology of $W^{0,\infty}_{\rm loc}(\mathbb{R}^d, \mathbb{R}^m) \times W^{0,\infty}_{\rm loc}(\mathbb{R}^d, \mathbb{R}^d)^k$.
\end{lemma}
\begin{proof}
Since $\mathcal{C}(\mathcal{F}_0 \times \mathcal{G}_0) \subset \mathcal{F}_0$ (see the statement \ref{composition} of Proposition \ref{prop: basic properties}), the proposition follows from the case $k=1$ via the mathematical induction.
Thus, we only treat the case $k=1$.
Let $(F_2, G_2) \in \mathcal{F}_0 \times \mathcal{G}_0$.
Then, it suffices to show that for any $\varepsilon>0$ and compact set $K\subset\mathbb{R}^d$, 
there exist $\delta>0$ and compact set $K_0 \subset \mathbb{R}^d$ such that for any $(F_1, G_1) \in \mathcal{F} \times \mathcal{G}$ satisfying $\WspKnorm{0}{\infty}{K_0}{G_2 - G_1}, \WspKnorm{0}{\infty}{K_0}{F_2 - F_1} < \delta$, we have
\[\WspKnorm{0}{\infty}{K}{F_2\circ G_2-F_1\circ G_1} < \varepsilon.\]
Take any positive number $\epsilon>0$ and compact set $K\subset\R^d$. 
Put $r:=\max_{K}|G_2|$ (note that $G_2$ is continuous) and 
$K':=\{x\in \R^d: |x|\leq r+1\}$. 
Let $F_1\in \mathcal{F}$  satisfying 
\begin{align*}
{\rm vol}\{ x\in K' : |F_2(x)-F_1(x)| >  \epsilon / 2\} = 0.  \end{align*}
Since any continuous map is uniformly continuous on a compact set, we can take a positive number $\delta>0$ such that for any $x, y\in K'$ with $|x-y|<\delta$, 
\[ |F_2(x)-F_2(y)|<\frac{\varepsilon}{2}. \]
From the assumption, we can take $G_1\in \mathcal{G}$ satisfying
\begin{align*} {\rm vol}\{x \in K : |G_2(x)-G_1(x)| > \min\{1,\delta\}\} = 0. \end{align*}
Since 
\[
 |F_2\circ G_2(x)-F_1\circ G_1(x)| \leq |F_2(G_2(x))-F_2(G_1(x))| + | F_2(G_1(x)) -F_1(G_1(x))|,
\]
we see that the set of $x \in K$ such that $\varepsilon < |F_2\circ G_2(x)-F_1\circ G_1(x)|$ is a null set.
Thus, we have
\[\WspKnorm{0}{\infty}{K}{F_2\circ G_2-F_1\circ G_1} < \varepsilon.\]
\end{proof}

Let $B_{\rm loc}(\ReD, \Re^m)$ be the linear space composed of locally bounded measurable maps from $\ReD$ to $\Re^m$.
We equip $B_{\rm loc}$ with the topology generated by the seminorms $\{\| \cdot \|_{\sup, K}\}_{K}$, where $K$ runs on the set of compact subsets of $\ReD$, and define for any $h \in B_{\rm loc}$, 
\[\|h\|_{\sup, K} := \sup_{x \in K} \|h(x)\|.\]
Then, we provide a similar result for the \(sup\)-norm case as follows:
\begin{lemma}\label{lem: sup compatibility of composition}
Let $m \ge 1$ and let $\mathcal{F}:= B_{\rm loc}(\ReD, \Re^m)$ and $\mathcal{G} $ be a subset $B_{\rm loc}(\ReD, \ReD)$.
Let $\mathcal{F}_0 \subset \mathcal{F}$ and $\mathcal{G}_0 \subset \mathcal{G}$ be the subsets composed of continuous mappings.
Then, the map
\begin{align}
     \mathcal{C}: \mathcal{F} \times \mathcal{G}^k \longrightarrow \mathcal{F}; (h, f_1,\dots,f_k) \mapsto h\circ f_1 \circ \cdots \circ f_k
\end{align}
is continuous at any point of $\mathcal{F}_0 \times \mathcal{G}_0^k$ with respect to the relative topology of $W^{0,\infty}_{\rm loc}(\mathbb{R}^d, \mathbb{R}^m) \times W^{0,\infty}_{\rm loc}(\mathbb{R}^d, \mathbb{R}^d)^k$.
\end{lemma}
\begin{proof}
We may assume $k=1$ and let $(F_2, G_2) \in \mathcal{F}_0 \times \mathcal{G}_0$ as in the proof of Lemma \ref{prop: compatibility of approximation p infity}.
Take any positive number $\epsilon>0$ and compact set $K\subset\R^d$. 
Put $r:=\max_{k\in K}|G_2(k)|$ and 
$K':=\{x\in \R^d: |x|\leq r+1\}$. 
Let $F_1\in \mathcal{F}$  satisfying 
\begin{align*}
\sup_{x\in K'}|F_2(x)-F_1(x)| \leq  \frac{\epsilon}{2}.  \end{align*}
Since any continuous map is uniformly continuous on a compact set, we can take a positive number $\delta>0$ such that for any $x, y\in K'$ with $|x-y|<\delta$, 
\[ |F_2(x)-F_2(y)|<\frac{\varepsilon}{2}. \]
Let $G_1\in \mathcal{G}$ satisfying
\begin{align*} \sup_{x\in K}|G_2(x)-G_1(x)|\leq \min\{1,\delta\}. \end{align*}
Then, it is clear that $G_2(K)\subset K'$ by the definition of $K'$.
Moreover, we have $G_1(K)\subset K'$. 
In fact, we have 
\begin{align*}
|G_1(k)| \leq \sup_{x\in K}|G_2(x)-G_1(x)|+|G_2(k)|\leq 1+r \quad (k\in K). 
\end{align*}
Then for any $x\in K$, we have
\begin{align*}
|F_2\circ G_2(x)-F_1\circ G_1(x)|
&\leq |F_2(G_2(x))-F_2(G_1(x))| + | F_2(G_1(x)) -F_1(G_1(x))|\\
&<\epsilon. 
\end{align*}
\end{proof}

Now, we provide a general result of compatibility of composition and approximation:
\begin{corollary}\label{cor: compatibility of approximation, higher derivatives}
Let $r\ge1$ and $p \in [1,\infty]$.
Let  $\mathcal{G}$ be the set $\ReD$-valued mappings.
Assume either of the following conditions:
\begin{enumerate}
    \item $1 \le p \le \infty$, $\mathcal{G}$ is composed of $C^{r}$ and piecewise $C^{r+1}$ diffeomorphisms on $\mathbb{R}^d$, and $\mathcal{G}_0 \subset \mathcal{G}$ is the subset composed of linearly increasing mappings.
    \item $p = \infty$, $\mathcal{G}$ is composed of locally $C^{r-1,1}$-mappings whose inverse image of nullsets are again nullsets, and $\mathcal{G}_0 \subset \mathcal{G}$ is $C^r$-mappings.
\end{enumerate}
Then, for any $k \ge 1$, the map
\begin{align}
     \mathcal{G}^k \longrightarrow \mathcal{G}; (f_1,\dots,f_k) \mapsto f_1 \circ \cdots \circ f_k
\end{align}
is continuous at any point of $\mathcal{G}_0^k$ with respect to the relative topology of $W^{r,p}_{\rm loc}(\mathbb{R}^d, \mathbb{R}^d )^{k}$. 
If $\mathcal{G} \subset B_{\rm loc}(\ReD, \ReD)$ and the subset $\mathcal{G}_0 \subset \mathcal{G}$ is composed of continuous mapping, we have a similar continuity of the composition with respect to the topology of $B_{\rm loc}(\ReD, \ReD)^k$.
\end{corollary}
\begin{proof}
The Leibniz rule and the chain rule hold for weak derivatives under the present condition (see \citet[Exercise~B.1.2]{McDuffJholomorphic2004} and \citet[Theorem~2.1.11]{ZiemerWeakly1989}).
Thus, it follows from Lemmas \ref{prop: compatibility of approximation} and \ref{prop: compatibility of approximation p infity}.
The last statement follows from Lemma \ref{lem: sup compatibility of composition} in the same way.
\end{proof}

\section{Proof of Distributional Universalities}

\subsection{Proof of Proposition~\ref{lem:body:distributional-universality}: From \texorpdfstring{\(L^p\)}{Lp}-universality to distributional universality}\label{appendix: from lp to dist}

Here, we prove Proposition~\ref{lem:appendix:lp to dist}, which corresponds to Proposition~\ref{lem:body:distributional-universality} in the main text.
We first include a proof that any probability measure on $\Re^m$ is arbitrarily approximated by an absolutely continuous probability measure in the weak convergence topology.
\begin{lemma}
\label{lem:appendix: abs aprox any}
Let $\mu \in \mathcal{P}$ be an arbitrary probability measure.
Then there exists a sequence $\{\mu_n\}_{n=1}^\infty \subset \mathcal{P}_{\rm ab}$ of absolutely continuous probability measures such that $\mu_n$ weakly converges to $\mu$.
\end{lemma}
\begin{proof}
Let $\phi$ be a compactly-supported positive bounded $C^\infty$ function such that $\int_{\Re^m} \phi(x) dx=1$ and $\mathrm{supp}(\phi) \subset \{x \in \Re^m : \|x\| \le B\}$ where $B > 0$.
For $t>0$, put $\phi_t(x):=t^{-m}\phi(x/t)$.
We define
\begin{align*}
   w_t(x)=\int_{\Re^m} \phi_t(x-y)d\mu(y).
\end{align*}
We prove that the absolutely continuous measure $w_tdx$ weakly converges to $\mu$ as $t\rightarrow 0$.
In fact, given an $L$-Lipschitz continuous function $f$ such that, we have
\begin{align*}
    \left|\int_{\Re^m} fw_t dx - \int fd\mu \right| 
    &= \left| \int \int_{\Re^m} \left(f(y+tx)-f(y)\right)\phi(x) dxd\mu(y) \right|\\
    &\le \int \int_{\Re^m} |f(y+tx)-f(y)|\phi(x)dx d\mu(y) \\
    &\le \int \int_{\Re^m} L t \|x\| \phi(x)dx d\mu(y) \\
    &\le L B t.
\end{align*}
Therefore, as $t\rightarrow 0$, we have
\[\int_{\Re^m} fw_t dx \rightarrow  \int fd\mu,\]
therefore, $\left\{w_{\frac{1}{n}}dx\right\}_n$ weakly converges to $\mu$.
\end{proof}

First, note that the larger \(p\), the stronger the notion of \(L^p\)-universality: if a model \(\INNModelGeneric\) is an \(L^p\)-universal approximator for \(\mathcal{F}\), it is also an $L^q$-universal approximator for \(\mathcal{F}\) for all $1 \leq q \leq p$.
In particular, we use this fact with \(q = 1\) in the following proof.
\begin{proposition}[Proposition~\ref{lem:body:distributional-universality} in the main text]\label{lem:appendix:lp to dist}
Let \(p \in [1, \infty)\).
Suppose \(\INNModelGeneric\) is an $L^p$-universal approximator for $\Triangular$. Then \(\INNModelGeneric\) is a $(\mathcal{P}^{\rm w}, \mu)$-distributional universal approximator for $\mathcal{P}$ for any $\mu \in \mathcal{P}_{\rm ab}$.
\end{proposition}
\begin{proof}
By Lemma \ref{lem:appendix: abs aprox any}, it suffices to prove that $\INNModelGeneric$ is a $(\mathcal{P}^{\rm w}, \mu)$-distributional universal approximator for $\mathcal{P}_{\rm ab}$ for any $\mu \in \mathcal{P}_{\rm ab}$. 
We denote by ${\rm BL}_1$ the set of bounded Lipschitz functions $f\colon{}\R^d\rightarrow \R$ satisfying  $\Vert f\Vert_{\sup, \R^d}+L_f\le 1$, where $L_f$ denotes the Lipschitz constant of $f$.
Let $\mu, \nu \in \mathcal{P}_{\rm ab}$ be absolutely continuous probability measures, and 
take any $\varepsilon>0$. 
By Theorem~11.3.3 in \citet{DudleyReal2002}, it suffices to show that there exists $g\in\INNModelGeneric$ such that
\[\beta(g_*\mu, \nu):=\sup_{f\in {\rm BL}_1}\left|\int_{\R^d}f\,dg_*\mu-f\,d\nu\right|<\varepsilon.\]
Let $p,q\in L^1(\R^d)$ be the density functions of $\mu$ and $\nu$ respectively.  Let $\phi\in L^1(\R^d)$ be a positive \(C^\infty\)-function such that $\int_{\R^d} \phi(x)dx=1$ (for example, the density function of the standard Gaussian distribution), and for $t>0$, put $\phi_t(x):=t^{-d}\phi(x/t)$.  We define  $\mu_t:=\phi_t*pdx$ and $\nu_t:=\phi_t*qdx$. Since both $\Vert \phi_{t}*p-p\Vert_{1,\R^d}$ and $\Vert \phi_{t}*q-q\Vert_{1,\R^d}$ converge to 0 as $t\rightarrow0$, there exists $t_0>0$ such that for any continuous mapping $G:\R^d\rightarrow \R^d$, 
\begin{align*}
\left|\int_{\R^d}f\,dG_*\mu_{t_0}-f\,dG_*\mu\right|&<\frac{\supRangenorm{\ReD}{f}\varepsilon}{5},\quad
\left|\int_{\R^d}f\,d\nu_{t_0}-f\,d\nu\right|<\frac{\supRangenorm{\ReD}{f}\varepsilon}{5}.
\end{align*}
By using Lemma \ref{existence of transformation for probability measures} below, there exists $T\in\Triangular$ such that $T_*\mu_{t_0}=\nu_{t_0}$.  Let $K\subset \R^d$ be a compact subset such that
\[1-\mu_{t_0}(K)<\frac{\varepsilon}{5}.\]
By the assumption, there exists  $g\in\INNModelGeneric$  such that 
\[\int_K |T(x)-g(x)|dx<\frac{\varepsilon}{5\sup_{x\in K}|\phi_{t_0}*p(x)|}.\]
Thus for any $f\in{\rm BL}_1$, we have
\begin{align*}
&\left|\int_{\R^d}f\,dg_*\mu-f\,d\nu\right|\\
&\le \left|\int_{\R^d}f\,dg_*\mu_{t_0}-f\,dg_*\mu\right| +
\left|\int_{\R^d}f\,d\nu_{t_0}-f\,d\nu\right| \\
&\hspace{12pt}+\left|\int_{\R^d\setminus K}f\circ T\,d\mu_{t_0}\right|
+\left|\int_{\R^d\setminus K}f\circ g\,d\mu_{t_0}\right|
+\int_{K}\left|f(T(x))-f(g(x))\right|\,d\mu_{t_0}(x)\\
&<\frac{\supRangenorm{\ReD}{f}\varepsilon}{5}+\frac{\supRangenorm{\ReD}{f}\varepsilon}{5}+\frac{\supRangenorm{\ReD}{f}\varepsilon}{5}+\frac{\supRangenorm{\ReD}{f}\varepsilon}{5}+\frac{L_f\varepsilon}{5}\\
&\le\varepsilon,
\end{align*}
where $L_f$ is the Lipschitz constant of $f$. 
Here we used $\supRangenorm{\Re^d}{f}+L_f\leq 1$. 
Therefore, we have $\beta(g_*\mu,\nu)<\varepsilon$.
\end{proof}

The following lemma is essentially due to \cite{HyvarinenNonlinear1999}.
\begin{lemma}
\label{existence of transformation for probability measures}
Let $\mu$ be a probability measure on $\R^d$ with a \(C^\infty\) density function $p$. Let $U:=\{x\in \R^d : p(x)>0\}$. Then there exists a diffeomorphism $T:U\rightarrow (0,1)^d$ such that its Jacobian is an upper triangular matrix with positive diagonals, and $T_*\mu={\rm U}(0,1)^d$. Here, ${\rm U}(0,1)^d$ is the uniform distribution on $[0,1]^d$.
\end{lemma}
\begin{proof}
Let $q_i(x_1,\dots,x_i):=\int_{\R^{d-i}}p(x_1,\dots,x_{i+1},\dots,x_d)\,dx_{i+1}\dots dx_d$. Then we define $T:U\rightarrow (0,1)^d$ by
\[T(x_1,\dots,x_d):=\left(\int_{-\infty}^{x_i}\frac{q_i(x_1,\dots,x_{i-1},y)}{q_{i-1}(x_1,\dots,x_{i-1})}dy\right)_i.\]
Then we see that $T$ is a diffeomorphism and its Jacobian is upper triangular with positive diagonal elements. Moreover, by direct computation, we have $T_*d\mu=U(0,1)$.
\end{proof}

\subsection{Proof of Proposition~\ref{proposition:body:generalized-distributional-universality}: From Sobolev Universality to Distributional Universality in the Total Variation Metric}\label{sec:proof-of-generalized-distributional-universality}

In this section, we prove Proposition~\ref{proposition:body:generalized-distributional-universality}.
Recall the definition of the total variation distance:
\[
\|\nu-\mu\|_{\rm TV} := \sup_{A} |\nu(A) - \mu(A)|,
\]
where the supremum is taken over all measurable sets of the underlying space.

Here, we restate the proposition.
\begin{theorem}[Proposition~\ref{proposition:body:generalized-distributional-universality} in the main text]
\label{thm: univ aprox for 0inf11 induces distribution TV approximator}
Let $r \ge 1$.
Let 
\[\mathcal{F}_0 := W^{0,\infty}_{\rm loc}(U,\mathbb{R}^d) \cap W^{1,1}_{\rm loc}(U,\mathbb{R}^d).\]
We define the topology of $\mathcal{F}_0$ as the weakest topology such that the inclusion maps $\dotlessi_0: \mathcal{F}_0 \xhookrightarrow{} W_{\rm loc}^{0,\infty}(U, \mathbb{R}^d)$ and $\dotlessi_1: \mathcal{F}_0 \xhookrightarrow{} W_{\rm loc}^{1,1}(U, \mathbb{R}^d)$ are both continuous.
Suppose any element in the model $\mathcal{M}$ is locally $C^{0,1}$ and a piecewise $C^1$-diffeomorphism.
If $\mathcal{M}$ is an $\mathcal{F}_0$-universal approximator for $\Triangular$,
then $\mathcal{M}$ is a $(\mathcal{P}^{\rm TV}, \mu)$-distributional universal approximator for $\mathcal{P}_{\rm ab}$ for any $\mu \in \mathcal{P}_{\rm ab}$.
\end{theorem}

\begin{proof}
Let $\mu$, $\nu\in \mathcal{P}_{\rm ab}$.
Take any $\varepsilon>0$. 
It is enough to show that  
there exists $f\in \mathcal{M}$ such that 
\[ 2 \|\nu-f_*\mu\|_{\rm TV}<\epsilon, \]
where $\|\cdot\|_{\rm TV}$ is the total variation norm.
By Lemmas \ref{existence of transformation for probability measures} and \ref{lem:molify_TV}, we can assume that there exist a positive smooth function $w$ satisfying $d\mu(x)=w(x)dx$ and $g \in \Triangular$ such that $\nu = g_*\mu$ and $g(\ReD) = \ReD$.
We fix a large compact set $K' \subset \ReD$ such that 
\[ \int_{\ReD \setminus K'} dg_*\mu < \frac{\varepsilon}{4}.\]
We fix an ``inverse'' $f^\dagger$  of the piecewise $C^1$-diffeomorphism $f$ as in \ref{existence of inverse} in Proposition \ref{prop: basic properties}.
We may assume $f^\dagger(K') \subset f^{-1}(K')$ if we take a suitable $f^\dagger$.
Note that $f^{-1}(K') \setminus f^\dagger(K')$ is a nullset.
Then, we can write $d(f_\ast\mu)(x)=w(f^{\dagger}(x))J_{f^{\dagger}}(x)dx$ and 
$d(g_\ast\mu)(x)=w(g^{-1}(x))J_{g^{-1}}(x)dx$. 
By Lemma~\ref{lem:existence_of_K} below, 
there exists a compact subset $K\subset \R^d$ 
such that $f^{-1}(K')\subset K$ for any $f\in \mathcal{M}$ 
satisfyting $\| f-g\|_{K,0,\infty}<\varepsilon$.

Since $g$ is a diffeomorphism, there exists $M_0>0$ such that $|J_g(g^{-1}(k') )|^{-1}<M_0$ for any $k'\in K'$.
Moreover, since the function $J_g(g^{-1}(\cdot ))$ is Lipschitz on $g(K)\cup K'$ , we can take $M_1>0$ satisfying 
$|J_g(g^{-1}(x))-J_g(g^{-1}(y))|<M_1|x-y|$ for any $x,y\in g(K)\cup K'$. 
Since the function $w$ is Lipschitz on $g^{-1}(K')\cup K$, 
we can take $L_0>0$ satisfying $|w(x)-w(y)|<L_0|x-y|$ for any $x,y \in g^{-1}(K')\cup K$. 
Since $g^{-1}$ is Lipschitz on $g(K)\cup K'$, we can take $L_1>0$ satisfying $|g^{-1}(x)-g^{-1}(y)|<L_1|x-y|$ for any $x, y\in g(K)\cup K'$. 

From the assumption, we can take $f\in \mathcal{M}$ satisfying 
\begin{align*}
\| f-g\|_{K,0,\infty}&<\frac{\varepsilon}{16M_0L_0\max\{M_1, L_1 \}\max\{\operatorname{vol}(K'), \operatorname{vol}(K), 1\}},\\
\| f-g\|_{K,1,1}&<\frac{\varepsilon}{16M_0 \max_{x\in K}|w(x)|}. 
\end{align*}
Then, since the total variation distance of probability measures is given by half the $L^1$-norm of the Radon-Nikodym derivative, we have
\begin{align*}
&2 \|g_\ast \mu-f_\ast \mu\|_{TV}\\
&\leq \int_{K'} |w(f^{\dagger}(x))J_{f^{\dagger}}(x)-w(g^{-1}(x))J_{g^{-1}}(x)|dx + \int_{\ReD \setminus K'} df_*\mu + \int_{\ReD \setminus K'} dg_*\mu\\
&\leq 
 2\int_{K'} |w(f^{\dagger}(x))J_{f^{\dagger}}(x)-w(g^{-1}(x))J_{g^{-1}}(x)|dx + 2\int_{\ReD \setminus K'} dg_*\mu\\
&\leq 2\int_{K'} | w(f^{\dagger}(x)) - w(g^{-1}(x))| |J_{g^{-1}}(x)|dx + 2\int_{K'}|J_{f^{\dagger}}(x)-J_{g^{-1}}(x)| |w(f^{\dagger}(x))| dx+ \frac{\varepsilon}{2}.
\end{align*}
As for the second equality, we use 
\begin{align*}
    \int_{\ReD \setminus K'} df_*\mu &=  1 - \int_{K'} df_*\mu \\
    &\leq \int_{K'} |w(f^{\dagger}(x))J_{f^{\dagger}}(x)-w(g^{-1}(x))J_{g^{-1}}(x)|dx  + 1- \int_{K'} dg_*\mu \\
    &= \int_{K'} |w(f^{\dagger}(x))J_{f^{\dagger}}(x)-w(g^{-1}(x))J_{g^{-1}}(x)|dx  +  \int_{\ReD \setminus K'} dg_*\mu
\end{align*}
The first term is estimated as follows:
\begin{align*}
&\int_{K'} | w(f^{\dagger}(x)) - w(g^{-1}(x))| |J_{g^{-1}}(x)|dx\\
&\leq L_0M_0 \int_{K'} | f^{\dagger}(x) - g^{-1}(x)| dx\\
&= L_0M_0 \int_{K'} |g^{-1}(g\circ f^{\dagger}(x))-g^{-1}(f\circ f^{\dagger}(x))|dx\\
&\leq L_0 M_0L_1\int_{K'}|g (f^{\dagger}(x))-f(f^{\dagger}(x))|dx\\
&\leq L_0 M_0L_1\operatorname{vol}(K') \sup_{k'\in f^{-1}(K')} |g(k')-f(k')|\\
&\leq L_0 M_0L_1\operatorname{vol}(K') \sup_{k\in K} |g(k)-f(k)|\\
&<\frac{\varepsilon}{8}. 
\end{align*}
Here, we used the fact $f^{\dagger}(K')\subset K$ in the second-to-last inequality and the bound for $\|f-g\|_{K, 0, \infty}$ in the last inequality.

Similarly, the second term is bounded as follows:
\begin{align*}
&\int_{K'}|J_{f^{\dagger}}(x)-J_{g^{-1}}(x)| |w(f^{\dagger}(x))| dx \\
&=\int_{f^{\dagger}(K')}|J_f (x)^{-1}-J_g(g^{-1}\circ f(x))^{-1}| |w(x)|J_f(x)dx\\
&\leq \int_{f^{\dagger}(K')}|1-J_f(x)J_g(g^{-1}\circ f (x))^{-1}||w(x)|dx\\
&=\int_{f^{\dagger}(K')}|J_g(g^{-1}\circ f(x))^{-1}||J_g(g^{-1}\circ f(x))-J_f(x)| |w(x)|dx\\
&\leq M_0\max_{x\in K}|w(x)|\int_{f^{\dagger}(K')}|J_g(g^{-1}\circ f(x))-J_f(x)|dx \\
&=M_0\max_{x\in K}|w(x)|\left[ \int_{f^{\dagger}(K')}|J_g(g^{-1}\circ f(x))-J_g(g^{-1}\circ g(x))| + |J_g(x)-J_f(x)|dx \right]\\
&\leq M_0\max_{x\in K}|w(x)|\left[M_1\int_{f^{\dagger}(K')}|f(x)-g(x)|dx+\int_{f^{\dagger}(K')}|J_g(x)-J_f(x)|dx\right]\\
&\leq M_0\max_{x\in K}|w(x)|\left[M_1\int_{x\in K}|f(x)-g(x)|dx+\int_{x \in K}|J_g(x)-J_f(x)|dx\right]\\
&< \frac{\varepsilon}{16}+\frac{\varepsilon}{16}=\frac{\varepsilon}{8}. 
\end{align*}
Again, we used $f^{\dagger}(K)\subset K$ in the second-to-last inequality.
In the last inequality, we used the bound for $\|f-g\|_{K, 0, \infty}$ for the first term and the bound for $\|f-g\|_{K, 1, 1}$ for the second term, respectively.
\end{proof}

\begin{lemma}\label{lem:molify_TV}
Let $\mu$ be an absolutely continuous probability measure on $\R^d$. 
For any $\varepsilon>0$, there exists an absolutely continuous probability measure $\nu$ such that 
$d\nu(x)= w(x) dx$ for some $w\in C^\infty(\R^d)$ with $w>0$ and $\|\mu-\nu\|_{\rm TV}<\varepsilon$.  
\end{lemma}
\begin{proof}
Let $p\in L^1(\R^d)$ be the density function of $\mu$. Let $\phi\in L^1(\R^d)$ be a positive $C^\infty$ function satisfying $\int_{\R^d} \phi(x) dx=1$.
For $t>0$, put $\phi_t(x):=t^{-d}\phi(x/t)$. 
Then we have 
\begin{align*}
   2 \| \mu-\nu\|_{TV}=\|p-\phi_t*p\|_{L^1(\R^d)}\to 0\quad (t\to +0). 
\end{align*}
\end{proof}

\begin{lemma}\label{lem:existence_of_K}
Let the model $\mathcal{M}$ be as in Theorem~\ref{thm: univ aprox for 0inf11 induces distribution TV approximator} and let $g$ be a homeomorphism from $\ReD$ to $\ReD$.
Let $K' \subset \mathbb{R}^d$ be a compact set and $\varepsilon > 0$.
Then, there exists a compact subset $K \subset \mathbb{R}^d$ such that $f^{-1}(K')  \subset K$ for any $f \in \mathcal{M}$ satisfying $\|f - g\|_{K,0,\infty} < \varepsilon$.
\end{lemma}
\begin{proof}
We may assume $K' = \overline{B(0, L)}$ for sufficiently large $L>0$ such that $L \ge \varepsilon$.
Since $g$ is a homeomorphism, there exists sufficiently large $R>0$ such that $\overline{B(0, R)} \supset g^{-1}(B(0, L + 2\varepsilon))$, that is, $g(\overline{B(0, R)}) \supset B(0, L+ 2\varepsilon) (\supset K')$.
We denote $K := \overline{B(0, R)}$.
Suppose $f\in \mathcal{M}$ satisfies $\|f - g\|_{K,0,\infty} < \varepsilon$.
Then, we have $f(\partial K) \cap K' = \emptyset$ for any $f$.
Thus, we see that $K' \subset f(B(0,R)) \cup (\R^d \setminus f(K))$.
Since $K'$ is connected, we see that either $K' \subset f(K)$ or $K' \subset \R^d \setminus f(K)$.
Suppose $K' \subset \R^d \setminus f(K)$.
On the other hand, since $0 \in K' \subset g(K)$, there exists $x \in K$ such that $g(x) = 0$.
Since $f(K) \cap K' = \emptyset$, we have 
\[L <|f(x) - 0| = |f(x)  - g(x)| < \varepsilon, \] 
which is a contradiction.
Therefore, we conclude $K' \subset f(K)$.
Since $f$ is a diffeomorphism, we have $f^{-1}(K') \subset K$.
\end{proof}
\subsection{Integral Probability Metrics}\label{appendix: IPM}

The results in Subsection~\ref{sec:proof-of-generalized-distributional-universality} imply the universality of INNs with respect to the total variation (TV) topology.
Here, we consider how the theoretical guarantees in the TV topology can be transported to other notions of closeness, namely those of integral probability metrics (IPMs).

We say a measurable set $A \subset \R^n$ is a {\em continuity set} of a measure \(\mu\) if the boundary $\partial A$ of $A$ is a null set, i.e., \(\mu(\partial A) = 0\).
We say a measurable set $A \subset \Re^n$ is a \emph{non-null set} of a measure \(\mu\) if \(\mu(A) \neq 0\).
For any measurable subset \(K \subset \Re^n\) and any probability measure \(\eta\) on \(\Re^n\), let us define the truncated measure \(\truncate{K}{\eta} := \eta(\cdot \cap K) / \eta(K)\) if \(\eta(K) > 0\) and \(\truncate{K}{\eta} := \zeroMeasure\) if \(\eta(K) = 0\), where \(\zeroMeasure\) is a constant zero measure.
To state the results, we define the following notion of universality.

\begin{definition}[Compact distributional universality]\label{def: compact distributional universality}
Let $\INNModelGeneric$ be a model which is a set of measurable maps from $\R^m$ to $\R^n$.
Let $\mathcal{P}_0$ be a set of probability measures on $\R^n$ with some topology. 
Let $\mathcal{Q}$ be a subset of \(\mathcal{P}_0\).
Fix a probability measure $\mu_0$ on $\Re^m$.
We say that a model $\INNModelGeneric$ is a \emph{$(\mathcal{P}_0, \mu_0)$-compact-distributional universal approximator} for $\mathcal{Q}$ (or \emph{has the $(\mathcal{P}_0, \mu_0)$-compact-distributional universal approximation property} for $\mathcal{Q}$) if for any \(\nu \in \mathcal{Q}\) and any non-null compact continuity set \(K \subset \Re^n\) of \(\nu\), $\{\truncate{K}{(g_*\mu_0)}: g \in \INNModelGeneric\} \setminus \{\zeroMeasure\}$ is a subset of $\mathcal{P}_0$ and if its closure (in $\mathcal{P}_0$) contains $\truncate{K}{\nu}$.
\end{definition}
Note that if \(\nu\) is compactly supported and \(K\) is such that \(\supp{\nu} \subset \interior{K}\), where \(\interior{K}\) denotes the interior of \(K\), then \(K\) is a continuity set of \(\nu\). Also, in this case, \(\truncate{K}{\nu} = \nu\).
Therefore, practically, given a compact distributional universality of a model \(\mathcal{M}\) and a compactly supported approximation target \(\nu \in \mathcal{Q}\), one can regard it as an approximation guarantee for \(\nu\) by taking a sufficiently large \(K\) so that it covers any practically relevant range of values as well as \(\supp \nu\).
\begin{remark}
Let $\mathcal{P}_0$ be a set of probability measures on $\R^n$ with some topology.
For $\mu \in \mathcal{P}_0$, a compact continuity set $K$ of $\mu$, and a neighborhood $V$ of $\mu|_K$ with $\mu|_K \neq \zeroMeasure$, we define
\[W_\mu(K,V) := \{ \nu \in \mathcal{P}_0 : \nu|_K \in V \}.  \]
We define a new topology of $\mathcal{P}_0$ via the neighborhoods of $\mu$'s by those generated by $W_\mu(K,V)$'s.
We denote by $\mathcal{P}_0^\tau$ the set $\mathcal{P}_0$ equipped with the topology above.
By definition, the truncation $\cdot|_K: \mathcal{P}_0^\tau \to \mathcal{P}_0 \cup \{\zeroMeasure\}$ for any compact continuity set of $\mu$ is continuous at any $\mu$ satisfying $\mu|_K \neq \zeroMeasure$, where the topology of $\mathcal{P}_0 \cup \{\zeroMeasure\}$ is the direct sum topology.
Conversely, $\mathcal{P}_0^\tau$ is characterized as the set $\mathcal{P}_0$ equipped with the weakest topology such that the above truncations are continuous.
If we impose that the topology of $\mathcal{P}_0$ is stronger than $\mathcal{P}_0^\tau$, namely the truncation $\cdot|_K$ is continuous at $\mu$ for any continuity set $K$ of $\mu$ with respect to the topology of $\mathcal{P}_0$.
Under the assumption, the compact distributional universality in Definition \ref{def: compact distributional universality} is rephrased as the $(\mathcal{P}_0^\tau, \mu_0)$-distributional universality for $\mathcal{Q}$.
Moreover, we may immediately prove that $(\mathcal{P}_0, \mu_0)$-distributional universality implies the compact distributional universality.
In the case of $\mathcal{P}_0 = \mathcal{P}^{\rm w}$, thanks to the portmanteau lemma, we may prove that the topology of $\mathcal{P}_0$ is stronger than $\mathcal{P}_0^\tau$, namely the truncation $\cdot|_K$ is continuous at $\mu$ for any continuity set $K$ of $\mu$. 
\end{remark}

IPMs are defined as follows.
\begin{definition}[Integral probability metric; \citet{MullerIntegral1997}]
Let \(\vIPMBaseSp\) be a measurable space, \(\mu\) and \(\nu\) be probability measures on \(\vIPMBaseSp\), and \(\vIPMClass\) be \(\Re\)-valued bounded measurable functions on \(\vIPMBaseSp\).
Then, the \emph{integral probability metric} (IPM) based on \(\vIPMClass\) is defined as
\[
\IPM{\vIPMClass}{\mu, \nu} := \sup_{\vIPMFn \in \vIPMClass} \left| \int_\vIPMBaseSp \vIPMFn d\mu - \int_\vIPMBaseSp \vIPMFn d\nu \right|
\]
\end{definition}
For a comprehensive review on IPMs, see, e.g., \citet{SriperumbudurIntegral2009}.

By selecting appropriate \(\vIPMClass\), various distance measures in probability theory and statistics can be obtained as special cases of the IPM.
In the following, assume that \(\vIPMBaseSp\) is equipped with a distance metric \(\vIPMBaseDist\) and that the \(\sigma\)-algebra is the Borel \(\sigma\)-algebra induced by the metric topology of \(\vIPMBaseDist\).
Let \(\Lipnorm{f} := \sup_{x, y \in \vIPMBaseSp, x \neq y} \frac{|f(x) - f(y)|}{\vIPMBaseDist(x, y)}\) and \(\bddLipnorm{f} := \inftynorm{f} + \Lipnorm{f}\).
Let \(\vRKHS\) be a reproducing kernel Hilbert space (RKHS) induced by a positive semidefinite kernel \(\vRKHSkern: \vIPMBaseSp \times \vIPMBaseSp \to \Re\), and let \(\RKHSnorm{\cdot}\) be its RKHS norm.
\begin{definition}[\citet{SriperumbudurIntegral2009}]\label{def: IPM examples}
We define the following metrics.
\begin{itemize}
    \item \emph{Dudley metric}: \(\cIPMDudley = \{f: \bddLipnorm{f} \leq 1\}\) yields the Dudley metric \(\IPM{\cIPMDudley}{\mu, \nu}\).
    \item \emph{Wasserstein distance}: if \(\vIPMBaseSp\) is separable, then \(\cIPMOneWasserstein = \{f: \Lipnorm{f} \leq 1\}\) yields the \(1\)-Wasserstein distance \(\IPM{\cIPMOneWasserstein}{\mu, \nu}\) for \(\mu, \nu \in \cOneWassersteinDists = \{\nu': \int \vIPMBaseDist(x, y) d\nu'(x)<\infty, \forall y \in \vIPMBaseSp\}\).
    \item \emph{Total variation distance}: \(\cIPMTV = \{f: \inftynorm{f} \leq 1\}\) yields the total variation distance \(\IPM{\cIPMTV}{\mu, \nu}\).
    \item \emph{Maximum mean discrepancy} (MMD): selecting \(\cIPMMMD = \{f \in \vRKHS: \RKHSnorm{f} \leq 1\}\) yields the MMD \(\IPM{\cIPMMMD}{\mu, \nu}\).
\end{itemize}
We use \(\mathcal{P}^{\cTextDudley}\), \(\mathcal{P}^{\cTextOneWasserstein}\), and \(\mathcal{P}^{\cTextMMD}\), to denote \(\mathcal{P}\) equipped with the induced topology of \(\IPM{\cIPMDudley}{\cdot,\cdot}\), \(\IPM{\cIPMOneWasserstein}{\cdot,\cdot}\), and \(\IPM{\cIPMMMD}{\cdot,\cdot}\), respectively.
\end{definition}
Note that, if \((\vIPMBaseSp, \vIPMBaseDist)\) is separable, e.g., \(\vIPMBaseSp=\ReD\), then the convergence in the Dudley metric is equivalent to the convergence in the weak topology \citep[Theorem~11.3.3.]{DudleyReal2002}.

\begin{remark}\label{rem: IPM implications}
If we interpret \(\vIPMClass\) in Definition~\ref{def: IPM examples} as a family of \emph{statistics}, i.e., functions that take random variables as the arguments,
we can interpret an approximation guarantee in terms of an IPM as an approximation guarantee for the expectation of the statistics computed from these distributions.
More concretely, once we obtain an approximation guarantee such as \(\IPM{\vIPMClass}{\mu, \nu} < \varepsilon\) where \(\nu\) is an approximation target, \(\mu\) is a model, and \(\varepsilon > 0\), then we can deduce that \(|\mathbb{E}_{X\sim\mu}[\vIPMFn(X)] - \mathbb{E}_{Y\sim\nu}[\vIPMFn(Y)]| < \varepsilon\), where \(\mathbb{E}\) denotes the expectation, holds uniformly over the class of statistics \(\vIPMFn\in\vIPMClass\).
If, moreover, we have a theoretical guarantee that \(|\int \vIPMFn d\mu - \sum_{i=1}^N \vIPMFn(X_i)| < \varepsilon'\) for \(\{X_i\}_{i=1}^N \overset{\text{i.i.d.}}{\sim}\mu\), where \(\text{i.i.d.}\) stands for \emph{independently and identically distributed}, with high probability for some \(\vIPMFn\in\vIPMClass\),
then we can combine these inequalities to provide an upper bound on \(|\sum_{i=1}^N \vIPMFn(X_i) - \mathbb{E}_{Y\sim\nu}[\vIPMFn(Y)]|\), i.e., the error of \emph{Monte Carlo} approximation based on the samples generated by the model \(\mu\) that approximated the target distribution \(\nu\).

Depending on the IPM, we have different families of statistics, \(\vIPMClass\), over which we can obtain such theoretical guarantees.
In the case of the Dudley metric corresponding to the weak convergence topology, we can obtain such an approximation guarantee over the class of (uniformly) bounded and Lipschitz-continuous (and hence measurable) functions \(\vIPMFn\) with a uniformly bounded Lipschitz constant.
In the case of the total variation, the guarantee is stronger, and we can obtain the guarantee over the class of (uniformly) bounded measurable functions \(\vIPMFn\).
\end{remark}

We have the following elementary relations that can be easily shown from the definitions.
\begin{proposition}\label{prop: TV-Dudley-MMD}
We have the following inequalities:
\begin{align*}
\IPM{\cIPMDudley}{\mu, \nu} &\leq \IPM{\cIPMTV}{\mu, \nu}, \\
\IPM{\cIPMMMD}{\mu, \nu} &\leq
\left(\sup_{x \in \vIPMBaseSp} \vRKHSkern(x, x)\right)^{\frac{1}{2}} \IPM{\cIPMTV}{\mu, \nu}.
\end{align*}
\end{proposition}
\begin{proof}
The first inequality follows from \(\cIPMDudley \subset \cIPMTV\), which holds by definition.
The second inequality follows from the Cauchy-Schwarz inequality:
\[
\inftynorm{f} = \sup_{x \in \vIPMBaseSp} |f(x)|
= \sup_{x \in \vIPMBaseSp} |\langle f, \vRKHSkern(x, \cdot) \rangle_{\vRKHS}|
\leq \RKHSnorm{f} \left(\sup_{x \in \vIPMBaseSp} \vRKHSkern(x, x)\right)^{\frac{1}{2}},
\]
where \(\langle \cdot, \cdot \rangle_{\vRKHS}\) denotes the inner product of \(\vRKHS\).
\end{proof}

We also have the following relation between the total variation distance and the \(1\)-Wasserstein distance for \(\vIPMBaseSp = \ReD\).
\begin{lemma}\label{lem: TV-W1}
Let \(\mu, \nu \in \mathcal{P}\), and let \(K\) be a compact non-null set of \(\nu\).
If \(\IPM{\cIPMTV}{\mu, \nu} < \nu(K)\), then
\begin{align}\label{eq: lem: TV-W1}
\IPM{\cIPMOneWasserstein}{\truncate{K}{\mu},\truncate{K}{\nu}}
\leq \frac{4 \cdot \diam{K}}{\nu(K)} \cdot \frac{\IPM{\cIPMTV}{\mu, \nu}}{\nu(K) - \IPM{\cIPMTV}{\mu, \nu}},
\end{align}
where \(\diam{K}\) denotes the diameter of \(K\).
\end{lemma}
We defer the proof of Lemma~\ref{lem: TV-W1} to the bottom part of this subsection, and we first display the following proposition to collect Corollary~\ref{prop: TV-Dudley-MMD} and Lemma~\ref{lem: TV-W1}.
\begin{proposition}\label{prop: TV-univ to Dudley, MMD, and W1}
Let \(\mathcal{Q} \subset \mathcal{P}\) and $\mu \in \mathcal{P}$.
Assume that $\mathcal{M}$ is a $(\mathcal{P}^{\rm TV}, \mu)$-distributional universal approximator for $\mathcal{Q}$.
Then, we have the following.
\begin{itemize}
    \item[(a)] $\mathcal{M}$ is a $(\mathcal{P}^{\cTextDudley}, \mu)$-distributional universal approximator for $\mathcal{Q}$,
    \item[(b)] If \(\sup_{x \in \ReD} \vRKHSkern(x, x) < \infty\), then $\mathcal{M}$ is a $(\mathcal{P}^{\cTextMMD}, \mu)$-distributional universal approximator for $\mathcal{Q}$,
    \item[(c)] $\mathcal{M}$ is a $(\mathcal{P}^{\cTextOneWasserstein}, \mu)$-compact-distributional universal approximator for $\mathcal{Q}$.
\end{itemize}
\end{proposition}
The condition part of Proposition~\ref{prop: TV-univ to Dudley, MMD, and W1} is covered by the conclusion part of Theorem~\ref{thm: univ aprox for 0inf11 induces distribution TV approximator}, where \(\mathcal{Q}\) and \(\mu\) are arbitrary \(\mathcal{Q} \subset \mathcal{P}_{\rm ab}\) and \(\mu \in \mathcal{P}_{\rm ab}\).
Therefore, we can immediately obtain the theoretical guarantee of distribution approximation using INNs with respect to these IPMs given a Sobolev universality of \(\mathcal{M}\).
\begin{proof}[Proof of Proposition~\ref{prop: TV-univ to Dudley, MMD, and W1}]
The first two immediately follow from Corollary~\ref{prop: TV-Dudley-MMD}.
The final assertion follows from Lemma~\ref{lem: TV-W1}.
To show the final assertion, one needs to show that, for any \(\nu \in \mathcal{Q}\), any non-null compact continuity set \(K \subset \ReD\) of \(\nu\), and any \(\varepsilon > 0\), there exists \(g \in \mathcal{M}\) such that \(\IPM{\cIPMOneWasserstein}{\truncate{K}{(\pushforward{g}{\mu})}, \truncate{K}{\nu}}\).
By the assumption that $\mathcal{M}$ is a $(\mathcal{P}^{\rm TV}, \mu)$-distributional universal approximator for $\mathcal{Q}$,
there exists \(g \in \mathcal{M}\) such that both \(\IPM{\cIPMTV}{\pushforward{g}{\mu}, \nu} < \nu(K)\)
and the right-hand side of Equation~\eqref{eq: lem: TV-W1} in Lemma~\ref{lem: TV-W1} is smaller than \(\varepsilon\),
so that \(\IPM{\cIPMOneWasserstein}{\truncate{K}{(\pushforward{g}{\mu})}, \truncate{K}{\nu}} < \varepsilon\).
\end{proof}
To prove Lemma~\ref{lem: TV-W1}, we use the following well-known inequality between the Wasserstein distance and the total variation distance.
\begin{fact}[{\citealp{VillaniOptimal2009}, Theorem~6.15}]\label{fact: Wasserstein bound bdd by TV}
Let \((\vIPMBaseSp, \vIPMBaseDist)\) be a separable complete metric space that is bounded with diameter \(R\), and \(\mu\) and \(\nu\) be probability measures on \(\vIPMBaseSp\).
Then, we have \(\IPM{\cIPMOneWasserstein}{\mu, \nu} \leq R \cdot \IPM{\cIPMTV}{\mu, \nu}\).
\end{fact}

Lemma~\ref{lem: TV-W1} is an immediate corollary of this fact.
Note that 
\[\IPM{\cIPMTV}{\mu, \nu} = 2 \sup_{A} |\mu(A) - \nu(A)|\] holds, where \(\sup_A\) denotes the supremum over all measurable subsets of the underlying space.
\begin{proof}[Proof of Lemma~\ref{lem: TV-W1}]
Since \((K, \|\cdot\|)\) is a separable complete metric space, we have, by applying Fact~\ref{fact: Wasserstein bound bdd by TV} with \(\truncate{K}{\mu}\) and \(\truncate{K}{\nu}\),
\begin{align*}
\IPM{\cIPMOneWasserstein}{\truncate{K}{\mu},\truncate{K}{\nu}}
&= \sup_{f \in \cIPMOneWasserstein} \left|\int_{\ReD} f d(\truncate{K}{\mu}) - \int_{\ReD} f d(\truncate{K}{\nu})\right| \\
&= \sup_{f \in \Restrict{\cIPMOneWasserstein}{K}} \left|\int_K f d(\truncate{K}{\mu}) - \int_K f d(\truncate{K}{\nu})\right| \\
&\leq \diam{K} \cdot 2 \cdot \sup_{A'} |(\truncate{K}{\mu})(A') - (\truncate{K}{\nu})(A')| =: \text{(RHS)},
\end{align*}
where \(\sup_{A'}\) denotes the supremum over all measurable subsets of \(K\), and \(\Restrict{\cIPMOneWasserstein}{K} := \{\Restrict{f}{K}: f \in \cIPMOneWasserstein\}\).
Now, since we have \(\nu(K) - \mu(K) \leq |\mu(K) - \nu(K)| \leq \IPM{\cIPMTV}{\mu, \nu}\), we obtain \(\mu(K) \geq \nu(K) - \IPM{\cIPMTV}{\mu, \nu} > 0\).
Thus, \(\truncate{K}{\mu}(\cdot) = \mu(\cdot \cap K)/\mu(K)\), and hence the right-hand side (RHS) is further bounded as
\begin{align*}
\text{(RHS)}
&= 2 \cdot \diam{K}\sup_{A} |\mu(A \cap K)/\mu(K) - \nu(A \cap K)/\nu(K)| \\
&\leq 2 \cdot \diam{K}\sup_{A} |\mu(A)/\mu(K) - \nu(A)/\nu(K)|,
\end{align*}
where \(\sup_{A}\) denotes the supremum over all measurable subsets of \(\ReD\), and the inequality holds since \(\sup_{A}\) runs through all the measurable subsets of the form \(A \cap K\) as well.
Now,
\begin{align*}
\left|\frac{\mu(A)}{\mu(K)} - \frac{\nu(A)}{\nu(K)}\right|
&\leq \left|\frac{\mu(A)}{\mu(K)} - \frac{\nu(A)}{\mu(K)}\right| + \left|\frac{\nu(A)}{\mu(K)} - \frac{\nu(A)}{\nu(K)}\right| \\
&= \frac{|\mu(A) - \nu(A)|}{\mu(K)} + \left|\nu(K) - \mu(K)\right| \frac{\nu(A)}{\mu(K)\nu(K)} \\
&\leq \frac{\nu(K) + \nu(A)}{\mu(K)\nu(K)}\IPM{\cIPMTV}{\mu, \nu}.
\end{align*}
Therefore, we have
\[
\IPM{\cIPMOneWasserstein}{\truncate{K}{\pushforward{g}{\mu}},\nu}
\leq \frac{4 \cdot \diam{K}}{\nu(K)} \frac{\IPM{\cIPMTV}{\pushforward{g}{\mu}, \nu}}{\nu(K) - \IPM{\cIPMTV}{\pushforward{g}{\mu}, \nu}},
\]
where we used \(\mu(K) \geq \nu(K) - \IPM{\cIPMTV}{\mu, \nu} > 0\) and \(\nu(K) + \nu(A) \leq 2\).
\end{proof}
\section{Proof of Theorem~\ref{thm:body:diffeo-universal-equivalences}: Equivalence of universal properties}
\label{sec:appendix:universality-proof}
In this section, we provide the proof details of Theorem~\ref{theorem:main:1} in the main text.
First, we give the overall proof of Theorem~\ref{theorem:main:1} in Section~\ref{sec:universality-proof-overall}.
In later sections, we give missing proofs for lemmas used in Section~\ref{sec:universality-proof-overall}.
Specifically, Section~\ref{sec:appendix:reduction to cpt supp} explains the reduction from \(\mathcal{D}^{\max\{1,r\}}\) to \(\DcRDCmd{\infty}\), Section~\ref{sec:appendix:Dc2 to flowends} explains the reduction from \(\DcRDCmd{\infty}\) to \(\FlowEnds{\infty}\), and
Section~\ref{sec:appendix:Dc2 to Sinfty} explains the reduction from \(\FlowEnds{\infty}\) to \(\CinftyOneDimTriangular\) and permutations of variables.

\subsection{Proof of Theorem~\ref{theorem:main:1}}\label{sec:universality-proof-overall}
\begin{proof}[Proof of Theorem~\ref{theorem:main:1} and \ref{thm: sup universality}]
First, we prove the equivalence of statements \ref{main thm: Dtwo} and \ref{main thm: Xitwo}.
In light of Lemmas~\ref{red Cr to Cinf}, \ref{red to comp. supp. diff}, and \ref{lem: diffc2 is generated by flow endpoints}, for any $f\in\DrV{\max\{1,r\}}{U}$ and a compact subset $K\subset U_f$,
there exist $W \in \FLin$ and $g_1,\dots,g_m\in \FlowEnds{\max\{1,r\}}$ such that $f(x)=W\circ g_1 \circ\cdots\circ g_m(x)$ for all $x\in K$.
Since $W$ and $g_i$'s satisfy the condition to apply Corollary \ref{cor: compatibility of approximation, higher derivatives}, are linearly increasing (see Remark~\ref{rem:linealy_increasing}), 
we obtain the equivalence of statements \ref{main thm: Dtwo} and \ref{main thm: Xitwo}. 

Next, we prove the equivalence of statements \ref{main thm: Dtwo}, \ref{main thm: Tinfty}, and \ref{main thm: Scinfty}. 
Since we have $\CinftyOneDimTriangular\subset \Triangular\subset \DrV{\max\{1,r\}}{\R^d}$, 
it is sufficient to prove that the $L^p$-universal approximation property for $\CinftyOneDimTriangular$ 
implies that for $\DrV{\max\{1,r\}}{U}$ for any open subset $U\subset \R^d$ which is $C^{\max\{1,r\}}$ diffeomorphic to $\R^d$. 
The strategy is similar to the flow endpoint case in the previous paragraph.
Using Theorem~\ref{thm: singles and permutations} on top of Lemma~\ref{red to comp. supp. diff} and Lemma~\ref{lem: diffc2 is generated by flow endpoints}, for any $f\in\DrV{\max\{1,r\}}{U}$ and a compact subset $K\subset U_f$, there exist $W_1,\dots, W_k \in \FLin$ and $\tau_1,\dots,\tau_k\in \CinftyOneDimTriangular$ such that $f(x)=W_1\circ\tau_1\circ\cdots\circ W_k\circ\tau_k(x)$ for all $x\in K$.
Again, we use Corollary~\ref{cor: compatibility of approximation, higher derivatives} to prove the claim.
\end{proof}

\subsection{Step 1: From \texorpdfstring{$\mathcal{D}^r$}{Dr} to \texorpdfstring{$\DcRDCmd{\infty}$}{Xi infty}}
\label{sec:appendix:reduction to cpt supp}
In this section, we describe how the approximation of $\mathcal{D}^r$ is reduced to that of $\DcRDCmd{\infty}$ when we are only concerned with its approximation on a compact set.
We first remark that we may assume any target map is $C^\infty$ mapping:
\begin{lemma} \label{red Cr to Cinf}
For any open subset $U \subset \ReD$, $\mathcal{D}_U^\infty$ is a $W^{r,\infty}$-universal approximator for $\mathcal{D}_U^r$.
\end{lemma}
\begin{proof}
It follows from Theorem 2.7, p.50 in \cite{HirschDifferential1976}.
\end{proof}
Thanks to this lemma, we can prove Theorem~\ref{thm:sobolev-universality-equivalence} without requiring the condition $r \neq d + 1$ that was required in the statement of Fact~\ref{thm: simplicity}.

The following lemma shows that we may assume the target map is compactly-supported.
\begin{lemma}
\label{lem1}
\label{red to comp. supp. diff}
Assume $r \ge 2$.
Let $U$ be an open set of $\R^d$, $K\subset U$ a compact set, and $f\in \DrV{r}{U}$.
Then, there exist $h \in \DcRDr$ and an affine transform $W \in \FLin$ such that \[\Restrict{W\circ h}{K}=\Restrict{f}{K}.\]
\end{lemma}
\begin{proof}
We denote the injections of $U$ and $f(U)$ into $\mathbb{R}^d$ by $\iota_1\colon  U\hookrightarrow \mathbb{R}^d$ and $\iota_2\colon  f(U) \hookrightarrow \mathbb{R}^d$, respectively.
Since $U$ is $C^r$-diffeomorphic to $\mathbb{R}^d$ and $f$ is $C^r$-diffeomorphic, $f(U)$ is also $C^r$-diffeomorphic to $\mathbb{R}^d$.
By applying Corollary~\ref{cor: extension} below 
to $\iota_1 \circ \Restrict{f^{-1}}{f(U)}\colon  f(U) \to \mathbb{R}^d$ and the injection $\iota_2$, we can obtain $C^r$-diffeomorphisms $F_1\colon f(U)\rightarrow \mathbb{R}^d$ and $F_2\colon  f(U) \rightarrow \mathbb{R}^d$ such that $\Restrict{F_1}{f(K)} = \Restrict{f^{-1}}{f(K)}$ and $\Restrict{F_2}{f(K)} = \Identity_{f(K)}$, where $\Identity_{f(K)}$ denotes the identity map on ${f(K)}$.
Let $F:=F_2\circ F_1^{-1}\colon  \mathbb{R}^d \to \mathbb{R}^d$.
By definition, we have $\Restrict{F}{K} = \Restrict{f}{K}$.

Take a sufficiently large open ball $B$ centered at 0 such that $K\subset \frac{1}{2}B$. 
Let $W\in\FLin{}$ such that $W^{-1}(x)=\Jac{F}(0)^{-1}(x-F(0))$. Then by Lemma \ref{extension lemma for function on ball} below, 
we conclude that there exists a compactly supported diffeomorphism $h\colon \mathbb{R}^d\rightarrow\mathbb{R}^d$ such that $\Restrict{W\circ h }{K}=\Restrict{F}{K} = \Restrict{f}{K}$.
\end{proof}

Here, we remark that Lemma~\ref{extension lemma for function on ball} below is a modified version of Lemma~D.1 in \citet{BernardExpressing2018}, with a correction to make it explicit that the extended diffeomorphism is compactly supported. Their Lemma~D.1 does not explicitly state that it is compactly supported, but by Theorem~1.4 in Section~8 of \citet{HirschDifferential1976}, it can be shown that the diffeomorphism is compactly supported. We provide the proof as follows:
\begin{lemma}
\label{extension lemma for function on ball}
Let $r\geq 2$ be an integer, $R$ a positive scalar, and $B_R\subset\ReD$ an open ball of radius $R$ with origin $0$, and let $f:B_R\rightarrow f(B_R)\subset \ReD$ be a $C^r$-diffeomorphism onto its image such that $f(0)=0$ and $\Jac{f}(0)=I$.
Let $\varepsilon\in (0,R/2)$. 
Then there exists $h\in \DcRDr$ such that $f(x)=h(x)$ for any $x\in B_{R-\varepsilon}$.
\end{lemma}
\begin{proof}
Put $\delta:=\varepsilon/(2R-\varepsilon)$, and define $I_\delta:=(-\delta,1+\delta)$.  We define $F:B_{R-\frac{\varepsilon}{2}}\times I_\delta \rightarrow\ReD$ by
\[F(x,t):=
\begin{cases}
\frac{f(tx)}{t} & \text{ if }t\neq 0,\\
x & \text{ if } t=0.
\end{cases}
\]
Here $F$ is $C^r$, $C^1$ with respect to  $x$, $t$, respectively. 
Let 
\[U:=\left\{ (F(x,t), t) : (x,t) \in B_{R-\frac{\varepsilon}{2}} \times I_\delta\right\}\subset \ReD\times \Re\]
and let $F^\dagger:U \to B_{R-\frac{\varepsilon}{2}}$ such that $F(F^\dagger(x,t),t)=x$ for any $(x,t)\in U $.
Here, $F^\dagger$ is the first component of the inverse of the map $(x,t) \mapsto (F(x,t), t)$ from $B_{R-\frac{\varepsilon}{2}} \times I_\delta$ onto $U$.
We note that $U$ is a bounded open subset in $\ReD \times \Re$.
Fix a compactly supported $C^\infty$-function $\phi$ on $\ReD\times I_\delta$ such that for $(x,t)\in F\big(\overline{B_{R-\varepsilon}}\times [0,1] \big)\times [0,1]$, $\phi(x,t)=1$, and for $(x,t)\notin U$, $\phi(x,t)=0$.
Then we define $H:\ReD\times I_\delta\rightarrow \ReD$ by
\[H(x,t):=
\begin{cases}
\phi(x,t)\frac{\partial F}{\partial t}(F^\dagger(x,t),t) & (x,t) \in U, \\
0 & \text{otherwise}.
\end{cases}
\]
Since $F^\dagger$ is $C^1$ and for fixed $t\in I_\delta$, $\frac{\partial F}{\partial t}(\cdot ,t)$ is $C^r$, 
there exists $L>0$ such that for any $t\in I_\delta$, $\Vert H(x,t)-H(y,t)\Vert < L\| x- y\|$ with $x,y\in \ReD$. Thus the differential equation
\[\frac{dz}{dt}=H(z,t),~~z(0)=x\]
has a unique solution $\phi_x(t)$. Then $h(x):=\phi_x(1)$ is the desired extension.
\end{proof}

As a corollary, we can prove a $C^r$-version of Theorem 3.3 in \citet{Bernarddiffeomorphism2015}:
\begin{corollary}\label{cor: extension}
Let $r\geq 2$ be a positive integer and $f \in \DrV{r}{U}$.
Assume $U$ is $C^r$-diffeomorphic to $\ReD$.
Then, for any compact $K \subset U$, there exists a $C^r$-diffeomorphism $F$ from $U$ to $\ReD$ with $f(U) = \ReD$ such that
\[F|_K = f|_K.\]
\end{corollary}
\begin{proof}
Fix a $C^r$-diffeomorphism $g: U \to \ReD$.
Let $\varepsilon > 0$ and take a sufficiently large $R$ such that $g^{-1}(B_{R-\varepsilon})$ contains $K$, where $B_R$ is the open ball of radius $R$ with origin $0$.
By using Lemma \ref{extension lemma for function on ball}, there exists $h \in \DcRDr$ and $W \in {\rm Aff}$ such that $h(x) = W\circ f\circ g^{-1}(x)$ for all $x \in B_{R-\varepsilon}$.
As $h$ is surjective mapping, $F:= W^{-1}\circ h \circ g$ is the desired $C^r$-diffeomorphism from $U$ onto $\ReD$.
\end{proof}

\subsection{Step 2: From \texorpdfstring{\(\DcRDCmd{\infty}\)}{Diffcinfty} to \texorpdfstring{\(\FlowEnds{\infty}\)}{Xi infty}}
\label{sec:appendix:Dc2 to flowends}
This section explains the reduction of the universality for \(\DcRDCmd{\infty}\) to $\FlowEnds{\infty}$.
We here prove a slightly general result.
The reduction involves a structure theorem from the field of differential geometry.
The results of this section are used as a building block for the proofs in Section~\ref{sec:appendix:Dc2 to Sinfty}.

Let $r$ be a positive integer or $\infty$.
The set $\DcRDr$ constitutes a group whose group operation is the function composition.
Moreover, $\DcRDr$ is a topological group with respect to the \emph{Whitney topology} \citep[Proposition~1.7.(9)]{HallerGroups1995}.
Then there is a crucial structure theorem of $\DcRDr$ attributed to Herman, Thurston \cite{ThurstonFoliations1974}, Epstein \cite{Epsteinsimplicity1970}, and Mather \cite{MatherCommutators1974, MatherCommutators1975}:
\begin{fact}
\label{thm: simplicity}
Assume $1 \leq r \leq \infty$ and $r\neq d+1$.
Then, the group $\DcRDr$ is simple, i.e., any normal subgroup $H \subset \DcRDr$ is either $\{\Identity\}$ or $\DcRDr$.
\end{fact}
The assertion is proven in \citet{MatherCommutators1975} for the connected component containing \(\Identity\), instead of the entire set of compactly-supported \(C^r\)-diffeomorphisms when the domain space is a general manifold instead of \(\ReD\). In the special case of \(\ReD\), the connected component containing \(\Identity\) is known to be \(\DcRDr\) itself \citep[Example~1.15]{HallerGroups1995}, hence Fact~\ref{thm: simplicity} follows.
For details, see \citep[Corollary~3.5 and Example~1.15]{HallerGroups1995}.
Also, \citet{BanyagaStructure1997} is an introductory monograph that explains the simplicity of $\DcRDCmd{\infty}$.

We use Fact~\ref{thm: simplicity} to prove that a compactly supported diffeomorphism can be represented as a composition of flow endpoints in $\DcRDr$.
\begin{lemma}
\label{lem: diffc2 is generated by flow endpoints}
If $r\not =d+1$, the set of compactly supported diffeomorphisms $\DcRDr$ coincides with the set of finite compositions of the elements of $\FlowEnds{r}$.
More specifically, we have 
\[\DcRDr=\{ g_1\circ \cdots \circ g_n : n\ge1, g_1,\dots,g_n \in \FlowEnds{r} \}.\]

\begin{proof}
Put $H^r:=\{ g_1\circ \cdots \circ g_n : n\ge1, g_1,\dots,g_n \in \FlowEnds{r} \}$. 
First, we prove that $H^r$ forms a subgroup of $\DcRDr$.
By definition, for any $g, h \in H^r$, it holds that $g \circ h \in H^r$.
Also, $H^r$ is closed under inversion; to see this, it suffices to show that $\FlowEnds{r}$ is closed under inversion.
Let $g= \Phi(\cdot, 1) \in \FlowEnds{r}$. Consider the map $\phi:\mathbb{R}^d\times U\rightarrow\mathbb{R}^d$ defined by $\phi(x, t) := \Phi(\cdot, t)^{-1}(x)$.
It is easy to confirm that $\phi$ satisfies the conditions of Definition~\ref{def: flow endpoints}, hence $g^{-1} = \phi(\cdot, 1)$ is an element of $\FlowEnds{r}$. Note that $\phi$ is confirmed to be $C^r$ on $\ReD \times U$ by applying the inverse function theorem (e.g., \citep[Theorem~1 of Chapter~I, Section~5]{LangDifferential1985}) to $(t, \x) \mapsto (t, \Phi(\x, t))$.

Next, we prove that $H^r$ is normal.
To show that the subgroup generated by $\FlowEnds{r}$ is normal, it suffices to show that $\FlowEnds{r}$ is closed under conjugation.
Take any $g\in \FlowEnds{r}$ and $h\in \DcRDr$, and let $\Phi$ be a flow associated with $g$.
Then, the function $\Phi': \mathbb{R}^d\times U \to \ReD$ defined by $\Phi'(\cdot, s) := h^{-1} \circ \Phi(\cdot, s) \circ h$ is a flow associated with $h^{-1}\circ g \circ h$ satisfying the conditions in Definition~\ref{def: flow endpoints}, which implies $h^{-1}\circ g \circ h\in \FlowEnds{r}$, i.e., $\FlowEnds{r}$ is closed under conjugation.

Next, we prove that $H^r$ is non-trivial by constructing an element of $\FlowEnds{r}$ that is not the identity element.
First, consider the case $d = 1$.
Let $\tilde v: \Re \to \Re_{\geq 0}$ be a non-constant $C^\infty$-function such that $\supp{\tilde v} \subset [0, 1]$ and $\tilde v^{(k)}(0) = 0$ for any \(k \in \Na\).
Then define \(v : \Re \to \Re\) by
\[v(x) = \begin{cases}\tilde v(|x|)\frac{x}{|x|} & \text{ if } x \neq 0, \\ 0 & \text{ if } x = 0,\end{cases}\]
which is a \(C^\infty\)-function on \(\Re\) with a compact support.
Since $v$ is Lipschitz continuous and $C^\infty$, there exists $\IVPFunc{v}$ that is a $C^\infty$-function over $\Re \times \Re$; see Fact~\ref{fact:ODE solution exists for Lip} and \citep[Chapter~V, Corollary~4.1]{HartmanOrdinary2002}.
Let $K_v \subset \Re$ be a compact subset that contains $\supp{v}$. Then, by considering the ordinary differential equation by which $\IVPFunc{v}$ is defined, we see that $\bigcup_{t \in \Re} \supp\IVP{v}{\cdot}{t} \subset K_v$ and also that $\IVP{v}{x}{0} = x$.
We also have $\IVP{v}{x}{s+t} = \IVP{v}{\IVP{v}{x}{s}}{t}$ for any $s, t \in \Re$. In particular, we have $\IVP{v}{\cdot}{s}^{-1} = \IVP{v}{\cdot}{-s}$ for any $s \in \Re$. Therefore, we have $\IVP{v}{\cdot}{1} \in \FlowEnds{r}$. Since $v \not \equiv 0$, $\IVP{v}{\cdot}{1}$ is not an identity map and thus $\FlowEnds{r}$ is not trivial.
Next, we consider the case $d \geq 2$.
Take a $C^\infty$-function $\phi\colon \R\to \R$ with $\supp{\phi}= [1,2]$ and a nonzero skew-symmetric matrix $A$ (i.e. $A^\top=-A$) of size $d$, and
let $X(x):=\phi(\|x\|)A$.
We define a $C^\infty$-map $\Phi\colon \R^d\times \R\to \R^d$ by 
\[ \Phi(x,t):= \exp(t X(x))x. \]
Since $\exp( tX(x))$ is an orthogonal matrix for any $t\in \R$ and $x\in \R^d$, $\Phi$ is a $C^\infty$-flow on $\R^d$. 
Now, it is enough to show that there exists a compact set $K_\Phi\subset \R^d$ satisfying $\cup_{t\in \R}\supp{\Phi(\cdot, t)}\subset K_\Phi$. 
Let $K_\Phi:=\{x\in \R^d\ |\ \|x\|\leq 2\}$. 
Then the inclusion $\supp{\Phi(\cdot, t)}\subset K_{\Phi}$ holds for any $t\in\R$ since $X(x)=0$ for $x\in \R^d\setminus K_\Phi$.
\end{proof}

\end{lemma}
\subsection{Step 3: From \texorpdfstring{$\FlowEnds{\infty}$}{Xiinfty} to \texorpdfstring{\(\CinftyOneDimTriangular\)}{Scinfty} and permutations}
\label{sec:appendix:Dc2 to Sinfty}
The goal of this section is to show Theorem~\ref{thm: singles and permutations}, which reduces the approximation problem of $\FlowEnds{\infty}$ to that of $\CinftyOneDimTriangular$.
We here show a slightly general result.

\begin{theorem}
\label{thm: singles and permutations}
Let $1 \le r \le \infty$.
Let $f\in \FlowEnds{r}$. Then there exist $\tau_1,\dots,\tau_n \in \CrOneDimTriangular$, and permutations of variables $\sigma_1,\dots,\sigma_n\in \mathfrak{S}_d$, such that
\[f=\tau_1\circ \sigma_1\circ\dots\circ \tau_n\circ\sigma_n.\]
\end{theorem}
\begin{proof}
Combining
Corollary~\ref{corollary: from dcrd to nearId},
Lemma~\ref{lem: operator norm to principal minors},
and Lemma~\ref{lemma:red_to_one_dim},
we have the assertion.
\end{proof}
By combining Theorem~\ref{thm: singles and permutations} with Lemma~\ref{lem: diffc2 is generated by flow endpoints}, we conclude that the same claim holds for any element $f$ in $\DcRD$.

\begin{definition}[near-$\Identity$ elements]
Let $f: \R^d\to \R^d$ be a differentiable map.
We say $f$ is \emph{near-$\Identity$} if, for any $x\in \R^d$, the Jacobian $\Jac{f}$ of $f$ at $x$ satisfies
\begin{align*}
   \opnorm{\Jac{f}(x)-I}<1,
\end{align*}
where $I$ is the unit matrix.
\end{definition}

\begin{corollary}
\label{cor1}
\label{corollary: from dcrd to nearId}
For any $f\in \FlowEnds{r}$, there exist finite elements $g_1,\dots, g_k\in \DcRDr$ such that $f=g_k\circ \dots\circ g_1$ and 
$g_i$ is near-$\Identity$ for any $i \in [k]$. 
\end{corollary}
\begin{proof}
Let $\Phi$ be a flow associated with $f$.
Since $\Phi(\cdot, 0)$ is the identity function and $\Phi$ is continuous on $\R^d \times U$, we can take a sufficiently large $n$ such that $\tilde{h}:=\Phi(\cdot, 1/n)$ is near-Id. 
By the additive property of $\Phi$, we have 
\begin{align*}
\label{expre}
    f
    =\underbrace{\tilde{h}\circ\dots\circ\tilde{h}}_{n\ \text{times}},
\end{align*}
which completes the proof of the corollary.
\end{proof}

In the remainder of this section, we describe Lemma~\ref{lem: operator norm to principal minors}, Lemma~\ref{lemma:red_to_one_dim}, and Lemma~\ref{smoothing lemma}. 
First, Lemma~\ref{lem: operator norm to principal minors} claims that the near-$\Identity$ elements necessarily satisfy the condition of Lemma~\ref{lemma:red_to_one_dim} below.

\begin{lemma}
\label{lem: operator norm to principal minors}
Let $A=(a_{i,j})_{i,j=1,\dots,d}$ be a matrix. If $\Vert A- I_d \Vert_{\rm op}<1$, then for $k=1,\dots,d$, the $k$-th trailing principal submatrix $A_k:=(a_{i+k-1,j+k-1})_{i,j=1,\dots,d-(k-1)}$ of $A$ is invertible. Here $I_d$ is a unit matrix of degree $d$.
\end{lemma}
\begin{proof}
Let $v\in\Re^{d-k+1}$ with $\Vert v\Vert=1$, and put $w:=(0,\dots,0,v)\in\ReD$. Then we have $1>\Vert (A-I_d)w\Vert^2\ge\Vert (A_k-I_k)v\Vert^2$. Thus $\Vert A_k-I_k\Vert<1$.  Since $\sum_{r=0}^\infty(I_k-A_k)^r$ absolutely converges, and it is identical to the inverse of $A_k$, we have that $A_k$ is invertible.
\end{proof}
We apply the following lemma together with Lemma~\ref{lem: operator norm to principal minors} to decompose near-$\Identity$ elements into $\CtwoOneDimTriangular$ and permutations.
For $a\in \Na$, we denote the set of $a$-by-$a$ real-valued matrices by $M(a, \mathbb{R})$.
\begin{lemma}\label{lemma:red_to_one_dim}
\label{lem6}\status{Ikeda checked, tojo modified}
Let $1 \le r \le \infty$ and
$f\colon \mathbb{R}^d\rightarrow\mathbb{R}^d$ a compactly supported $C^r$-diffeomorphism. We write $f=(f_1,\dots,f_d)$ with $f_i\colon \mathbb{R}^d\rightarrow \mathbb{R}$.  
For $k\in [d]$, let $\Delta^f_k(\bm{x})\in M(d-(k-1), \R)$ be the $k$-th trailing principal submatrix of Jacobian matrix of $f$, whose  $(i,j)$ component is given by $\left(\frac{\partial f_{i+k-1}}{\partial x_{j+k-1}}(\bm{x})\right)$ $(i, j=1,\cdots, d-(k-1))$. 
We assume 
\[\det \Delta^f_k(x)\neq 0 \text{ for any }k\in [d]\text{ and }x\in \R^d. \]
Then there exist compactly supported $C^r$-diffeomorphisms $F_1,\dots,F_d:\mathbb{R}^d\rightarrow \mathbb{R}^d$ in the forms of 
\[F_i(\bm{x}):=(x_1,\dots,x_{i-1},h_i(\bm{x}),x_{i+1},\dots,x_d)\]
for some $h_i\colon \mathbb{R}^d\rightarrow\mathbb{R}$ such that the identity holds:
\[f=F_1\circ\dots\circ F_d. \]
\end{lemma}
\begin{proof}
The proof is based on induction. Suppose that $f$ is in the form of 
\[f(\bm{x})=(f_1(\bm{x}),\dots,f_m(\bm{x}),x_{m+1},\dots,x_d).\]
By means of induction with respect to $m$, we prove that there exist compactly supported $C^r$-diffeomorphisms $F_1,\dots,F_m:\mathbb{R}^d\rightarrow \mathbb{R}^d$ in the forms of 
\(F_i(\bm{x}):=(x_1,\dots,x_{i-1},h_i(\bm{x}),x_{i+1},\dots,x_d)\)
for some $h_i:\mathbb{R}^d\rightarrow\mathbb{R}$ such that
\(f=F_1\circ\dots\circ F_m\).

In the case of $m=1$, the above is clear.  Assume that the statement is true in the case of any $k<m$.
Define
\begin{align*}
    F(x_1,\dots,x_d)&:=(x_1,\dots,x_{m-1},f_m(\bm{x}),x_{m+1},\dots,x_d), \\
    \tilde{f}&:=f\circ F^{-1}. 
\end{align*}
Note that 
$F$ is a compactly supported $C^r$-diffeomorphism from $\R^d$ to $\R^d$. 
In fact, compactly supportedness and surjectivity of $F$ comes from the compactly supportedness of $f$.
Moreover, since we have $\det \Jac F_x=\frac{\partial f_m}{\partial x_m}(x)\neq 0$ for any $x\in \R^d$ by the assumption on $f$, $F$ is injective and is a $C^r$-diffeomorphism from $\R^d$ to $\R^d$ by inverse function theorem. 
Therefore, $\tilde{f}$ is also a $C^r$-diffeomorphism from $\R^d$ to $\R^d$.
We show 
that $\tilde{f}$ is of the form $\tilde{f}(\bm{x})=(g_1(\bm{x}), \cdots, g_{m-1}(\bm{x}), x_m,\cdots, x_d)$ for some $C^r$-functions $g_i\colon \R^d\to \R$ $(i=1,\cdots, m-1)$ satisfying $\det \Delta^{\tilde{f}}_k(x)\neq 0$ for any $x\in \R^d$ and $k\in [d]$. 
From Lemma~\ref{lemma:inverse_component}, 
there exist $g_i, h\in C^r(\R^d)$ $(i=1,\cdots, m)$ such that 
\begin{align*}
    f^{-1}(\bm{x})&=(g_1(\bm{x}), \cdots, g_m(\bm{x}), x_{m+1}, \cdots, x_d)\\
    F^{-1}(\bm{x})&=(x_1,\cdots, x_{m-1}, h(\bm{x}), x_{m+1}, \cdots, x_d). 
\end{align*}
Then we have 
\begin{align*}
\tilde{f}^{-1}(\bm{x})=F\circ f^{-1}(\bm{x})
&= (g_1(\bm{x}),\cdots, g_{m-1}(\bm{x}), f_m(f^{-1}(\bm{x})), x_{m+1},\cdots, x_{d})\\
&=(g_1(\bm{x}), \cdots,g_{m-1}(\bm{x}), x_m, \cdots, x_d). 
\end{align*}
Therefore, from Lemma~\ref{lemma:inverse_component}, 
$\tilde{f}$ is of the following form 
\[ \tilde{f}(x)= f\circ F^{-1}(x)= (f_1\circ F^{-1}(x), \cdots, f_{m-1}\circ F^{-1}(x), x_m,\cdots, x_d).  \]
Moreover, by the form of $F^{-1}$ and $f$, 
we have $\Jac{\tilde{f}}(x)=\Jac{f}(F^{-1}(x))\circ \Jac{F^{-1}}(x)$ and 
\[  \Jac f=
\begin{pmatrix}
A & \\
  & I
\end{pmatrix}, \quad
\Jac{(F^{-1})}=
\begin{pmatrix}
I_{m-1} & &   \\
  \frac{\partial h}{\partial x_1} & \cdots & \frac{\partial h}{\partial x_d}\\
 & & I_{d-m} 
\end{pmatrix}
\]
for some $A\in M(m,\R)$ with all the trailing principal minors nonzero. 
Therefore, we obtain $\det \Delta^f_k(x)\neq 0$ for any $x\in \R^d$ and $k\in [d]$. 
Here, by the assumption of the induction, 
there exist compactly supported $C^r$-diffeomorphisms $F_i\colon \R^d\to \R^d$ and $h_i\in C^r(\R^d)$ $(i=1,\cdots, m-1)$ such that
\[\tilde{f}=F_1\circ \cdots \circ F_{m-1},\ F_i(\bm{x})=(x_1,\cdots x_{i-1}, h_i(x), x_{i+1}, \cdots, x_d). \]
Thus $f= \tilde{f}\circ F$ has the desired form.
\end{proof}

\begin{lemma}\label{lemma:inverse_component}
\label{lem5}\status{Ikeda checked}
Let $1\le r \le \infty$ and 
$f\colon \R^d\to \R^d$ $C^r$-diffeomorphism of the form 
\[ f(\bm{x}):=(f_1(\bm{x}), \cdots, f_m(\bm{x}), x_{m+1}, \cdots, x_d), \]
where $f_i\colon \R^d\rightarrow\R$ belongs to $C^r (\R^d)$ $(i=1,\cdots, m)$. 
Then the inverse map $f^{-1}$ becomes of the form
\[ f^{-1}(\bm{x})= (g_1(\bm{x}), \cdots, g_m(\bm{x}), x_{m+1},\cdots x_d), \]
where $g_i:\R^d\rightarrow\R$ belongs to $C^r(\R^d)$ for $i=1,\cdots, m$. 
\end{lemma}
\begin{proof}
We write $f^{-1}(\bm{x})=(h_1(\bm{x}), \cdots, h_d(\bm{x}))$, where $h_i\in C^r(\R^d)$ $(i=1,\cdots, d)$. 
Then by the definition of the inverse map, the identity 
\[ (x_1,\cdots, x_d)=f\circ f^{-1}(\bm{x}) =(f_1(h_1(\bm{x})),\cdots, f_m(h_m(\bm{x})), h_{m+1}(\bm{x}),\cdots, h_d(\bm{x})) \]
holds for any $\bm{x}\in \R^d$, which implies that we obtain $h_i(x)=x_i$ $(i=m+1,\cdots, d)$. This completes the proof of the lemma. 
\end{proof}

\subsection{\texorpdfstring{$L^p$}{Lp} universality for continuous mappings}
Here, we prove the following lemma, which is essentially proved in \cite{LiDeep2020}.
In this section, we always assume $p \in [1, \infty)$.
For any finite subset $S \subset \ReD$, we denote by ${\rm Map}(S,\ReD)$ the set of maps from $S$ to $\ReD$ and equip it with the supremum topology.
Then, for any finite subset $S \subset \ReD$, a set of bijections $\mathcal{M}$, and a subset $\mathcal{F} \subset {\rm Map}(S, \ReD)$, $\mathcal{M}$ is an $L^\infty$-universal approximator for $\mathcal{F}$ if $\mathcal{M}$ is a ${\rm Map}(S, \ReD)$-universal approximator for $\mathcal{F}$.
\begin{lemma} \label{lem: simple Lp universality}
Let $\mathcal{M}$ be a set of bijections from $\ReD$ to $\ReD$.
We assume that $\mathcal{M}$ satisfies the following three conditions:
\begin{enumerate}[(1)]
\item all function of $\mathcal{M}$ is locally Lipschitz.
\item \label{universality for finite set} for any finite subset $S \subset \ReD$, $\mathcal{M}$ is the $L^\infty$-universal approximator for the set of all the injections from $S$ to $\ReD$.
\item \label{special universality} $\mathcal{M}$ is the $L^p$-universal approximator for 
the subset 
\[\left\{ f: [0,1]^d \to \ReD : f(x_1,\dots, x_d) = (f_i(x_i))_{i=1}^d\text{ and } f_i\text{ is nondecreasing} \right\}. \]
\end{enumerate}
Then, $\mathcal{M}\circ \mathcal{M} := \{g\circ f : g,f \in \mathcal{M}\}$ is a $L^\infty$-universal approximator for $C^0([0,1]^d, \ReD)$, where $C^0(U, V)$ is the set of continuous maps from $U$ to $V$.
\end{lemma}
\begin{proof}
Let $\varepsilon>0$ be a positive number.
Let $f \in C^0([0,1]^d, \ReD)$, $m$ be a positive integer,  and $K \subset [0,1]^d$.
For any $\alpha \in \mathbb{Z}_{\ge 0}^d$ with $|\alpha| \ge m$, 
let , where
\begin{align*}
    \Delta_\alpha &:= \prod_{i=1}^d \left[\frac{\alpha_i-1}{m}, \frac{\alpha_i}{m}\right) \subset \ReD \\
    p_\alpha &:= \left(\frac{\alpha_1-1}{m}, \dots, \frac{\alpha_m-1}{m} \right)
\end{align*}
Put $y_\alpha := f(p_\alpha)$.
We define 
\[ H_m(x_1, \dots, x_m) := \left( \sum_{k=0}^m \frac{k}{m}\mathbf{1}_{[k/m, k+1/m)}(x_i) \right).\]
By \eqref{universality for finite set}, there exists $\psi_m \in \mathcal{M}$ such that
\[ \|\psi_m(p_\alpha) - y_\alpha \| < 1/m\]
for any $\alpha$ with $|\alpha| \le m$.
Since $f$ is continuous, we see that
\[\sup_{|\alpha| \le m}\sup_{x \in \Delta_\alpha}\|\psi_m(p_\alpha) - f(x)\| < \varepsilon/2 \]
if we take $m$ sufficiently large.
let $L_m$ be the Lipschitz constant for $\psi_m|_K$.
by \eqref{special universality}, there exists $g_m \in \mathcal{M}$ such that
\[ \LpKnorm{g_m - H_m} < \frac{\varepsilon}{2L_m}.\]
therefore, we have
\begin{align*}
\LpKnorm{\psi_m \circ g_m - f} &\le \LpKnorm{\psi_m \circ g_m - \psi_m \circ H_m} + \LpKnorm{\psi_m \circ H_m - f} \\
&\le L_m \LpKnorm{g_m -H_m} + \sup_{|\alpha| \le m}\sup_{x \in \Delta_\alpha}\|\psi_m(p_\alpha) - f(x)\|\\
&< \varepsilon.
\end{align*}
\end{proof}

Then, we have the following corollary:
\begin{corollary} 
\label{cor: Lp universality of D1}
let $U \subset \ReD$ be an open subset.
Then, $\mathcal{D}_{\ReD}^\infty$ is an $L^p$-universal approximator for $C^0(U, \ReD)$.
\end{corollary}
\begin{proof}
it suffices to show that for any $f \in C^0(U, \ReD)$, $\varepsilon >0$, and compact subset $K \subset U$, there exists $g \in \mathcal{D}_{\ReD}^\infty$ such that
\[\supKnorm{g-f} < \varepsilon.\]
we may assume $U=\ReD$ and $K=[0,1]^d$.
then, we easily see that $\mathcal{D}_{\ReD}^\infty$ satisfies the three conditions in Lemma \ref{lem: simple Lp universality} (see Lemma \ref{smoothing lemma} for the third condition).
thus, it follows from Lemma \ref{lem: simple Lp universality}.
\end{proof}

We also obtain a stronger version of \eqref{main thm: A} in  Theorem \ref{theorem:main:1}:
\begin{theorem}
\label{thm:strong approximation property}
We use the same notation as in Theorem \ref{theorem:main:1}.
Assume the condition of \eqref{main thm: A} in Theorem \ref{theorem:main:1}.
Then, if $\INN{\ARFINNFlow} $ is an $L^p$ universal approximator for $\OneDimTriangular$, then it is an $L^p$-universal approximator for $C^0(U, \ReD)$ for any open subset $U \subset \ReD$.
\end{theorem}

\section{Universality of coupling-flow based INNs}
\label{appendix:sec:examples}
In this section, we give the proofs for the universal approximation properties of certain \ARFINNs{}.

\subsection{Using permutation matrices instead of \texorpdfstring{\(\FLin\)}{Aff} in the definition of \texorpdfstring{\(\INN{\ARFINNFlow}\)}{INNG}}
\label{sec:appendix:elementary matrix}
In terms of representation power, there is no essential difference if we substitute the general linear group in Definition~\ref{def: INNM} with the permutation group.
It comes from the fact that one can express the elementary operation matrices using affine coupling flows and permutations.
More formally, we have the following proposition.

\begin{proposition}
\label{prop:appendix: sign flip and permutations}
Assume that $\mathcal{H}$ includes all the  functions $\R^{d-1}\to \R$ of the following forms: $x\mapsto -x\cdot e_i$, $x\mapsto x\cdot e_i$, and $x \mapsto b$ (constant map),
where $b\in \R^{d-1}$ and $i=1,\cdots, d-1$. 
Then, we have 
\begin{align}
    \INN{\FSACFH}=\{ W_1\circ g_1 \circ \cdots \circ W_n \circ g_n ~:~ g_i\in\FSACFH, W_i\in \mathfrak{S}_d\},\label{INNHACF=HACF + permutation}
\end{align}
where $\mathfrak{S}_d$ is the permutation group of degree $d$.
\end{proposition}
\begin{proof}
Since the multiplication of any permutation matrix is an affine transformation, the right-hand side of \eqref{INNHACF=HACF + permutation} is included in the left-hand side.

We prove the converse inclusion.
Since any translation operator (i.e., the addition of a constant vector) can be easily represented by the elements of $\FSACFH$ and permutations, it is enough to show that any element of ${\rm GL}(d,\R)$ can be realized by a finite composition of elements of $\FSACFH$ and $\mathfrak{S}_d$.
To show that, it is sufficient to consider only the elementary matrices. 
Row switching comes from $\mathfrak{S}_d$.
Moreover, element-wise sign flipping can be described by a composition of finite elements of $\FSACFH$. To see this, first observe that
\begin{align}
\left(
\begin{array}{cc}
-1&0\\
0&1
\end{array}
\right)
=
\left(
\begin{array}{cc}
1&0\\
1&1
\end{array}
\right)
\left(
\begin{array}{cc}
0&1\\
1&0
\end{array}
\right)
\left(
\begin{array}{cc}
1&0\\
-1&1
\end{array}
\right)
\left(
\begin{array}{cc}
0&1\\
1&0
\end{array}
\right)
\left(
\begin{array}{cc}
1&0\\
1&1
\end{array}
\right)
\left(
\begin{array}{cc}
0&1\\
1&0
\end{array}
\right)
\nonumber
\end{align}
holds.
Here, the linear transforms
\[
\left(
\begin{array}{cc}
1&0\\
-1&1
\end{array}
\right),
\left(
\begin{array}{cc}
1&0\\
1&1
\end{array}\right)\]
are realized by the $\FSACFH$ layers
\[
(x, y) \mapsto (x, y - x),
\quad
(x, y) \mapsto (x, y + x),
\]
respectively.
Now, any lower triangular matrix with positive diagonals can be described by a composition of finite elements of $\FSACFH$.
Therefore, any diagonal matrix whose components are $\pm1$ can be described by a composition of elements in $\FSACFH$ and $\mathfrak{S}_d$.
Therefore, any affine transform is an element of the right-hand side of \eqref{INNHACF=HACF + permutation}.
\end{proof}

This result implies that employing \(\FLin\) in Definition~\ref{def: INNM} instead of the permutation matrices is not an essential requirement for the universal approximation properties to hold.
For this reason, we believe that the empirically reported difference in the performances of Glow~\cite{KingmaGlow2018} and RealNVP~\cite{DinhDensity2016a} is mainly in the efficiency of approximation rather than the capability of approximation.

\subsection{Affine coupling flows (ACFs)}
\label{appendix: theorem 2 proof}
In this section, we provide the proof details of Theorem~\ref{theorem:main:2} in the main text.

\subsubsection{Proof of Theorem~\ref{prop:body:acfinn-Lp}: \texorpdfstring{\(L^p\)}{Lp}-universality of \texorpdfstring{\ACFHINN{}}{INNHACF}}
\label{sec:appendix:lp ACFINN approx general}
In this section, we prove the following lemma to construct an approximator for an arbitrary element of $\CzeroOneDimTriangular$ (hence for $\CinftyOneDimTriangular$) within \ACFHINN{}.
It is based on Lemma~\ref{lem: univ. approx.} proved in Section~\ref{sec:appendix:coordinate-wise independent}, which corresponds to a special case.

Here, we rephrase Theorem~\ref{theorem:main:2} as the following:
\begin{lemma}[$L^p$-universality of \ACFHINN{} for compactly supported $\CinftyOneDimTriangular$]
\label{lem:appendix:lp-univ for ACF}
Let \(p \in [1, \infty)\).
Assume \(\ACFINNUniversalClass\) is an \(L^\infty\)-universal approximator for \(\CcinftyRDminus\) and that it consists of piecewise \(C^1\)-functions.
Let $f \in \CzeroOneDimTriangular$, $\varepsilon>0$, and $K\subset \R^d$ be a compact subset.
Then, there exists $g \in \FACFHINN$ such that $\LpKnorm{f - g} < \varepsilon$.
\end{lemma}

\begin{proof}
Since we can take $a>0$, $b\in \R$ satisfying $aK+b \subset [0,1]^d$,
it is enough to prove the assertion for the case $K=[0,1]^d$.

Next, we show that we can assume that
for any $(\bm{x},y)\in \R^d$, $u(\bm{x},0)=0$ and $u(\bm{x},1)=1$ for any $\bm{x}\in \R^{d-1}$. 
Since $u(\bm{x}, \cdot)$ is a homeomorphism, we have $u(\bm{x}, 0) \not = u(\bm{x}, 1)$ for any $x\in \mathbb{R}$.
By the continuity of $f$, either of $u(\bm{x}, 0) > u(\bm{x}, 1)$ for all $\bm{x}\in [0, 1]^{d-1}$ or $u(\bm{x}, 0) < u(\bm{x}, 1)$ for all $x\in [0, 1]^{d-1}$ holds.
Without loss of generality, we assume the latter case holds (if the former one holds, we just switch $u(\bm{x},0)$ and $u(\bm{x},1)$).
 We define $s(\bm{x}) = -\log(u(\bm{x}, 1) - u(\bm{x}, 0))$ and $t(\bm{x}) = -u(\bm{x}, 0)(u(\bm{x}, 1)-u(\bm{x}, 0))^{-1}$.
By a direct computation, we have
\[
    \Psi_{d-1, s, t} \circ f(\bm{x}, y) = \left(\bm{x}, \frac{u(\bm{x}, y) - u(\bm{x}, 0)}{u(\bm{x}, 1) - u(\bm{x}, 0)}\right) =: (\bm{x}, u_0(\bm{x}, y)).
\]
In particular, $\Psi_{s,t} \circ f(\bm{x}, 0) = (\bm{x}, 0)$ and $\Psi_{s,t}\circ s (\bm{x}, 1) = (\bm{x}, 1)$ hold.
, and the map $y \mapsto u_0(\bm{x}, y)$ is a diffeomorphism for each $\bm{x}$. 
Thus if we prove the existence of an approximator for $\Psi_{s,t} \circ f$, by Proposition \ref{prop: compatibility of approximation}, we can arbitrarily approximate $f$ itself.

For $\underline{k}:=(k_1,\dots,k_{d-1})\in\mathbb{Z}^{d-1}$ and $n\in \mathbb{N}$, we define $(\underline{k})_n := \sum_{i=1}^d k_i n^{i-1}\in \{0, \ldots, n^d-1\}$, that is, $\underline{k}$ is the $n$-adic expansion of $(\underline{k})_n$.
For any $n\in \mathbb{N}$, define the following discontinuous \ACF{}:
$\psi_n \colon [0,1]^d\to [0,1]^{d-1}\times [0,n^d]$ by 
\[\psi_n(\bm{x},y):=\left(\bm{x}, y+\sum_{k_1,\cdots, k_{d-1}=0}^{n-1} (\underline{k})_n 1_{\Delta^n_{\underline{k} + 1}}(\bm{x}) \right),\]
where $\underline{k} := (k_1, \ldots, k_d)$ and $\underline{k} + 1 := (k_1+1, \ldots, k_d+1)$.
We take an increasing function $v_n\colon \R\to \R$ that is 
smooth outside finite points such that 
\[
v_n(z):=
\begin{cases}
u\left(\frac{k_1}{n},\cdots, \frac{k_{d-1}}{n}, z-(\underline{k})_n\right)+(\underline{k})_n & \text{ if }z\in [(\underline{k})_n, (\underline{k})_n + 1) \\
z &\text{ if }z\notin[0,n^d).
\end{cases}
\]
We consider maps $h_n$ on $[0,1]^{d-1}\times [0,n^d]$ 
and $f_n: [0,1]^d\to [0,1]^d$ defined by 
\begin{align*} 
h_n(\bm{x},z)&:=(\bm{x},v_n(z)),\\
f_n&:=\psi_n^{-1}\circ h_n\circ \psi_n. 
\end{align*}
Then we have the following claim.\\
\textbf{Claim.} For all $k_1, \cdots, k_{d-1}=0,\cdots, n-1$, we have
\begin{equation*}
    f_n(\bm{x},y)=\left(\bm{x},u\left(\frac{k_1}{n},\ldots, \frac{k_{d-1}}{n}, y\right)\right)
\end{equation*}
on $\prod_{i=1}^{d-1}[\frac{k_i}{n},\frac{k_i+1}{n})\times [0,1)$.

In fact, we have
\begin{align*}
f_n(\bm{x},y)&=\psi_n^{-1}\circ h_n\circ \psi_n(\bm{x},y)\\
&=\psi_n^{-1}\circ h_n(\bm{x},y+(\underline{k})_n)\\
&=\psi_n^{-1}(\bm{x},v_n(y+(\underline{k})_n))\\
&=\psi_n^{-1}\left(\bm{x},u\left(\frac{k_1}{n},\ldots, \frac{k_{d-1}}{n},y\right)+(\underline{k})_n\right)\\
&=\left(\bm{x},u\left(\frac{k_1}{n},\ldots, \frac{k_{d-1}}{n}, y\right)\right). 
\end{align*}
Therefore, the claim above has been proved. 
Hence we see that $\supKnorm{f-f_n}\rightarrow 0$ as $n\rightarrow\infty$.
By Lemma~\ref{lem: univ. approx.} below and the universal approximation property of \(\ACFINNUniversalClass\), for any compact subset $K$ and $\varepsilon > 0$, there exist $g_1, g_2, g_3\in \FACFHINN$ such that $\LpKnorm{g_1-\psi_n^{-1}}<\varepsilon$, $\LpKnorm{g_2-h_n}<\varepsilon$, and $\LpKnorm{g_3-\psi_n}<\varepsilon$. Thus by Proposition \ref{prop: compatibility of approximation}, for any compact $K$ and $\varepsilon>0$, there exists $g\in \FACFHINN$ such that $\LpKnorm{g-f}<\varepsilon$.
\end{proof}

\subsubsection{Special case: Approximation of coordinate-wise independent transformation}
\label{sec:appendix:coordinate-wise independent}
In this section, we show the lemma claiming that special cases of single-coordinate transformations, namely coordinate-wise independent transformations, can be approximated by the elements of \ACFHINN{} given sufficient representational power of $\ACFINNUniversalClass$.

\begin{lemma}
\label{lem: univ. approx.}
Let \(p \in [1, \infty)\).
Assume \(\ACFINNUniversalClass\) is an \(L^\infty\)-universal approximator for \(\CcinftyRDminus\) and that it consists of piecewise \(C^1\)-functions.
Let $u:\mathbb{R}\rightarrow\mathbb{R}$ be a continuous increasing function.
Let $f:\mathbb{R}^d\rightarrow\mathbb{R}^d;(\bm{x}, y)\mapsto(\bm{x}, u(y))$ where $\bm{x}\in \R^{d-1}$ and $y\in\R$.
For any compact subset $K\subset\mathbb{R}^d$ and $\varepsilon>0$, there exists $g\in \FACFHINN$ such that $\LpKnorm{f - g} < \varepsilon$.
\end{lemma}
\begin{proof}
We may assume without loss of generality, in light of Lemma~\ref{smoothing lemma}, that $u$ is a $C^\infty$-diffeomorphism on $\R$ and that the inequality $u'(y)>0$ holds for any $y\in \mathbb{R}$.
Furthermore, we may assume that $u$ is compactly supported (i.e., $u(y) = y$ outside a compact subset of $\Re$) without loss of generality because we can take a compactly supported diffeomorphism $\tilde u$ and $a, b \in \Re$ ($a \neq 0$) such that $a \tilde{u} + b = u$ on any compact set containing $K$ by Lemma~\ref{red to comp. supp. diff}, and the scaling $a$ and the offset $b$ can be realized by the elements of \ACFHINN{}.

Fix $\delta \in (0,1)$.  We define the following functions:
\begin{align*}
\psi_0(\x,y):&=(\upto{d-2}{\x}, u'(y) x_{d-1}, y)\\
&= (\upto{d-2}{\x}, \exp(\log u'(y)) x_{d-1}, y),\\
\psi_1(\x,y):&=\left(\upto{d-2}{\x}, x_{d-1} + \delta^{-1}(u(y) - y), y\right),\\
\psi_2(\x,y):&=(\upto{d-2}{\x}, x_{d-1}, y+\delta x_{d-1}),\\
\psi_3(\x,y):&=\left(\upto{d-2}{\x}, x_{d-1}-\delta^{-1}(y-u^{-1}(y)), y\right),
\end{align*}
where we denote $\bm{x}=(x_1,\dots,x_{d-1})\in\R^{d-1}$.
First, we show that $\supKnorm{f - \psi_3\circ\psi_2\circ\psi_1\circ\psi_0} \to 0$ as $\delta \to 0$.
By a direct computation, we have
\begin{align*}
    \psi_3\circ\psi_2\circ\psi_1(\x,y)
    &= \psi_3\circ\psi_2(\upto{d-2}{\x}, x_{d-1} +\delta^{-1}(u(y) - y), y)\\
    &= \psi_3(\upto{d-2}{\x}, x_{d-1}+\delta^{-1}(u(y) - y), y+\delta(x_{d-1}+\delta^{-1}(u(y)-y)))\\
    &= \psi_3(\upto{d-2}{\x}, x_{d-1}+\delta^{-1}(u(y)-y),\delta x_{d-1} + u(y))\\
    &= (\upto{d-2}{\x}, x_{d-1}-\delta^{-1}(\delta x_{d-1}+u(y)-u^{-1}(\delta x_{d-1} + u(y))), \delta x_{d-1} + u(y))\\
    &= (\upto{d-2}{\x}, \delta^{-1}u^{-1}(\delta x_{d-2} + u(y)) - \delta^{-1}y, u(y) + \delta x_{d-1}),
\end{align*}
where $\x=(x_1,\dots,x_{d-1})\in\R^{d-1}$.
Since $u\in C^\infty([-r, r])$ where $r = \max_{(\x, y) \in K} |y|$, by applying Taylor's theorem, there exists a function $R(\x, y; \delta)$ and $C = C([-r, r], u) > 0$ such that
\[u^{-1}(u(y)+\delta x)=y+u'(y)^{-1}\delta x + R(\x,y; \delta)(\delta x)^2 \quad\text{ and }\quad \sup_{\delta\in (0,1)}|R(\x, y; \delta)| \le C\]
for all $(\x, y)\in K$.
Therefore, we have 
\[\psi_3\circ\psi_2\circ\psi_1\circ\psi_0(\x, y)=(\bm{x},u(y))+\delta(R(\x,u'(y)x_{d-1};\delta)\upto{d-1}{\x},u'(y)x_{d-1}).\]
For any compact subset $K$, the last term uniformly converges to 0 as $\delta\rightarrow 0$ on $K$.

Assume $\delta$ is taken to be small enough.
Now, we approximate $\psi_3 \circ \cdots \circ \psi_0$ by the elements of \ACFHINN{}.
Since $u$ is a compactly-supported $C^\infty$-diffeomorphism on $\Re$, the functions $(\upto{d-2}{\x}, y) \mapsto \log u'(y)$, $(\upto{d-2}{\x}, y) \mapsto u(y) - y$, and $(\upto{d-2}{\x}, y) \mapsto y - u^{-1}(y)$, each appearing in $\psi_0$, $\psi_1$, $\psi_3$, respectively, belong to \(\CcinftyRDminus\).
On the other hand, $\psi_2$ can be realized by $\FGL \subset \FLin$.
Therefore, combining the above with the fact that $\ACFINNUniversalClass$ is a \(L^\infty\)-universal approximator for \(\CcinftyRDminus\), we have that for any compact subset $K'\subset\ReD$ and any $\varepsilon > 0$, there exist $\phi_0$, $\ldots, \phi_3\in\FACFHINN$ such that $\supRangenorm{K'}{\psi_i - \phi_i} < \varepsilon$.
In particular, we can find $\phi_0, \ldots, \phi_3\in\FACFHINN$ such that $\LpRangenorm{K'}{\psi_i - \phi_i} < \varepsilon$.

Now, recall that $\ACFINNUniversalClass$ consists of piecewise $C^1$-functions as well as $\psi_i$ ($i =0, \ldots, 3$).
Moreover, $\psi_0, \psi_1, \psi_3$ are compactly supported while $\psi_2 \in \FGL$, hence they are Lipschitz continuous outside a bounded open subset.
Therefore, by Proposition~\ref{prop: compatibility of approximation}, we have the assertion of the lemma.

\end{proof}

The following Lemma~\ref{smoothing lemma} is used above when reducing the approximation problem from \(\CtwoOneDimTriangular\) to \(\CinftyOneDimTriangular\).
\begin{definition}
We say that a map \(f: \ReD \to \Re\) is \emph{last-increasing} (resp. \emph{last-non-decreasing}) if, for any $(a_1,\dots,a_{d-1})\in\R^{d-1}$, the function \(f(a_1, \ldots, a_{d-1}, x)\) is strictly increasing (resp. non-decreasing) with respect to $x$.
\end{definition}
\begin{lemma}\status{Tojo added a proof}
\label{smoothing lemma}
Let $r \ge 0$ be an integer, and let $p \in [1,\infty]$.
Let $\tau\colon \R^d\rightarrow \R$ be a last-non-decreasing measurable function.
We assume that $\tau$ is locally $C^{r-1,1}$-function if $r\ge1$ or locally $L^\infty$ if $r=0$.
Then for any compact subset $K\subset \R^d$ and any $\varepsilon>0$, there exists a last-increasing \(C^\infty\)-function $\tilde{\tau}\colon\mathbb{R}^d\rightarrow \R$ satisfying
 \[ \|\tau-\tilde{\tau}\|_{K,r,p}<\varepsilon. \]
\end{lemma}
\begin{proof}
Let $\phi:\R^d\rightarrow\R$ be a compactly supported non-negative \(C^\infty\)-function with $\int |\phi(x)|dx =1$ such that for any $(a_1,\dots,a_{d-1})\in\R^{d-1}$, the function $\phi(a_1,\dots,a_{d-1}, x)$ of $x$ is even and decreasing on $\{x>0 : \phi(a_1,\dots,a_{d-1},x)>0\}$.
For $t>0$, we define $\phi_t(x):=t^{-d}\phi(x/t)$. Then we see that $\tau_t:=\phi_t*\tau$ is a \(C^\infty\)-function. 
We take any $\bm{a}\in \R^{d-1}$. 
We verify that $\tau_t(\bm{a}, x_d)$ is strictly increasing with respect to $x_d$.
Take any $x_d, x_d'\in \R$ satisfying $x_d>x_d'$. 
Since $\tau$ is strictly increasing, we have 
\begin{align*}
\tau_t(\bm{a}, x_d)-\tau_t(\bm{a}, x_d')
&=\int_{\R^d} \phi_t(x) (\tau((\bm{a},x_d)-x)-\tau(( \bm{a},x_d')-x))dx>0. 
\end{align*}
Thus for any $(a_1,\dots,a_{d-1})\in\R^{d-1}$, 
the \(C^\infty\)-function $\tau_t(a_1,\dots,a_{d-1},x)$ is strictly increasing for with respect to $x$. 

Assume $p<\infty$. 
Take any compact subset $K\subset \R^d$. 
We show
$\|\tau_t-\tau\|_{K,r,p}\to 0$ as $t\to 0$.  
We prove $\tau_t$ converges $\tau$ as $t\rightarrow 0$.  
Take $R>0$ satisfying $K\subset B(R):=\{x\in \R^d : |x|\leq R \}$. 
We assume $0<t<1$. 
Then we have $\phi_t*\tau=\phi_t*(\Indicator{B(R+1)}\tau)$. 
Since we have $\Indicator{B(R+1)}\tau\in L^p(\R^d)$, we obtain 
\begin{align*}
\| \phi_t * \tau-\tau\|_{K,r,p}
&=\sum_{|\alpha| \le r}\|\phi_t* (\Indicator{B(R+1)}\partial_\alpha \tau)-\Indicator{B(R+1)}\partial_\alpha \tau \|_{K,0,p}\\
&=\sum_{|\alpha| \le r}\|\phi_t* (\Indicator{B(R+1)}\partial_\alpha \tau)-\Indicator{B(R+1)}\partial_\alpha \tau \|_{\ReD,0,p} \to 0 \quad (t\to 0). 
\end{align*}
Here, we used a property of mollifier $\phi_t$ (see Theorem~8.14 in \cite{FollandReal1999} for example). 

In the case of $p=\infty$, by direct computation, we have
\begin{align*}
|\tau_t-\tau|_{K,r,\infty}
& \le C \sum_{|\alpha|\le r}\sup_{(x,y)\in {\rm supp}(\phi)\times K} |\partial_\alpha\tau(y-tx)-\partial_\alpha\tau(y)|
\to 0\quad (t\to 0). 
\end{align*}
Here $C:=\sup_{x\in \R^d}|\phi(x)|$. 
Thus in both cases above,  
By taking sufficiently small $t$, we obtain the desired \(C^\infty\)-function $\tilde{\tau}=\tau_t$. 
\end{proof}

\subsection{Neural autoregressive flows (NAFs)}
In this section, we prove that \emph{neural autoregressive flows} \cite{HuangNeural2018b} yield $\sup$-universal approximators for $\ConeOneDimTriangular$ (hence for $\CinftyOneDimTriangular$).
The proof is not merely an application of a known result in \citet{HuangNeural2018b} but it requires additional non-trivial consideration to enable the adoption of Lemma~3 in \citet{HuangNeural2018b} as it is applicable only for those smooth mappings that match certain boundary conditions.
\begin{definition}
A \emph{deep sigmoidal flow} (DSF; a special case of neural autoregressive flows) \citep[Equation~(8)]{HuangNeural2018b} is a flow layer \(g = (g_1, \ldots, g_d) \colon \ReD \to \ReD\) of the following form:
\begin{equation*}\begin{aligned}
g_k(\x) &:= \sigma^{-1} \left(\sum_{j=1}^{n} w_{k,j}(\upto{k-1}{\x}) \cdot \sigma\left(\frac{x_k - b_{k,j}(\upto{k-1}{\x})}{\tau_j(\upto{k-1}{\x})}\right)\right),
\end{aligned}\end{equation*}
where \(\sigma\) is the sigmoid function, \(n \in \Na\), \(w_j, b_j, \tau_j \colon \Re^{k-1} \to \Re\) (\(j \in [n]\)) are neural networks such that \(b_j(\cdot) \in (r_0, r_1)\), \(\tau_j(\cdot) \in (0, r_2)\), \(w_j(\cdot) > 0\), and \(\sum_{j=1}^{n} w_j(\cdot) = 1\) (\(r_0, r_1 \in \Re\), \(r_2 > 0\)).
We define \(\FDSF\) to be the set of all possible DSFs.
\end{definition}
\begin{proposition}[Universality of INNs based on DSF]\label{prop:DSF}
The elements of \(\FDSF\) are locally bounded, and \(\INN{\FDSF}\) is a \(\sup\)-universal approximator for \(\ConeOneDimTriangular\).
\end{proposition}
\begin{proof}
The elements of \(\FDSF\) are continuous, hence locally bounded.
Let $s=(s_1,\cdots, s_d)\in \ConeOneDimTriangular$. 
Take any compact set $K\subset \R^d$ and $\epsilon>0$. 
Since $K$ is compact, there exist $r_0, r_1\in \R$ such that $K\subset [r_0,r_1]^d$. 
Put $r_0'=r_0-1$, $r_1'=r_1+1$. 
We take a $C^1$-function $b\colon (r_0', r_1')\to \R$ satisfying 
\begin{enumerate}
    \item $b|_{[r_0,r_1]}=0$, 
    \item $b|_{(r_0',r_0)}$ and $b|_{(r_1,r_1')}$ are strictly increasing, 
    \item $\lim_{x\to r_0'+0}b(x)=-\infty$ and $\lim_{x\to  r_1'-0}b(x)=\infty$,
    \item $\lim_{x\to r_0'+0}\frac{d(\sigma\circ b)}{dx}(x)$ and  $\lim_{x\to  r_1'-0}\frac{d(\sigma\circ b)}{dx}(x)$ exist in $\R$, 
\end{enumerate}
where $\sigma$ is the sigmoid function.
For each $k \in [d]$, we define a $C^1$-map $\tilde{s}_k\colon [r_0',r_1']^{k-1}\times (r_0',r_1')\times [r_0', r_1']^{d-k}\to \R$, which is strictly increasing with respect to $x_k$, by 
\[
\tilde{s}_k(x):=s_k(x)+b(x_k)\quad (x=(x_1,\cdots, x_d)). 
\]
Moreover, we define a map $S\colon [r_0',r_1']^d\to [0,1]^d$ by 
\begin{align*}
    S_k|_{[r_0',r_1']^{k-1}\times (r_0',r_1')\times [r_0', r_1']^{d-k}}&=\sigma \circ \tilde{s}_k,\\
    S_k(x_1,\cdots, x_{k-1}, r_0', x_{k+1}, \cdots, x_d)&=0,\\
    S_k(x_1,\cdots, x_{k-1}, r_1', x_{k+1}, \cdots, x_d)&=1,
\end{align*}
where we write $S=(S_1,\cdots, S_d)$. 
Then, by Lemma~\ref{lem:appendix:DSF extended map is smooth}, $S$ satisfies the assumptions of Lemma~3 in \cite{HuangNeural2018b}. 
Since $S([r_0,r_1]^d)\subset (0,1)^d$ is compact, 
there exists a positive number $\delta>0$ such that 
\[ 
S([r_0,r_1]^d) + B(\delta):= \{S(x)+v \ :\ x\in [r_0,r_1]^d, v\in B(\delta)\} \subset [\delta,1-\delta]^d,
\]
where $B(\delta):=\{x\in \R^d : |x|\leq \delta\}$. 
Let $L>0$ be a Lipschitz constant of $\sigma^{-1}\colon (0,1)^d\to \R^d$ on $[\delta, 1-\delta]^d$. 
By Lemma~3 in \cite{HuangNeural2018b}, 
there exists $g\in \INN{\FDSF}$ such that 
\begin{align*}
    \|S-\sigma \circ g\|_{[r_0', r_1']^d,0,\infty}<\min\left\{\delta, \frac{\epsilon}{L}\right\}. 
\end{align*}
As a result, $\sigma\circ g([r_0,r_1]^d) \subset S([r_0,r_1]^d) + B(\delta) \subset [\delta, 1-\delta]^d$.
Then we obtain 
\begin{align*}
\|s-g\|_{K, 0, \infty}
\leq 
\|s-g\|_{[r_0,r_1]^d, 0, \infty}
&=\| \sigma^{-1} \circ \sigma\circ s- \sigma^{-1}\circ \sigma \circ g\|_{[r_0,r_1]^d, 0, \infty}\\
&\leq L\| S -\sigma \circ g\|_{[r_0,r_1]^d,0,\infty}\\
&<\epsilon. 
\end{align*}
\end{proof}

\begin{lemma}
\label{lem:appendix:DSF extended map is smooth}
We denote by $\mathcal{T}^1$ the set of all $C^1$-increasing triangular mappings from $\R^d$ to $\R^d$. 
For $s=(s_1,\cdots, s_d)\in \mathcal{T}^1$, we define 
a map $S\colon [r_0',r_1']^d\to [0,1]^d$ as in the proof of Proposition~\ref{prop:DSF}. 
Then $S$ is a $C^1$-map. 
\end{lemma}
\begin{proof}
It is enough to show that 
$S_d\colon [r_0', r_1']^d\to [0,1]$ is a $C^1$-function. 
We prove that for any $i\in [d]$, the $i$-th partial derivative of $S_d$ exists and that it is continuous on $[r_0', r_1']^d$. 
First, for $i\in [d-1]$, we consider the $i$-th partial derivative. \\
\textbf{Claim 1}. 
\begin{align*}
    \frac{\partial S_d}{\partial x_i}(x)=
    \begin{cases}
    \frac{d\sigma}{d x}(s_i(x)+b(x_d))\frac{\partial s_d}{\partial x_i}(x) & (x\in [r_0',r_1']^{d-1}\times (r_0', r_1'))\\
    0 & (x_d= r_0', r_1')
    \end{cases}
\end{align*}
In fact, 
for $x\in [r_0', r_1']^{d-1}\times (r_0',r_1')$, we have  
\begin{align*}
    \frac{\partial S_d}{\partial x_i}(x)
    =\frac{\partial (\sigma \circ \tilde{s_d})}{\partial x_i}(x)
    =\frac{d\sigma }{dx} (s_d(x)+b(x_d)) \left(\frac{\partial s_d}{\partial x_i}(x)+0\right). 
\end{align*}
For $x=(x_{\leq {d-1}}, r_0')$, we have 
\begin{align*}
    \frac{\partial S_d}{\partial x_i}(x)
    &=\lim_{h\to 0}\frac{S_d(x_{\leq i-1}, x_i+h, x_{i+1},\cdots, x_{d-1}, r_0')-S_d(x_{\leq d-1}, r_0')}{h}\\
    &=\lim_{h\to 0}\frac{0-0}{h}=0
\end{align*}
Here, note that by the definition of $S_d$, the notation $S_d(x_{\leq i-1}, x_i+h, x_{i+1},\cdots, x_{d-1}, r_0')$ makes sense even if $x_i=r_0'$ or $x_i=r_1'$. 
We can verify the case $x=(x_{\leq d-1}, r_1')$ similarly.

Next, we show that $\frac{\partial S_d}{\partial x_i}$ is continuous. 
We take any $x_{\leq d-1}\in [r_0',r_1']^{d-1}$.
Since we have $\lim_{x\to r_0'}b(x)=-\infty$, $\lim_{x\to r_1'}b(x)$,  $\lim_{x\to \pm \infty} \frac{d\sigma }{dx}(x)=0$, and 
$|\frac{\partial s_d}{\partial x_I}(x)|<\infty$ $(x\in [r_0', r_1']^d)$, 
we obtain 
\begin{align*}
    \lim_{x\to (x_{d-1}, r_0')}\frac{d\sigma}{d x}(s_i(x)+b(x_d))\frac{\partial s_d}{\partial x_i}(x)=0,\\
    \lim_{x\to (x_{d-1}, r_1')}\frac{d\sigma}{d x}(s_i(x)+b(x_d))\frac{\partial s_d}{\partial x_i}(x)=0.
\end{align*}
Therefore, the partial derivative $\frac{\partial S_d}{\partial x_i}(x)$ is continuous on $[r_0', r_1']^d$ for $i\in [d-1]$. 

Next, we consider the $d$-th derivative of $S_d$. \\
\textbf{Claim 2.}
\begin{align*}
    \frac{\partial S_d}{\partial x_d}(x)=
    \begin{cases}
    \frac{d\sigma}{d x}(s_d(x)+b(x_d)) \left( \frac{\partial s_d}{\partial x_d }(x)+\frac{d b}{d x}(x_d)\right) & (x\in [r_0', r_1']^{d-1}\times (r_0', r_1'))\\
    e^{s_d(x_{\leq d-1},r_0')}\lim_{x\to r_0'+0}\frac{d (\sigma\circ b)}{dx}(x) & (x_d=r_0')\\
    e^{-s_d(x_{\leq d-1}, r_1')}\lim_{x\to r_1'-0}\frac{d (\sigma\circ b)}{dx}(x) & (x_d=r_1')
    \end{cases}
\end{align*}
We verify Claim 2. 
Since it is clear for the case $x\in [r_0', r_1']^{d-1}\times (r_0',r_1')$ by the definition of $S_k$, we consider the case $x_d=r_0', r_1'$. \\
\textbf{Subclaim.}
For $x_{\leq d-1}'\in [r_0',r_1']^{d-1}$, 
\begin{align*}
    \lim_{x\to (x_{\leq d-1}', r_0')}&\frac{\sigma (s_d(x)+b(x_d))}{\sigma(b(x_d))}=e^{s_d(x_{\leq d-1}',r_0')}\\
    \lim_{x\to (x_{\leq d-1}', r_1')}&\frac{\sigma(s_d(x)+b(x_d))-1}{\sigma(b(x_d))-1}=e^{-s_d(x_{\leq d-1}', r_1')}
\end{align*}
We verify this subclaim. From $\lim_{x\to r_0'}b(x)=-\infty$, we have
\begin{align*}
    \frac{\sigma (s_d(x)+b(x_d))}{\sigma(b(x_d))}
    &=\frac{1+e^{-b(x_d)}}{1+e^{-s_d(x)-b(x_d)}}
    =\frac{e^{b(x_d)}+1}{e^{b(x_d)}+e^{-s_d(x)}}\\
    &\to \frac{1}{e^{-s_d(x_{\leq d-1}',r_0')}}=e^{s_d(x_{\leq d-1}',r_0')} \quad (x \to (x_{\leq {d-1}}', r_0'))
\end{align*}
Similarly, from $\lim_{x\to r_1'}b(x)=\infty$, we have 
\begin{align*}
    \frac{\sigma(s_d(x)+b(x_d))-1}{\sigma(b(x_d))-1}
    &=e^{-s_d(x)}\frac{1+e^{-b(x_d)}}{1+e^{-s_d(x)-b(x_d)}}\\
    &\to e^{-s_d(x_{\leq d-1}, r_1')} \quad (x\to (x_{\leq d-1}',r_1')).
\end{align*}
Therefore, our subclaim has been proved. 
By using L'H\^opital's rule, we have 
\begin{align*}
\lim_{h\to +0}\frac{\sigma(b(r_0'+h))}{h}=\lim_{x\to r_0'}\frac{d(\sigma\circ b)}{dx} (x),\quad 
\lim_{x\to r_1'}\frac{\sigma(b(r_1'+h))-1}{h}=\lim_{x\to r_1'}\frac{d(\sigma\circ b)}{dx}(x). 
\end{align*}
Then, from Subclaim, we obtain 
\begin{align*}
   \frac{\partial S_d}{\partial x_d}(x_{\leq d-1}, r_0')
   &=\lim_{h\to +0} \frac{\sigma(s_d(x_{\leq d-1},r_0'+h)+b(r_0'+h))-0}{h}\\
   &=\lim_{h\to +0} \frac{\sigma(s_d(x_{\leq d-1}, r_0'+h)+b(r_0'+h))}{\sigma(b(r_0+h))}\cdot \frac{\sigma(b(r_0'+h))}{h}\\
   &=e^{s_d(x_{\leq d-1}, r_0')}\lim_{x\to r_0'+0}\frac{d(\sigma\circ b)}{dx}(x), \\
   \frac{\partial S_d}{\partial x_d}(x_{\leq d-1}, r_1')
&=\lim_{h\to -0}\frac{\sigma(s_d(x_{\leq d-1},r_1'+h)+b(r_1'+h))-1}{h}\\
&=\lim_{h\to -0}\frac{\sigma(s_d(x_{\leq d-1},r_1'+h)+b(r_1'+h))-1}{\sigma(b(r_1'+h))-1}\cdot \frac{\sigma(b(r_1'+h))-1}{h}\\
&=e^{s_d(x_{\leq d-1},r_1')}\lim_{x\to r_1'}\frac{d(\sigma \circ b)}{dx}(x).  
\end{align*}
Therefore, Claim 2 was proved. 

Finally, we verify $\frac{\partial S_d}{\partial x_d}(x)$ is continuous on $[r_0',r_1']^d$. 
Fix $x'_{\leq d-1}\in [r_0', r_1']^{d-1}$. 
Since we have 
$\lim_{x\to (x'_{\leq d-1}, r_0')}\frac{d\sigma}{dx}(\sigma_d(x)+b(x_d))\frac{\partial s_d}{\partial x_d}(x)=0$, 
from Claim 2, it is enough to show the following:\\
\textbf{Claim 3.}
\begin{align*}
    \lim_{x\to (x_{\leq d-1}', r_0')}\frac{d \sigma}{d x}(s_d(x)+b(x_d))\frac{db}{d x}(x_d)
    &=e^{s_d(x_{\leq d-1},r_0')}\lim_{x\to r_0'+0}\frac{d (\sigma\circ b)}{dx}(x), \\
    \lim_{x\to (x_{\leq d-1}', r_1')}\frac{d \sigma}{d  x}(s_d(x)+b(x_d))\frac{db}{dx}(x_d)
    &=e^{-s_d(x_{\leq d-1}, r_1')}\lim_{x\to r_1'-0}\frac{d (\sigma\circ b)}{dx}(x). 
\end{align*}
We verify Claim~3. 
We have
\begin{align*}
\frac{d \sigma}{d x}(s_d(x)+b(x_d))\frac{db}{d x}(x_d)
&=\frac{\frac{d\sigma}{dx}(s_d(x)+b(x_d))}{\frac{d\sigma}{dx}(b(x_d))}\frac{d\sigma}{dx}(b(x_d))\frac{db}{dx}(x_d)\\
&=\frac{\frac{d\sigma}{dx}(s_d(x)+b(x_d))}{\frac{d\sigma}{dx}(b(x_d))}\frac{d(\sigma \circ b)}{dx}(x_d). 
\end{align*}
Since we have $\frac{d\sigma}{dx}(x)= \sigma(x)(1-\sigma(x))$, from Subclaim above, Claim 3 follows from 
\begin{align*}
\frac{\frac{d\sigma}{dx}(s_d(x)+b(x_d))}{\frac{d\sigma}{dx}(b(x_d))}
&=\frac{\sigma(s_d(x)+b(x_d))}{\sigma (b(x_d))}\cdot \frac{1-\sigma(s_d(x)+b(x_d))}{1-\sigma(b(x_d))}\\
&\to \begin{cases} 
e^{s_d(x_{\leq d-1}', r_0')} & (x\to (x_{\leq d-1}', r_0'))\\
e^{-s_d(x_{\leq d-1}', r_1')} & (x\to (x_{\leq d-1}', r_1'))
\end{cases}. 
\end{align*}
Therefore, we proved the continuity of $\frac{\partial S_d}{\partial x_d}(x)$. 
\end{proof}

\newcommand{\SoSTransformer}[2]{\mathfrak{B}_{#1}(#2)}
\subsection{Sum-of-squares polynomial flows (SoS flows)} \label{appendix: SOS}
In this section, we prove that \emph{sum-of-squares polynomial flows} \cite{DBLP:conf/icml/JainiSY19} yield \ARFINNs{} with the $\sup$-universal approximation property for $\ConeOneDimTriangular$ (hence for $\CinftyOneDimTriangular$).
Even though \citet{DBLP:conf/icml/JainiSY19} claimed the distributional universality of the SoS flows by providing a proof sketch based on the univariate Stone-Weierstrass approximation theorem, we regard the sketch to be invalid or at least incomplete as it does not discuss the smoothness of the coefficients, i.e., whether the polynomial coefficients can be realized by continuous functions. Here, we provide complete proof that takes an alternative route to prove the $\sup$-universality of the SoS flows via the multivariate Stone-Weierstrass approximation theorem.

A \emph{sum-of-squares polynomial flow} (SoS flow) \citep[Equation~(9)]{DBLP:conf/icml/JainiSY19} is a flow layer \(g = (g_1, \ldots, g_d) \colon \ReD \to \ReD\) of the following form:
\begin{equation*}\begin{aligned}
g_k(\x) &:= \SoSTransformer{2r+1}{x_k; C_k(\upto{k-1}{\x})}, \\
\SoSTransformer{2r+1}{z; (c, \ba)} &:= c + \int_0^z \sum_{b=1}^B\left(\sum_{l=0}^r a_{l, b} u^l\right)^2 du,
\end{aligned}\end{equation*}
where \(r \in \Na \cup \{0\}\), \(B \in \Na\), $c \in \Re$, $\ba \in \Re^{B (r+1)}$, and \(C_k \colon \Re^{k-1} \to \Re^{B (r+1) + 1}\) is a certain map, for example, a neural network. 

Here, we consider a small class of SoS flows as follows:
\begin{definition}
Let $\mathcal{H}$ be a function on $\Re^{d-1}$.
For $c \in \Re$ and $h_1,\dots, h_r \in \mathcal{H}$, Let
\[\tilde{\mathfrak{B}}(\x; c,h_1,\dots, h_r) := c + \int_0^{x_{d}} \left(\sum_{l=0}^r h_l(\x_{\le d-1} )u^l\right)^2 du.\]
Then, we define the set $\HSoS$ as a subset consisting of $\tilde{\mathfrak{B}}(\cdot; h_1,\dots, h_r)$ where $r \ge 1$ and $h_i$'s are elements of $\mathcal{H}$. 
\end{definition}
Then, we have the following proposition:
\begin{proposition} \label{prop: universality for sos}
Let $r \ge 0$.
Let $\mathcal{H} \subset C^r(\Re^{d-1})$ and assume that $\mathcal{H}$ is a $W^{r,\infty}$-universal approximator for the set of $(d-1)$-variable polynomials.
Then, \(\INN{\HSoS}\) is a $W^{r,\infty}$-universal approximator for $\mathcal{S}_c^{r+1}$.
\end{proposition}
\begin{proof}
We only illustrate the proof in the cases of $r=0$ and $r=1$.
The general cases follow from a similar argument with the Leibniz rule and chain rule.

The $L^\infty$-universality follows from the Stone-Weierstrass approximation theorem as in the below.
Let \(s = (s_1, \ldots, s_d) \in \ConeOneDimTriangular\), a compact subset \(K \subset \ReD\), and \(\epsilon > 0\) be given.
Then, there exists \(R > 0\) such that \(K \subset [-R, R]^d\).
Since \(s_d(\x)\) is strictly increasing with respect to \(x_d\) and \(s\) is \(C^1\), we have \(\eta(\x) := \frac{\partial s_d}{\partial x_d}(\x) > 0\) and \(\eta\) is continuous.
Therefore, we can apply the Stone-Weierstrass approximation theorem \citep[Corollary~4.50]{FollandReal1999} to \(\sqrt{\eta(\x)}\):
for any \(\delta > 0\), there exists a polynomial \(\pi(x_1, \ldots, x_d)\) such that \(\supRangenorm{[-R, R]^d}{\sqrt{\eta}- \pi} < \delta\).
Then, by rearranging the terms, there exist \(r \in \Na\) and polynomials \(\xi_{l}(x_1, \ldots, x_{d-1})\) such that \(\pi(x_1, \ldots, x_d) = \sum_{l=0}^r \xi_l(x_1, \ldots, x_{d-1})x_d^l\).
Now, define
\begin{equation*}\begin{aligned}
\tilde g_d(\x) &:= s_d(\upto{d-1}{\x}, 0) + \int_0^{x_d} (\pi(\upto{d-1}{\x}, u))^2 du \\
&= s_d(\upto{d-1}{\x}, 0) + \int_0^{x_d} \left(\sum_{l=0}^r \xi_l(x_1, \ldots, x_{d-1})u^l\right)^2 du
\end{aligned}\end{equation*}
and \(\tilde g(\x) := (x_1, \ldots, x_{d-1}, \tilde g_d(\x))\).
Then,
\begin{equation*}\begin{aligned}
\supKnorm{s - \tilde g} &= \sup_{\x \in K} \left|s_d(\x) - \tilde g_d(\x)\right| \\
&= \sup_{\x \in K} \left|s_d(\upto{d-1}{\x}, 0) + \int_0^{x_d} \eta(\upto{d-1}{\x}, u) du - \tilde g_d(\x) \right| \\
&= \sup_{\x \in K} \left|\int_0^{x_d} (\sqrt{\eta(\upto{d-1}{\x}, u)}^2 - \pi(\upto{d-1}{\x}, u)^2) du \right| \\
&\leq R \cdot \sup_{\x \in [-R, R]^d} \left|\sqrt{\eta(\x)}^2 - \pi(\x)^2\right| \\
&= R \cdot \sup_{\x \in [-R, R]^d} |\sqrt{\eta(\x)} + \pi(\x)| \cdot |\sqrt{\eta(\x)} - \pi(\x)| \\
&\leq R \left(\sup_{\x \in [-R, R]^d} 2\sqrt{\eta(\x)} + \delta\right) \delta,
\end{aligned}\end{equation*}
where we used
\begin{equation*}\begin{aligned}
\sup_{\x \in [-R, R]^d}|\sqrt{\eta(\x)} + \pi(\x)| &\leq \sup_{\x \in [-R, R]^d}|2\sqrt{\eta(\x)}| + |\sqrt{\eta(\x)} - \pi(\x)| \\
&\leq \sup_{\x \in [-R, R]^d} 2\sqrt{\eta(\x)} + \delta.
\end{aligned}\end{equation*}
It is straightforward to show that there exists \(g \in \FSoS\) such that \(\supKnorm{\tilde g - g} < \frac{\epsilon}{2}\) by approximating each of \(s_d(\upto{d-1}{\x})\) and \(\xi_l\) on \(K\) using neural networks.
Finally, take \(\delta\) to be small enough so that \(\supKnorm{s - \tilde g} < \frac{\epsilon}{2}\) holds. 

Next, we consider the $W^{1,\infty}$-universality. We use the same notations as above. We note that since $s\in \mathcal{S}_c^2$, we have $\eta\in C^1$, and \(\eta\) is positive and continuous. This enables us to apply the Stone-Weierstrass approximation theorem \citep[Theorem 5]{peetExponentiallyStableNonlinear2009} to \(\sqrt{\eta(\x)}\) :
for any \(\delta > 0\), there exists a polynomial \(\pi(x_1, \ldots, x_d)\) such that $\|\sqrt{\eta}- \pi\|_{[-R,R]^d,1,\infty}<\delta$. We define $\tilde g_d$ and $\tilde{g}$ as above. Then we have
\begin{align*}
 \|s - \tilde g\|_{K,1,\infty}
 &=\left\|\int_0^{x_d} (\sqrt{\eta(\upto{d-1}{\x}, u)}^2 - \pi(\upto{d-1}{\x}, u)^2) du\right\|_{K,1,\infty}\\
 &\le \sup_{\x \in K} \left|\int_0^{x_d} (\sqrt{\eta(\upto{d-1}{\x}, u)}^2 - \pi(\upto{d-1}{\x}, u)^2) du \right|\\
 &\ +\sup_{\x \in K} \sum_{i=1}^{d-1}\left|\partial_{x_i}\int_0^{x_d} (\sqrt{\eta(\upto{d-1}{\x}, u)}^2 - \pi(\upto{d-1}{\x}, u)^2) du \right|\\
 &\ +\sup_{\x \in K} \left|\partial_{x_d}\int_0^{x_d} (\sqrt{\eta(\upto{d-1}{\x}, u)}^2 - \pi(\upto{d-1}{\x}, u)^2) du \right| \\
 &=:I+II+III.
\end{align*}
In a similar manner as above, we have $I\le R \left(\sup_{\x \in [-R, R]^d} 2\sqrt{\eta(\x)} + \delta\right) \delta$. We note that since $\eta\in C^1$ and \(\eta\) is positive and continuous, we have $\|\sqrt{\eta}\|_{[-R,R]^d,1,\infty}<\infty$. A direct computation gives
\begin{align*}
    II&=2\sup_{\x \in K} \sum_{i=1}^{d-1}\left|\int_0^{x_d}\left\{ \sqrt{\eta(\upto{d-1}{\x}, u)}\partial_{x_i}\sqrt{\eta(\upto{d-1}{\x}, u)} - \pi(\upto{d-1}{\x}, u)\partial_{x_i}\pi(\upto{d-1}{\x}, u)\right\} du \right|\\
    &\le 2\sup_{\x \in K} \sum_{i=1}^{d-1}
    \left|\int_0^{x_d}\left\{\sqrt{\eta(\upto{d-1}{\x}, u)}-\pi(\upto{d-1}{\x}, u)\right\}\partial_{x_i}\sqrt{\eta(\upto{d-1}{\x}, u)} du \right|\\
    &\ +2\sup_{\x \in K} \sum_{i=1}^{d-1}
    \left|\int_0^{x_d}\pi(\upto{d-1}{\x}, u)\partial_{x_i}\left\{\sqrt{\eta(\upto{d-1}{\x}, u)}-\pi(\upto{d-1}{\x}, u)\right\} du \right|\\
    &\le 2(d-1)R(2\|\sqrt{\eta}\|_{[-R,R]^d,1,\infty}+\delta)\|\sqrt{\eta}- \pi\|_{[-R,R]^d,1,\infty}\\
    &\le 2(d-1)R(2\|\sqrt{\eta}\|_{[-R,R]^d,1,\infty}+\delta)\delta.
\end{align*}
A simple computation gives
\begin{align*}
    III
    =\sup_{\x \in K} \left| \sqrt{\eta(\x)}+ \pi(\x) \right|\left| \sqrt{\eta(\x)}- \pi(\x) \right|
    \le \left(\sup_{\x \in [-R, R]^d} 2\sqrt{\eta(\x)} + \delta\right) \delta.
\end{align*}
In the similar manner as above, we can see that there exists \(g \in \FSoS\) such that $\|\tilde g - g \|_{K,1,\infty} <\frac{\epsilon}{2}$. Finally, taking \(\delta\) to be small enough so that $\|s - \tilde g\|_{K,1,\infty} < \frac{\epsilon}{2}$ holds, the assertion is proved.
\end{proof}

\section{Universality of NODE-based INNs}
\label{sec:appendix:universality-proof-NODE-version}

Here, we provide a proof of Theorem~\ref{thm: NODE is sup-universal}:
\begin{proof}[Proof of Theorem~\ref{thm: NODE is sup-universal}]
By Theorem \ref{theorem:main:1}, we only consider an approximation of the elements of $\Xi^{\infty}$.
Let $g \in \Xi^\infty$.
Then, by Definition \ref{def: flow endpoints}, there exists $f \in \Lipsp{} \cap \infty$ such that
\[f(\cdot):=\left.\frac{\partial  \Phi(\cdot,t)}{\partial t}\right|_{t=0}.
\]
for some flow $\Phi$
Therefore, $g$ is arbitrarily approximated by an element of $\INNHNODE$ by Lemma \ref{appendix:lem:ODE flow endpoint approximation}.
\end{proof}

The following lemma, used in the above proof, allows us to approximate an autonomous ODE flow endpoint by approximating the differential equation. See Definition~\ref{def: autonomous ODE flow endpoints} for the definition of $\ODEFlowEnds{\cdot}$.
\begin{lemma}[Approximation of Autonomous-ODE flow endpoints]
\label{appendix:lem:ODE flow endpoint approximation}
Let $r \ge 0$.
Assume \(\NODEJacobiFuncClass \subset \Lipsp{} \cap C^r\) is a $W^{r,\infty}$-universal approximator for $\Lipsp{} \cap C^r$.
Then, \(\ODEFlowEnds{\NODEJacobiFuncClass{}}\) is a \(W^{r,\infty}\)-universal approximator for \(\ODEFlowEnds{\Lipsp{} \cap C^r}\).
\end{lemma}
\begin{proof}
We first treat the case of $r > 0$.
By combining the fact that the map
\[ (\x , f) \mapsto \IVP{f}{\x}{1} \]
is $C^r$ map (Theorem B.3 (ii) in \cite{DuistermaatLie2000}) with the Berge maximum theorem \cite{AliprantisInfinite2006}, we see that
for any compact set $K \subset \ReD$ and $F \in \Lipsp{} \cap C^r$ we see that the map
\[
    f \mapsto
    \WspKnorm{K}{r}{\infty}{\IVP{f}{\cdot}{1} - \IVP{F}{\cdot}{1}
    }
    = \sum_{|\alpha| \le r}\sup_{\x \in K} \|{\IVP{f}{\x}{1} - \IVP{F}{\x}{1}}\|
\]
is continuous.
Therefore, the $W^{r, \infty}$-universality of $\ODEFlowEnds{\NODEJacobiFuncClass}$ for $\Psi(\Lipsp{} \cap C^r)$ follows from that of $\NODEJacobiFuncClass{}$ for $\Lipsp{} \cap C^r$.

We next treat the case of $r=0$.
Let $\phi \in \ODEFlowEnds{\Lipsp{}}$. Then, by definition, there exists $F \in \Lipsp{}$ such that $\phi = \IVP{F}{\cdot}{1}$.
Let $\LipConst{F}$ denote the Lipschitz constant of $F$.
In the following, we approximate $\IVP{F}{\cdot}{1}$ by approximating $F$ using an element of $\NODEJacobiFuncClass{}$.

Let $\varepsilon > 0$, and let $K \subset \R^d$ be a compact subset of $\R^d$.
We show that there exists $f \in \NODEJacobiFuncClass{}$ such that $\supKnorm{\IVP{F}{\cdot}{1} - \IVP{f}{\cdot}{1}} < \varepsilon$. Note that $\IVP{f}{\cdot}{\cdot}$ is well-defined because $\NODEJacobiFuncClass{} \subset \Lipsp{}$.
Define
\[
K' := \left\{\x \in \R^d \ \bigg|\ \inf_{\y \in \IVP{F}{K}{[0, 1]}} \|\x - \y\| \leq 2 e^{\LipConst{F}}\right\}.
\]
Then, $K'$ is compact. This follows from the compactness of $\IVP{F}{K}{[0, 1]}$: (i) $K'$ is bounded since $\IVP{F}{K}{[0, 1]}$ is bounded, and (ii) it is closed since the function $\x \mapsto \min_{\y \in \IVP{F}{K}{[0, 1]}} \|\x - \y\|$ is continuous and hence $K'$ is the inverse image of a closed interval $[0, 2e^{\LipConst{F}}]$ by a continuous map.

Since $\NODEJacobiFuncClass{}$ is assumed to be an \(L^\infty\)-universal approximator for $\Lipsp{}$, for any $\delta > 0$, we can take $f \in \NODEJacobiFuncClass{}$ such that $\supRangenorm{K'}{f - F} < \delta$.
Let $\delta$ be such that $0 < \delta < \min\{\varepsilon / (2e^{\LipConst{F}}), 1\}$, and take such an $f$.

Fix $\x_0 \in K$ and define $\targetError{t} := \|\IVP{F}{\x_0}{t} - \IVP{f}{\x_0}{t}\|$.
Let $\Bound{} := \BoundDef{}$ and we show that
\[
\targetError{t} < 2\Bound{}
\]
holds for all $t \in [0, 1]$.
We prove this by contradiction. Suppose that there exists $t'$ for which the inequality does not hold. Then, the set $\mathcal{T} := \{t \in [0, 1] | \targetError{t} \geq 2 \Bound{}\}$ is not empty and
 thus $\tau := \inf \mathcal{T} \in [0, 1]$.
For this $\tau$, we show both $\targetError{\tau} \leq \Bound{}$ and $\targetError{\tau} \geq 2\Bound{}$.
First, we have
\begin{align*}
\targetError{\tau} &= \left\|\IVP{F}{\x_0}{\tau} - \IVP{f}{\x_0}{\tau}\right\| \\
&= \left\|\x_0 + \int_0^\tau F(\IVP{F}{\x_0}{t}) dt - \x_0 - \int_0^\tau f(\IVP{f}{\x_0}{t}) dt\right\| \\
&\leq \left\|\int_0^\tau (F(\IVP{F}{\x_0}{t}) - F(\IVP{f}{\x_0}{t})) dt\right\| \\
&\qquad + \left\|\int_0^\tau (F(\IVP{f}{\x_0}{t}) - f(\IVP{f}{\x_0}{t})) dt\right\|.
\end{align*}
The last term can be bounded as
\[
\left\|\int_0^\tau (F(\IVP{f}{\x_0}{t}) - f(\IVP{f}{\x_0}{t})) dt\right\| \leq \int_0^\tau \delta dt
\]
because of the following argument.
If $\tau = 0$, then both sides are equal to zero, hence it holds with equality.
If $\tau > 0$, then for any $t < \tau$, we have $\IVP{f}{\x_0}{t} \in K'$ because $t < \tau$ implies $\targetError{t} \leq 2 \Bound{}$.
In this case, $\supRangenorm{K'}{F - f} < \delta$ implies the inequality.
Therefore, we have
\[
\targetError{\tau} \leq \LipConst{F}\int_0^\tau \targetError{t} dt + \int_0^\tau \delta dt.
\]
Now, by applying \Gronwall{}'s inequality \cite{GronwallNote1919}, we obtain
\[
\targetError{\tau} \leq \delta \tau e^{\LipConst{F} \tau} \leq \Bound{}.
\]
On the other hand, by the definition of $\mathcal{T}$ and the continuity of $\targetError{\cdot}$, we have $\targetError{\tau} \geq 2 \Bound{}$.
These two inequalities contradict.

Therefore, $\supKnorm{\IVP{F}{\cdot}{1} - \IVP{f}{\cdot}{1}} = \sup_{\x_0 \in K} \targetError{1} \leq 2 \Bound{} = 2\BoundDef{}$ holds.
Since $\delta < \varepsilon / (2e^{\LipConst{F}})$, the right-hand side is smaller than $\varepsilon$.
\end{proof}

When we construct a NODE to approximate target a diffeomorphism, we may insert any invertible affine map between flow layers by definition (see Definition \ref{def: INNM}).
However, we actually need an affine layer only in the last layer to obtain a universality of NODE, 
namely we have the following proposition:
\begin{proposition}
\label{prop: strong ver. of NODE universality}
The notation is as in Theorem \ref{thm: NODE is sup-universal}.
Then, the subset 
\[\{ W\circ g_1 \circ \dots \circ g_k : k\ge0, W \in \FLin, g_1,\dots, g_k \in \Psi(\mathcal{H}) \}  \]
of $\INNHNODE$ has a $W^{r,\infty}$-universal approximation property for $\mathcal{D}^{\max\{r,1\}}$, where $\mathcal{H}$ is a subset of $\Lipsp{} \cap C^r$ as in Theorem \ref{thm: NODE is sup-universal}.
\end{proposition}
\begin{proof}
Let $F \in \mathcal{D}^{\max\{r,1\}}$.
Take any compact set $K\subset U$ and $\varepsilon>0$.
First, thanks to Lemma~\ref{red Cr to Cinf} and \ref{red to comp. supp. diff}, there exists a $G \in \DcRDCmd{\infty}$ and an affine transform $W \in \FLin$ such that \[\Restrict{W\circ G}{K}=\Restrict{F}{K}.\]
Then, we use Lemma~\ref{lem: diffc2 is generated by flow endpoints} to show that there exists a finite set of flow endpoints (Definition~\ref{def: flow endpoints}) $g_1, \ldots, g_k \in \FlowEnds{\infty}$ such that
\[
G = g_k \circ \cdots \circ g_1.
\]

We now construct $f_j \in \Lipsp{}$ such that $g_j = \IVP{f_j}{\cdot}{1}$.
By Definition~\ref{def: flow endpoints}, for each $g_j$ ($1 \leq j \leq k$), there exists an associated flow $\Phi_j$.
Now, define 
\[f_j(\cdot):=\left.\frac{\partial  \Phi_j(\cdot,t)}{\partial t}\right|_{t=0}.
\]
Then, $f_j \in \Lipsp{}$ because it is a compactly-supported $C^\infty$-map:
it is compactly supported since there exists a compact subset $K_j \subset \R^d$ containing the support of $\Phi(\cdot, t)$ for all $t$, and hence $\Phi(\cdot, t) - \Phi(\cdot, 0)$ is zero in the complement of $K_j$.

Now, $\Phi_j(\x, t) = \IVP{f_j}{\x}{t}$
since, by additivity of the flows,
\begin{align*}
    \frac{\partial \Phi_j}{\partial t}(\x,t)&=\lim_{s\rightarrow 0}\frac{\Phi_j(\x,t+s)-\Phi_j(\x,t)}{s}
    = \lim_{s\rightarrow 0}\frac{\Phi_j(\Phi_j(\x,t),s)-\Phi_j(\Phi_j(\x,t), 0)}{s}\\
    &= \left.\frac{\partial  \Phi_j(\Phi_j(\x,t),s)}{\partial s}\right|_{s=0}
    = f_j(\Phi_j(\x,t)),
\end{align*}
and hence it is a solution to the initial value problem that is unique.
As a result, we have $g_j = \Phi_j(\cdot, 1) = \IVP{f_j}{\cdot}{1}$.

By combining Lemma~\ref{prop: compatibility of approximation} and Lemma~\ref{appendix:lem:ODE flow endpoint approximation}, there exist $\phi_1, \ldots, \phi_k \in \Psi(\NODEJacobiFuncClass{})$ such that
\[
\WspKnorm{K}{r}{\infty}{g_k \circ \cdots \circ g_1 - \phi_k \circ \cdots \circ \phi_1} < \frac{\varepsilon}{\opnorm{W}},
\]
where $\opnorm{\cdot}$ denotes the operator norm.
Therefore, we have that $W \circ \phi_k \circ \cdots \circ \phi_1 \in \INNHNODE$ satisfies \begin{align*}
\WspKnorm{K}{r}{\infty}{F - W \circ \phi_k \circ \cdots \circ \phi_1}
&= \WspKnorm{K}{r}{\infty}{W \circ G - W \circ \phi_k \circ \cdots \circ \phi_1} \\
&\leq \opnorm{W}\WspKnorm{K}{r}{\infty}{g_k \circ \cdots \circ g_1 - \phi_k \circ \cdots \circ \phi_1} \\
&< \varepsilon
\end{align*}
\end{proof}

\end{document}